\documentclass{article} 
\usepackage[numbers]{natbib}
\usepackage[preprint]{neurips_2022}

\usepackage{amsmath,amsfonts,bm}

\def\eqref#1{equation~\ref{#1}}

\def\1{\bm{1}}

\DeclareMathAlphabet{\mathsfit}{\encodingdefault}{\sfdefault}{m}{sl}
\SetMathAlphabet{\mathsfit}{bold}{\encodingdefault}{\sfdefault}{bx}{n}

\newcommand{\R}{\mathbb{R}}

\newcommand{\KL}{D_{\mathrm{KL}}}

\DeclareMathOperator*{\argmax}{arg\,max}

\usepackage{hyperref}
\usepackage{url}

\usepackage[utf8]{inputenc} 
\usepackage[T1]{fontenc}    
\usepackage{hyperref}       
\usepackage{url}            
\usepackage{booktabs}       
\usepackage{amsfonts}       
\usepackage{nicefrac}       
\usepackage{microtype}      
\usepackage{amsmath}
\usepackage{algorithm}
\usepackage{algorithmic}
\usepackage{multirow}
\usepackage{array}
\usepackage{booktabs}
\usepackage{graphicx}
\usepackage{svg}
\usepackage{paralist}
\usepackage{subfig}
\usepackage{float}
\usepackage[stable]{footmisc}
\usepackage{wrapfig}
\usepackage{floatflt}
\usepackage{tikz}
\usepackage{enumitem}
\setlist{leftmargin=5.3mm}

\usepackage[english]{babel}
\usepackage{amsthm}
\usepackage{mathtools}
\usepackage{bm}
\usepackage{amsmath}
\usepackage{xcolor}
\usepackage{array}

\newcolumntype{P}[1]{>{\centering\arraybackslash}p{#1}}
\newcolumntype{M}[1]{>{\centering\arraybackslash}m{#1}}

\renewcommand{\eqref}[1]{(\ref{#1})}

\newtheorem{innercustomgeneric}{\customgenericname}
\providecommand{\customgenericname}{}
\newcommand{\newcustomtheorem}[2]{%
	\newenvironment{#1}[1]
	{%
		\renewcommand\customgenericname{#2}%
		\renewcommand\theinnercustomgeneric{##1}%
		\innercustomgeneric
	}
	{\endinnercustomgeneric}
}

\newcustomtheorem{customthm}{Theorem}
\newcustomtheorem{customlemma}{Lemma}

\title{Truly Deterministic Policy Optimization}

\author{%
	Ehsan Saleh\textsuperscript{1}, Saba Ghaffari\textsuperscript{1}, Timothy Bretl\textsuperscript{1,2}, Matthew West\textsuperscript{3}\\
	\textsuperscript{1}Department of Computer Science\\
	\textsuperscript{2}Department of Aerospace Engineering\\
	\textsuperscript{3}Department of Mechanical Science and Engineering\\
	University of Illinois Urbana-Champaign\\
	\texttt{ehsans2,sabag2,tbretl,mwest@illinois.edu} \\
}

\newcommand{\insertnlrfig}[2]{
	\begin{figure}[!htbp]
		\includegraphics[width=\linewidth]{figures/supplements/nlrbench/nlr_return_vs_samples_v#1}
		\includegraphics[width=\linewidth]{figures/supplements/nlrbench/nlr_angle_vs_time_freq_v#1}
		\caption{Results for the #2 variant of the simple pendulum with non-local rewards. Upper panel: training curves with empirical discounted payoffs. Lower panels: trajectories in both the time domain and frequency domain, showing target values of oscillation frequency, amplitude, and offset.}
		\label{fig:nlrv#1}
	\end{figure}
}

\DeclarePairedDelimiter{\ip}\langle\rangle
\newcommand{\EE}{\mathbb{E}}
\newcommand{\PP}{\mathbb{P}}

\newtheorem{theorem}{Theorem}[section]

\newtheorem{lemma}[theorem]{Lemma}
\newcommand{\dth}{\delta\theta}
\DeclareMathOperator*{\Lip}{Lip}
\DeclareMathOperator*{\TV}{TV}

\DeclarePairedDelimiter{\Bip}{\Big\langle}{\Big\rangle}
\DeclarePairedDelimiter{\bip}{\big\langle}{\big\rangle}

\newcommand{\bs}{s}
\newcommand{\ba}{a}

\begin{document}

\vspace{-2mm}
\maketitle
\vspace{-2mm}

\begin{abstract}
	In this paper, we present a policy gradient method that avoids exploratory noise injection and performs policy search over the deterministic landscape. By avoiding noise injection all sources of estimation variance can be eliminated in systems with deterministic dynamics (up to the initial state distribution). Since deterministic policy regularization is impossible using traditional non-metric measures such as the KL divergence, we derive a Wasserstein-based quadratic model for our purposes. We state conditions on the system model under which it is possible to establish a monotonic policy improvement guarantee, propose a surrogate function for policy gradient estimation, and show that it is possible to compute exact advantage estimates if both the state transition model and the policy are deterministic. Finally, we describe two novel robotic control environments---one with non-local rewards in the frequency domain and the other with a long horizon (8000 time-steps)---for which our policy gradient method (TDPO) significantly outperforms existing methods (PPO, TRPO, DDPG, and TD3). Our implementation with all the experimental settings is available at \href{https://github.com/ehsansaleh/code_tdpo}{https://github.com/ehsansaleh/code\_tdpo}.
\end{abstract}

Policy Gradient (PG) methods can be broadly characterized by three defining elements: the policy gradient estimator, the regularization measures, and the exploration profile. For gradient estimation, episodic \citep{williams1992simple}, importance-sampling-based \citep{schulman2015trust}, and deterministic \citep{silver2014deterministic} gradients are some of the most common estimation oracles. As for regularization measures, either an Euclidean distance within the parameter space \citep{williams1992simple,silver2014deterministic,lillicrap2015continuous}, or dimensionally consistent non-metric measures \citep{schulman2015trust,kakade2002approximately,schulman2017proximal,kakade2002natural,wu2017scalable} have been frequently adapted. Common exploration profiles include Gaussian \citep{schulman2015trust} and stochastic processes \citep{lillicrap2015continuous}. These elements form the basis of many model-free and stochastic policy optimization methods successfully capable of learning high-dimensional policy parameters.

Both stochastic and deterministic policy search can be useful in applications. A stochastic policy has the effect of smoothing or filtering the policy landscape, which is desirable for optimization. Searching through stochastic policies has enabled the effective control of challenging environments under a general framework \citep{schulman2015trust,schulman2017proximal}. The same method could either learn robotic movements or play basic games (1) with minimal domain-specific knowledge, (2) regardless of function approximation classes, and (3) with less human intervention (ignoring reward engineering and hyper-parameter tuning) \citep{duan2016benchmarking}. Using stochasticity for exploration, although it imposes approximations and variance, has provided a robust way to actively search for higher rewards. Despite many successes, there are practical environments which remain challenging for current policy gradient methods. For example, non-local rewards (e.g., those defined in the frequency domain), long time horizons, and naturally-resonant environments all occur in realistic robotic systems \citep{golnaraghi2002auto,meirovitch1975elements,preumont2008active} but can present issues for policy gradient search.

To tackle challenging environments such as these, this paper considers policy gradient methods based on deterministic policies and deterministic gradient estimates, which could offer advantages by allowing the estimation of global reward gradients on long horizons without the need to inject noise into the system for exploration. To facilitate a dimensionally consistent and low-variance deterministic policy search, a compatible policy gradient estimator and a metric measure for regularization should be employed. For gradient estimation we focus on Vine estimators \citep{schulman2015trust}, which can be easily applied to deterministic policies. As a metric measure, we use the Wasserstein distance, which can measure meaningful distances between deterministic policy functions that have non-overlapping supports (in contrast to the Kullback-Liebler (KL) divergence and the Total Variation (TV) distance).

The Wasserstein metric has seen substantial recent application in a variety of machine-learning domains, such as the successful stable learning of generative adversarial models \citep{arjovsky2017wasserstein}. Theoretically, this metric has been studied in the context of Lipschitz-continuous Markov decision processes in reinforcement learning \citep{hinderer2005lipschitz,ferns2012metrics}. \citet{pirotta2015policy} defined a policy gradient method using the Wasserestein distance by relying on Lipschitz continuity assumptions with respect to the policy gradient itself. Furthermore, for Lipschitz-continuous Markov decision processes, \citet{asadi2018lipschitz} and \citet{oatao17977} used the Wasserstein distance to derive model-based value-iteration and policy-iteration methods, respectively. On a more practical note, \citet{pacchiano2019wasserstein} utilized Wasserstein regularization for behavior-guided stochastic policy optimization. Moreover,  \citet{abdullah2019wasserstein} has proposed another robust stochastic policy gradient formulation. Estimating the Wasserstein distance for general distributions is more complicated than typical KL-divergences \citep{villani2008optimal}. This fact constitutes and emphasizes the contributions of \citet{abdullah2019wasserstein} and  \citet{pacchiano2019wasserstein}. However, for our deterministic observation-conditional policies, closed-form computation of Wasserstein distances is possible without any approximation.

Existing deterministic policy gradient methods (e.g., DDPG and TD3) use \emph{deterministic policies} \citep{silver2014deterministic,lillicrap2015continuous,fujimoto2018addressing}, meaning that they learn a deterministic policy function from states to actions. However, such methods still use \emph{stochastic search} (i.e., they add stochastic noise to their deterministic actions to force exploration during policy search). In contrast, we will be interested in a method which not only uses \emph{deterministic policies}, but also uses \emph{deterministic search} (i.e., without constant stochastic noise injection). We call this method \emph{truly deterministic policy optimization} (TDPO) and it may have lower estimation variances and better scalability to long horizons, as we will show in numerical examples.

Scalability to long horizons is one of the most challenging aspects of policy gradient methods that use stochastic search. This issue is sometimes referred to as the \emph{curse of horizon} in reinforcement learning \citep{liu2018breaking}. General worst-case analyses suggest that the sample complexity of reinforcement learning is exponential with respect to the horizon length \citep{kakade2003sample,kearns2000approximate,kearns2002sparse}. Deriving polynomial lower-bounds for the sample complexity of reinforcement learning methods is still an open problem \citep{jiang2018open}. Lower-bounding the sample complexity of reinforcement learning for long horizons under different settings and simplifying assumptions has been a topic of theoretical research \citep{dann2015sample,wang2020long}. Some recent work has examined the scalability of importance sampling gradient estimators to long horizons in terms of both theoretical and practical estimator variances \citep{liu2018breaking,kallus2019efficiently,kallus2020statistically}. All in all, long horizons are challenging for all reinforcement learning methods, especially the ones suffering from excessive estimation variance due to the use of stochastic policies for exploration, and our truly deterministic method may have advantages in this respect.

In this paper, we focus on continuous-domain robotic environments with reset capability to previously visited states. The main contributions of this work are: (1) we introduce a Deterministic Vine (DeVine) policy gradient estimator which avoids constant exploratory noise injection; (2) we derive a novel deterministically-compatible surrogate function and provide monotonic payoff improvement guarantees; (3) we show how to use the DeVine policy gradient estimator with the Wasserstein-based surrogate in a practical algorithm (TDPO: Truly Deterministic Policy Optimization); (4) we illustrate the robustness of the TDPO policy search process in robotic control environments with non-local rewards, long horizons, and/or resonant frequencies.

\section{Background} \label{sec:prelim}

\paragraph{MDP preliminaries.}
An infinite-horizon discounted Markov decision process (MDP) is specified by $(\mathcal{S}, \mathcal{A}, P, R, \mu, \gamma)$, where $\mathcal{S}$ is the state space, $\mathcal{A}$ is the action space, $P: \mathcal{S}\times\mathcal{A}\to\Delta(\mathcal{S})$ is the transition dynamics, $R: \mathcal{S}\times\mathcal{A}\to[0, R_{\max}]$ is the reward function, $\gamma \in [0, 1)$ is the discount factor, and $\mu(s)$ is the initial state distribution of interest (where $\Delta(\mathcal{F})$ denotes the set of all probability distributions over $\mathcal{F}$, otherwise known as the Credal set of $\mathcal{F}$). The transition dynamics $P$ is defined as an operator which produces a distribution over the state space for the next state $s' \sim P(s,a)$. The transition dynamics can be easily generalized to take distributions of states or actions as input (i.e., by having $P$ defined as $P: \Delta(\mathcal{S})\times\mathcal{A}\to\Delta(\mathcal{S})$ or $P: \mathcal{S}\times\Delta(\mathcal{A})\to\Delta(\mathcal{S})$). We may abuse the notation and replace $\delta_s$ and $\delta_a$ by $s$ and $a$, where $\delta_s$ and $\delta_a$ are the deterministic distributions concentrated at the state $s$ and action $a$, respectively. A policy $\pi: \mathcal{S}\to\Delta(\mathcal{A})$ specifies a distribution over actions for each state, and induces trajectories from a given starting state $s$ as follows: $s_1=s$, $a_1\sim \pi(s_1)$, $r_1 = R(s_1, a_1)$, $s_2 \sim P(s_2, a_2)$, $a_2\sim \pi(s_2)$, etc. We will denote trajectories as state-action tuples $\tau = (s_1,a_1,s_2,a_2, \ldots)$. One can generalize the dynamics (1) by using a policy instead of an action distribution $\PP(\mu_s, \pi) := \EE_{s\sim \mu_s}[\EE_{a\sim \pi(s)}[P(s,a)]]$, and (2) by introducing the $t$-step transition dynamics recursively as $\PP^t(\mu_s,\pi):=\PP(\PP^{t-1}(\mu_s, \pi), \pi)$ with $\PP^0(\mu_s,\pi):=\mu_s$, where $\mu_s$ is a distribution over $\mathcal{S}$. The visitation frequency can be defined as $\rho_{\mu}^{\pi}:=(1-\gamma)\sum_{t=1}^\infty \gamma^{t-1}\PP^{t-1}(\mu, \pi)$. Table~\ref{tab:mdpnotation} of the Supplementary Material summarizes all MDP notation.

The value function of $\pi$ is defined as $V^\pi(s):= \EE[\sum_{t=1}^\infty \gamma^{t-1} r_t \mid s_1 = s; \pi]$. Similarly, one can define $Q^\pi(s,a)$  by conditioning on the first action. The advantage function can then be defined as their difference (i.e. $A^\pi(s,a):= Q^\pi(s,a) - V^\pi(s)$). Generally, one can define the advantage/value of one policy with respect to another using $A^\pi(s,\pi'):= \EE[Q^\pi(s,a) - V^\pi(s) \mid a\sim \pi'(\cdot|s)]$ and $Q^\pi(s,\pi'):= \EE[Q^\pi(s,a) \mid a\sim \pi'(\cdot|s)]$.
Finally, the payoff of a policy $\eta_\pi := \EE[V^\pi(s); s\sim\mu]$ is the average value over the initial states distribution of the MDP. 

\paragraph{Probabilistic and mathematical notations.} Sometimes we refer to $\int f(x)g(x)dx$ integrals as $\langle f, g\rangle_x$ Hilbert space inner products. Assuming that $\zeta$ and $\nu$ are two probabilistic densities, the Kulback-Liebler (KL) divergence is $\KL(\zeta|\nu):=\langle \zeta(x), \log(\frac{\zeta(x)}{\nu(x)})\rangle_x$, the Total-Variation (TV) distance is $\TV(\zeta,\nu)=:\frac{1}{2}\langle |\zeta(x)-\nu(x)|, 1\rangle_x$, and the Wasserstein distance is $W(\zeta, \nu)=\inf_{\gamma \in \Gamma(\zeta, \nu)} \langle \|x-y\|, \gamma(x,y) \rangle_{x,y}$ where $\Gamma(\zeta, \nu)$ is the set of couplings for $\zeta$ and $\nu$. We define $\Lip(f(x,y);x):=\sup_{x} \|\nabla_x f(x,y)\|_2$ and assume the existence of $\Lip(Q^{\pi}(\bs, a);a)$ and $\| \Lip(\nabla_s Q^{\pi}(s, a); a) \|_2$ constants. Under this notation, the Rubinstein-Kantrovich (RK) duality states that the $|\ip{\zeta(x)-\nu(x), f(x)}_x| \leq  W(\zeta, \nu) \cdot \Lip(f;x)$ bound is tight for all $f$. For brevity, we may abuse the notation and denote $\sup_{s}W(\pi_1(\cdot|s), \pi_2(\cdot|s))$ with $W(\pi_1, \pi_2)$ (and similarly for other measures). For parameterized policies, we define $\nabla_{\pi} f(\pi):=\nabla_{\theta} f(\pi)$ where $\pi$ is parameterized by the vector $\theta$. Table~\ref{tab:mathnotation} of the Supplementary Material summarizes all these mathematical definitions.

\paragraph{Policy gradient preliminaries.}
The advantage decomposition lemma provides insight into the relationship between payoff improvements and advantages~\citep{kakade2002approximately}. That is,
\begin{equation}
\eta_{\pi_2} - \eta_{\pi_1} = \frac{1}{1-\gamma}\cdot\EE_{s\sim \rho_\mu^{\pi_2}}[A^{\pi_1}(s, \pi_2)].
\end{equation}

We will denote the current and the candidate next policy as $\pi_1$ and $\pi_2$, respectively. Taking derivatives of both sides with respect to $\pi_2$ at $\pi_1$ yields
\begin{equation}
\nabla_{\pi_2} \eta_{\pi_2} = \frac{1}{1-\gamma}\bigg[\langle \nabla_{\pi_2}\rho_\mu^{\pi_2}(\cdot), A^{\pi_1}(\cdot, \pi_1)\rangle + \langle \rho_\mu^{\pi_1}(\cdot), \nabla_{\pi_2} A^{\pi_1}(\cdot, \pi_2)\rangle \bigg].
\end{equation}
Since $\pi_1$ does not have any advantage over itself (i.e., $A^{\pi_1}(\cdot, \pi_1)=0$), the first term is zero. Thus, the Policy Gradient (PG) theorem is derived as
\begin{equation}
\nabla_{\pi_2} \eta_{\pi_2}\Big|_{\pi_2=\pi_1} = \frac{1}{1-\gamma}\cdot\EE_{s\sim \rho_\mu^{\pi_1}}[\nabla_{\pi_2} A^{\pi_1}(s, \pi_2)]\Big|_{\pi_2=\pi_1}.
\end{equation}
For policy iteration with function approximation, we assume $\pi_2$ and $\pi_1$ to be parameterized by $\theta_2$ and $\theta_1$, respectively. One can view the PG theorem as a Taylor expansion of the payoff at $\theta_1$.

A brief introduction of the Conservative Policy Iteration (CPI) \citep{kakade2002approximately}, the Trust Region Policy Optimization (TRPO) \citep{schulman2015trust}, the Proximal Policy Optimization (PPO) \citep{schulman2015high}, the Deep Deterministic Policy Gradient (DDPG) \citep{lillicrap2015continuous}, and the Twin-Delayed Deterministic Policy Gradient (TD3) \citep{fujimoto2018addressing}  policy gradient methods is left to the Supplementary Material. Whether using deterministic policy gradients (e.g., DDPG and TD3) or stochastic policy gradients (e.g., TRPO and PPO), all these methods still perform stochastic search by adding stochastic noise to the deterministic policies to force exploration.

\paragraph{Reinforcement learning challenges.} Non-local rewards and soft horizon scalability are two major challenges in reinforcement learning. Due to space constraints, we leave the discussion of these challenges with simple and intuitive numerical examples to the Supplementary Material. The rest of the paper assumes the reader's familiarity with these concepts.

\section{Monotonic Policy Improvement Guarantee}

We use the Wasserstein metric because it allows the effective measurement of distances between probability distributions or functions with non-overlapping support, such as deterministic policies, unlike the KL divergence or TV distance which are either unbounded or maximal in this case. The physical transition model's smoothness enables the use of the Wasserstein distance to regularize deterministic policies. Therefore, we make the following two assumptions about the transition model:
\begin{align}
W(\PP(\mu, \pi_1), \PP(\mu, \pi_2))&\leq L_\pi \cdot W(\pi_1, \pi_2) \label{eq:piassumption}, \\
W(\PP(\mu_1, \pi), \PP(\mu_2, \pi))&\leq L_\mu \cdot W(\mu_1, \mu_2) \label{eq:sassumption}.
\end{align}

Also, we make the dynamics stability assumption $\sup \sum_{k=1}^t \hat{L}_{\mu, \pi_1, \pi_2}^{(k-1)} \prod_{i=k+1}^{t-1} \tilde{L}_{\mu, \pi_1, \pi_2}^{(i)} < \infty$, with the definitions of the new constants and further discussion of the implications deferred to the Supplementary Material where we also discuss Assumptions~\ref{eq:piassumption} and~\ref{eq:sassumption} and the existence of other Lipschitz constants which appear as coefficients in the final lower bound.

The advantage decomposition lemma can be rewritten as
\begin{equation}
\eta_{\pi_2} = \eta_{\pi_1} + \frac{1}{1-\gamma}\cdot \EE_{s\sim \rho_\mu^{\pi_1}}[A^{\pi_1}(s, \pi_2)] + \frac{1}{1-\gamma}\cdot \ip{\rho^{\pi_2}_{\mu} - \rho^{\pi_1}_{\mu}, A^{\pi_1}(\cdot, \pi_2)}_s.
\label{eq:advdecompmainpaper}
\end{equation}
The $\ip{\rho^{\pi_2}_{\mu} - \rho^{\pi_1}_{\mu}, A^{\pi_1}(\cdot, \pi_2)}$ term has zero gradient at $\pi_2=\pi_1$, which qualifies it to be crudely called ``the second-order term''. We dedicate a full section of our Supplementary Material to the theoretical derivations and proofs necessary to lower-bound this second-order term into an objective well-suited form for practical optimization. Next, we present the theoretical bottom line and the final bound:
\begin{align}
\label{eq:wtrpompibound}
\mathcal{L}_{\pi_1}^{\sup}(\pi_2) &= \eta_{\pi_1} +\frac{1}{1-\gamma} \mathbb{E}_{s\sim \rho^{\pi_1}_\mu}[A^{\pi_1}(s,\pi_2)] - C_2 \cdot \sup_{s} \bigg[ W(\pi_2(a|s), \pi_1(a|s))^2  \bigg] \nonumber\\
\qquad - &C_1 \cdot \sup_{s} \bigg[ \bigg\| \nabla_{s'} W\bigg(\frac{\pi_2(a|s') + \pi_1(a|s)}{2}, \frac{\pi_2(a|s) + \pi_1(a|s')}{2}\bigg)\bigg|_{s'=s} \bigg\|_2^2 \bigg].
\end{align}


For brevity, we denote the $\big\| \nabla_{s'} W(\cdots)\big|_{s'=s} \big\|_2^2$ expression as $\mathcal{L}_{G^2}(\pi_1, \pi_2; s)$ in the rest of the paper. We have $\eta_{\pi_2} \geq \mathcal{L}_{\pi_1}^{\sup}(\pi_2)$ and $\mathcal{L}_{\pi_1}^{\sup}(\pi_1) = \eta_{\pi_1}$. This facilitates the application of Theorem~\ref{thm:nondecreasepayoff} as an instance of Minorization-Maximization algorithms \citep{hunter2004tutorial}.

\begin{theorem}
	\label{thm:nondecreasepayoff}
	Successive maximization of $\mathcal{L}^{\sup}$ yields non-decreasing policy payoffs.
\end{theorem}
\textit{Proof.}
	With $\pi_2=\argmax_{\pi} \mathcal{L}_{\pi_1}^{\sup}(\pi)$, we have $\mathcal{L}_{\pi_1}^{\sup}(\pi_2) \geq \mathcal{L}_{\pi_1}^{\sup}(\pi_1)$. Thus,
	\begin{equation}\eta_{\pi_2} \geq \mathcal{L}_{\pi_1}^{\sup}(\pi_2) \text{ and } \eta_{\pi_1} = \mathcal{L}_{\pi_1}^{\sup}(\pi_1) \Longrightarrow \eta_{\pi_2} - \eta_{\pi_1} \geq \mathcal{L}_{\pi_1}^{\sup}(\pi_2) - \mathcal{L}_{\pi_1}^{\sup}(\pi_1) \geq 0. \quad \qedsymbol{}\end{equation}

\begin{algorithm}[t]
	\caption{Truly Deterministic Policy Optimization (TDPO)}
	\begin{algorithmic}[1]
		\label{alg:wppoalg}
		\REQUIRE Initial policy $\pi_0$.
		\REQUIRE Advantage estimator and sample collector oracle $\mathbb{A}^{\pi}$.
		\FOR{$k = 1, 2, \ldots$}
		\STATE Collect trajectories and construct the advantage estimator oracle $\mathbb{A}^{\pi_k}$.
		\STATE Compute the policy gradient $g$ at $\theta_k$ :
		$\displaystyle g\leftarrow \nabla_{\theta'}\mathbb{A}^{\pi_k}(\pi') |_{\pi'=\pi_k}$
		\STATE Construct a surrogate Hessian vector product oracle $v \rightarrow H\cdot v$ such that for $\theta' = \theta_k + \delta\theta$,
		\begin{equation}
		\mathbb{E}_{s\sim \rho^{\pi_k}_\mu} \bigg[ W(\pi'(a|s), \pi_k(a|s))^2 \bigg] + \frac{C_1}{C_2} \mathbb{E}_{s\sim \rho^{\pi_k}_\mu} \bigg[ \mathcal{L}_{G^2}(\pi', \pi_k; s) \bigg] = \frac{1}{2} \delta\theta^T H \delta\theta + \text{h.o.t.},
		\end{equation}
		where $\text{h.o.t.}$ denotes higher order terms in $\delta\theta$.
		\STATE Find the optimal update direction $\delta\theta^* = H^{-1} g$ using the Conjugate Gradient algorithm.
		\STATE \textbf{(Basic Variant)} Determine the best step size $\alpha^*$ within the trust region:
		\begin{align}
		\alpha^* = &\argmax_{\alpha} g^T(\alpha \delta\theta^*) - \frac{C_2}{2} (\alpha \delta\theta^*)^T H (\alpha \delta\theta^*) \nonumber\\
		&\text{s.t.} \qquad \frac{1}{2} (\alpha^* \delta\theta^*)^T H (\alpha^* \delta\theta^*) \leq \delta_{\max}^2
		\end{align}
		\STATE \textbf{(Advanced Variant)} Determine the best step size $\alpha^*$ using a line-search procedure and pick the best one; each coefficient's performance can be evaluated by sampling from the environment.
		\STATE \textbf{(Advanced Variant)} Update the exploration scale in $\mathbb{A}^{\pi}$ using the collected samples.
		\STATE Update the policy parameters: $\theta_{k+1} \leftarrow \theta_k + \alpha^* \delta\theta^*$.
		\ENDFOR
	\end{algorithmic}
\end{algorithm}


Successive optimization of $\mathcal{L}_{\pi_1}^{\sup}(\pi_2)$ generates non-decreasing payoffs. However, it is impractical due to the large number of constraints and statistical estimation of maximums. To mitigate this, we take a similar approach to TRPO and replace the maximums with expectations over the observations.


The coefficients $C_1$ and $C_2$ are dynamics-dependent. In the \textit{basic variant} of our method, we used constant coefficients and a trust region. This yields the Truly Deterministic Policy Optimization (TDPO) as given in Algorithm~\ref{alg:wppoalg}. See the Supplementary Material for practical notes on the manual choice of $C_1$ and $C_2$. Alternatively, one could adopt processes similar to \citet{schulman2015trust} where the IS-based advantage estimator used a line search for proper step size selection, or the adaptive penalty coefficient setting in \citet{schulman2017proximal}. In the \textit{advanced variant} of our method, we implement such a line search procedure by collecting samples from the environment and picking the coefficient yielding the best improvement. This comes at an increase in the sample complexity, but the benefits can outweigh the added costs. Furthermore, the exploration scale in the sampling oracle is adaptively tuned using the collected payoffs in the advanced variant; by constructing an importance sampling derivative estimate for the exploration scale parameter, stochastic gradient descent can be used to tune the exploration scale individually.

\begin{algorithm}[t]
	\caption{Deterministic Vine (DeVine) Policy Advantage Estimator}
	\begin{algorithmic}[1]
		\label{alg:vineAdv}
		\REQUIRE The number of parallel workers $K$
		\REQUIRE A policy $\pi$, an exploration policy $q$, discrete time-step distribution $\nu(t)$, initial state distribution $\mu(s)$, and the discount factor $\gamma$.
		\STATE Sample an initial state $s_0$ from $\mu$, and then roll out a trajectory $\tau=(s_0,a_0,s_1,a_1, \cdots)$ using $\pi$.
		\FOR{$k = 1, 2, \cdots, K$}
		\STATE Sample the integer number $t=t_k$ from $\nu$.
		\STATE Compute the value $V^{\pi_1}(s_t) = \sum_{i=t}^{\infty} \gamma^{t-i} R(s_i, a_i)$.
		\STATE Reset the initial state to $s_t$, sample the first action $a'_t$ according to $q(\cdot|s_t)$, and use $\pi$ for the rest of the trajectory. This will create $\tau'=(s_t,a_t',s_{t+1}',a_{t+1}', \cdots)$.
		\STATE Compute the value $Q^{\pi_1}(s_t, a_t') = \sum_{i=t}^{\infty} \gamma^{t-i} R(s_i', a_i')$.
		\STATE Compute the advantage $A^{\pi_1}(s_t, a_t') = Q^\pi(s_t, a_t') - V^\pi(s_t)$.
		\ENDFOR
		\STATE Define $\displaystyle \mathbb{A}^{\pi_1}(\pi_2):=\frac{1}{K}\sum_{k=1}^K \frac{\dim(\mathcal{A})\cdot\gamma^{t_k}}{\nu(t_k)} \cdot \frac{(\pi_2(s)-a_{t_k})^T(a_{t_k}'-a_{t_k})}{(a_{t_k}'-a_{t_k})^T(a_{t_k}'-a_{t_k})} \cdot A^{\pi_1}(s_{t_k}, a_{t_k}').$
		\STATE Return $\mathbb{A}^{\pi_1}(\pi_2)$ and $\nabla_{\pi_2} \mathbb{A}^{\pi_1}(\pi_2)$ as unbiased estimators for $E_{s\sim \rho_\mu^{\pi_1}}[A^{\pi_1}(s, \pi_2)]$ and the PG.
	\end{algorithmic}
\end{algorithm}

\subsection{On the Interpretation of the Surrogate Function}

For deterministic policies, the squared Wasserstein distance $W(\pi_2(a|s), \pi_1(a|s))^2$ degenerates to the Euclidean distance over the action space. Any policy defines a sensitivity matrix at a given state $s$, which is the Jacobian matrix of the policy output with respect to $s$. The policy sensitivity term $\mathcal{L}_{G^2}(\pi_1, \pi_2; s)$ is essentially the squared Euclidean distance over the action-to-observation Jacobian matrix elements. In other words, our surrogate prefers to step in directions where the action-to-observation sensitivity is preserved within updates. 

Although our surrogate uses a metric distance instead of the traditional non-metric measures for regularization, we do not consider this sole replacement a major contribution. The squared Wasserestein distance and the KL divergence of two identically-scaled Gaussian distributions are the same up to a constant (i.e., $\KL(\mathcal{N}(m_1, \sigma)\| \mathcal{N}(m_2, \sigma)) = W(\mathcal{N}(m_1, \sigma), \mathcal{N}(m_2, \sigma))^2 / 2\sigma^2 $). On the other hand, our surrogate's compatibility with deterministic policies makes it a valuable asset for our policy gradient algorithm; both $W(\pi_2(a|s), \pi_1(a|s))^2$ and $\mathcal{L}_{G^2}(\pi_1, \pi_2; s)$ can be evaluated for two deterministic policies $\pi_1$ and $\pi_2$ numerically without any approximations to overcome singularities.

\section{Model-Free Estimation of Policy Gradient}
The DeVine advantage estimator is formally defined in Algorithm~\ref{alg:vineAdv}. Unlike DDPG and TD3, the DeVine estimator allows our method to perform \emph{deterministic search} by not consistently injecting noise in actions for exploration. Under deterministic dynamics and policies, if DeVine samples each dimension at each time-step exactly once then in the limit of small exploration scale $\sigma$ it can produce exact advantages, as stated in Theorem~\ref{thm:accuratevinepg}, whose proof is deferred to the Supplementary Material.


\begin{figure}[t]
  \includegraphics[width=\linewidth]{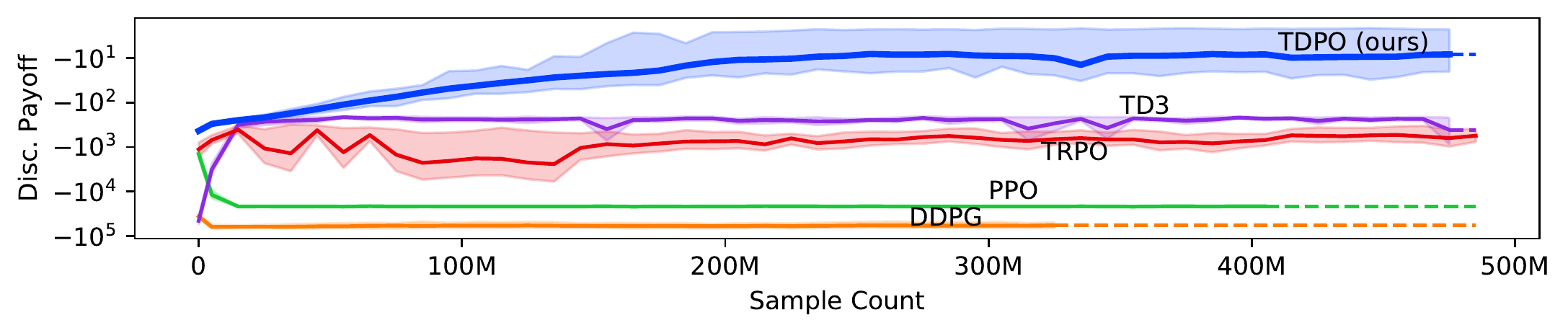}
  \includegraphics[width=\linewidth]{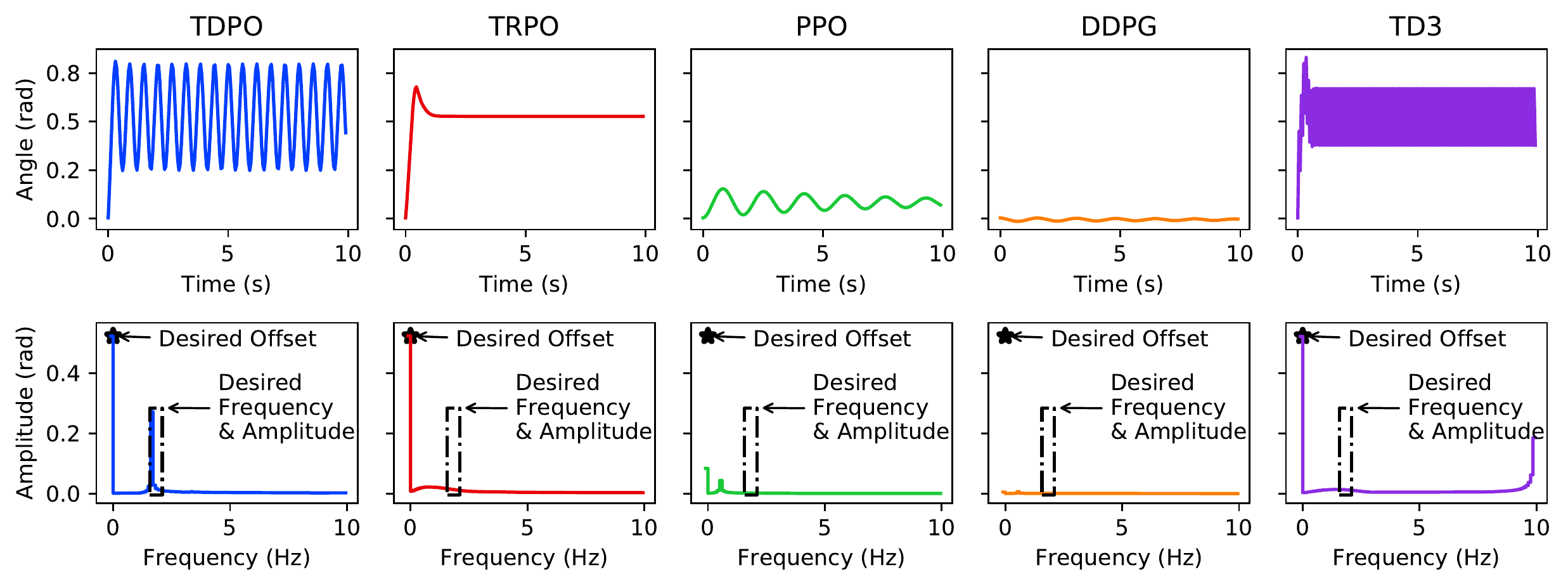}
  \caption{Results for the simple pendulum with non-local rewards. Upper panel: training curves with empirical discounted payoffs. Lower panels: trajectories in both the time domain and frequency domain, showing target values of oscillation frequency, amplitude, and offset. The basic variant of our method (non-adaptive exploration scales and update coefficients) was used in this experiment.}
	\label{fig:nlr}
\end{figure}

\begin{theorem}
	\label{thm:accuratevinepg}
	Assume a finite horizon MDP with both deterministic transition dynamics $P$ and initial distribution $\mu$, with maximal horizon length of $H$. Define $K=H\cdot \dim(\mathcal{A})$, a uniform $\nu$, and $q(s;\sigma)=\pi_1(s) + \sigma \mathbf{e}_j$ in Algorithm~\ref{alg:vineAdv} with $\mathbf{e}_j$ being the $j^{th}$ basis element for $\mathcal{A}$. If the $(j,t_k)$ pairs are sampled to exactly cover $\{1,\cdots,\dim(\mathcal{A})\} \times \{1,\cdots,H\}$, then we have
	\begin{equation}
	\lim_{\sigma\rightarrow 0}\nabla_{\pi_2} \mathbb{A}^{\pi_1}(\pi_2)\big|_{\pi_2=\pi_1} = \nabla_{\pi_2} \eta_{\pi_2}\big|_{\pi_2=\pi_1} .
	\end{equation}
\end{theorem}

Theorem~\ref{thm:accuratevinepg} provides a guarantee for recovering the exact policy gradient if the initial state distribution was deterministic and all time-steps of the trajectory were used to branch vine trajectories. Although this theorem sets the stage for computing a fully deterministic gradient, stochastic approximation can be used in Algorithm~\ref{alg:vineAdv} by randomly sampling a small set of states for advantage estimation. In other words, Theorem~\ref{thm:accuratevinepg} would use $\nu$ to deterministically sample all trajectory states, whereas this is not a practical requirement for Algorithm~\ref{alg:vineAdv} and the gradients are still unbiased if a random set of vine branches is used.

The DeVine estimator can be advantageous in at least two scenarios. First, in the case of rewards that cannot be decomposed into summations of immediate rewards. For example, overshoot penalizations or frequency-based rewards as used in robotic systems are non-local. DeVine can be robust to non-local rewards as it is insensitive to whether the rewards were applied immediately or after a long period. Second, DeVine can be an appropriate choice for systems that are sensitive to the injection of noise, such as high-bandwidth robots with natural resonant frequencies. In such cases, using white (or colored) noise for exploration can excite these resonant frequencies and cause instability, making learning difficult. DeVine avoids the need for constant noise injection.


\begin{figure}[t]
	\includegraphics[width=0.98\linewidth]{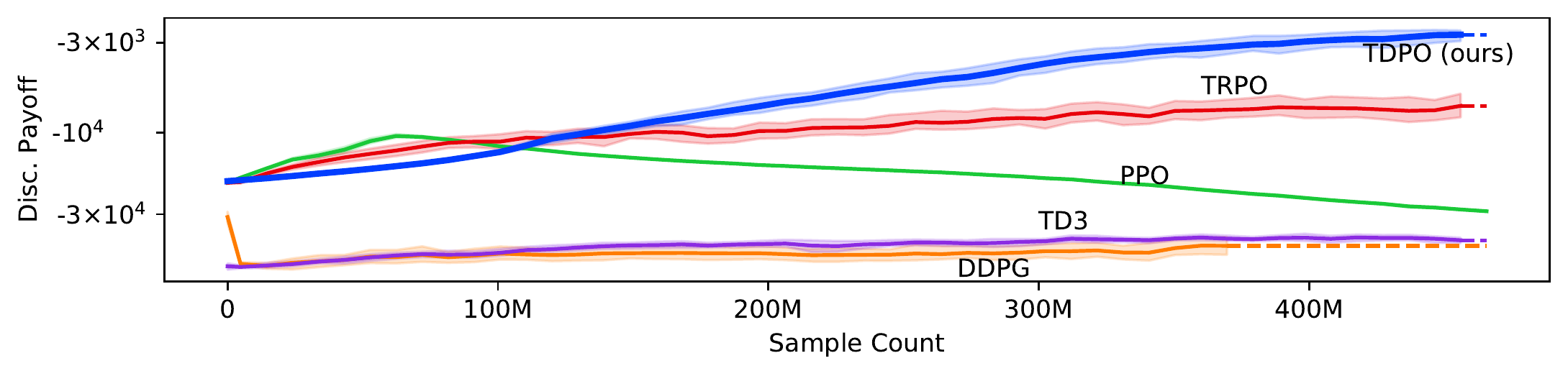}
	\includegraphics[width=0.98\linewidth]{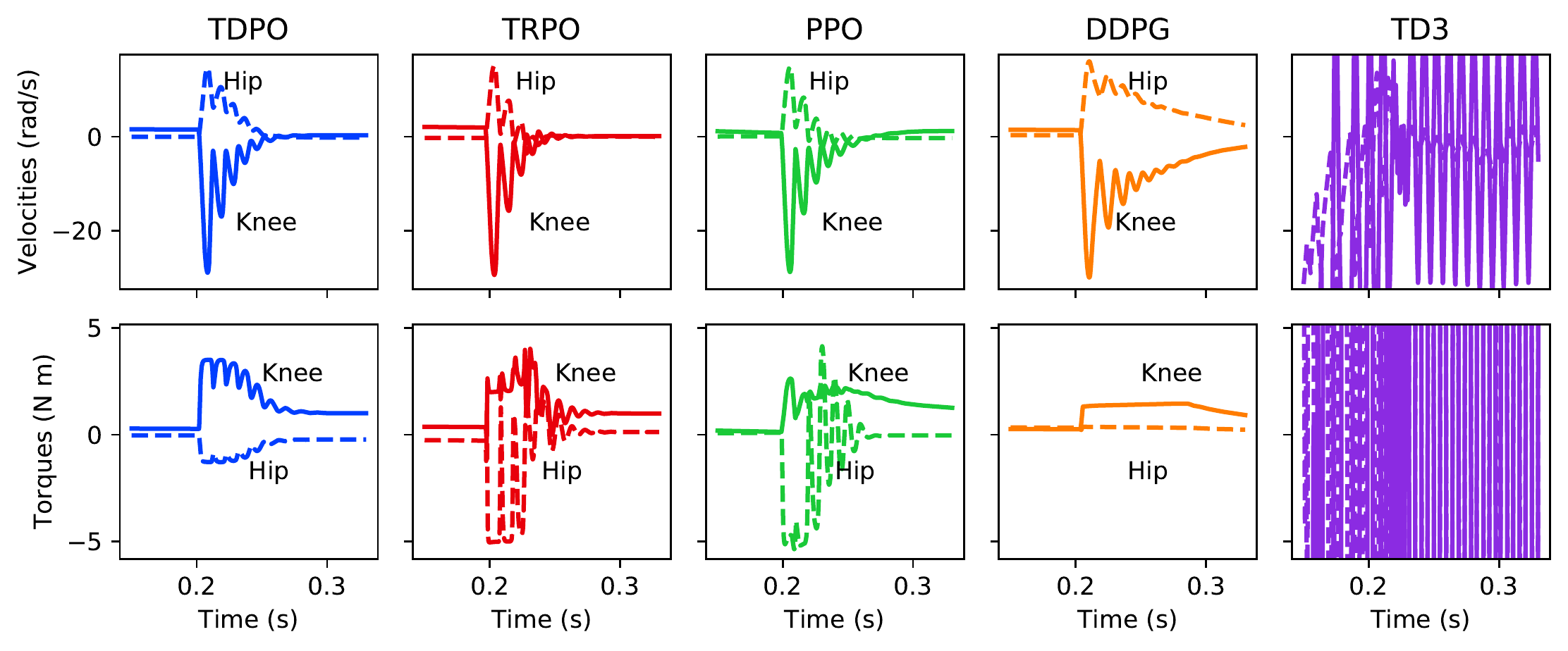}
	\caption{Results for the leg environment with a long horizon and resonant frequencies due to ground compliance.  Upper panel: training curves with empirical discounted payoffs. Lower panel: partial trajectories, restricted to times shortly before and after impact with the ground. Note the oscillations at about $100\;\text{Hz}$ that appear just after the impact at 0.2\ s---these oscillations are evidence of a resonant frequency. The basic variant of our method (non-adaptive exploration scales and update coefficients) was used in this experiment.}
	\label{fig:leg}
	\vspace{-0mm}
\end{figure}

\section{Experiments}
The next three subsections show challenging robotic control tasks including frequency-based non-local rewards, long horizons, and sensitivity to resonant frequencies. In Sections~\ref{sec:nlrsec} and~\ref{sec:legsimple}, we use the basic variant of our method (i.e., fixed exploration scale and update coefficient hyper-parameters throughout the training). This will facilitate a better understanding of our core method's capabilities without any additional tweaks. See the Supplementary Material for a comparison on traditional gym environments, where the basic variant of TDPO works similarly to existing methods. Section~\ref{sec:legsto} includes the most difficult setting in our paper, where we use the advanced variant of our method (i.e., with line-search the update coefficient and adaptive exploration scales).

\subsection{An Environment with Non-Local Rewards \footnote{Non-local rewards are reward functions of the entire trajectory whose payoffs cannot be decomposed into the sum of terms such as $\eta=\sum_{t} f_t(s_t,a_t)$, where functions $f_t$ only depend on nearby states and actions. An example non-local reward is one that depends on the Fourier transform of the complete trajectory signal.}}\label{sec:nlrsec}

The first environment that we consider is a simple pendulum.
The transition function is standard---the states are joint angle and joint velocity, and the action is joint torque.
The reward function is non-standard---rather than define a local reward in the time domain with the goal of making the pendulum stand upright (for example), we define a non-local reward in the frequency domain with the goal of making the pendulum oscillate with a desired frequency and amplitude about a desired offset.
In particular, we compute this non-local reward by taking the Fourier transform of the joint angle signal over the entire trajectory and by penalizing differences between the resulting power spectrum and a desired power spectrum.
We apply this non-local reward at the last time-step of the trajectory. Implementation details and similar results for more pendulum variants are left to the Supplementary Material.

Figure \ref{fig:nlr} shows training curves for TDPO (our method) as compared to TRPO, PPO, DDPG, and TD3. These results were averaged over 25 experiments in which the desired oscillation frequency was $1.7\;\text{Hz}$ (different from the pendulum's natural frequency of $0.5\;\text{Hz}$), the desired oscillation amplitude was $0.28\;\text{rad}$, and the desired offset was $0.52\;\text{rad}$. Figure \ref{fig:nlr} also shows trajectories obtained by the best agents from each method. TDPO (our method) was able to learn high-reward behavior and to achieve the desired frequency, amplitude, and offset. TRPO was able to learn the correct offset but did not produce any oscillatory behavior. TD3 also learned the correct offset, but could not produce desirable oscillations. PPO and DDPG failed to learn any desired behavior.



\subsection{An Environment with Long Horizon and Resonant Frequencies\footnote{Resonant frequencies are a concept from control theory. In the frequency domain, signals of certain frequencies are excited more than others when applied to a system. This is captured by the frequency-domain transfer function of the system, which may have a peak of magnitude greater than one. The resonant frequency is the frequency at which the frequency-domain transfer function has the highest amplitude. Common examples of systems with a resonant frequency include the undamped pendulum, which oscillates at its natural frequency, and RLC circuits which have characteristic frequencies at which they are most excitable. See Chapter 8 of \citet{golnaraghi2002auto} for more information.}}\label{sec:legsimple}

\begin{figure}[t]
	\includegraphics[width=0.33\linewidth]{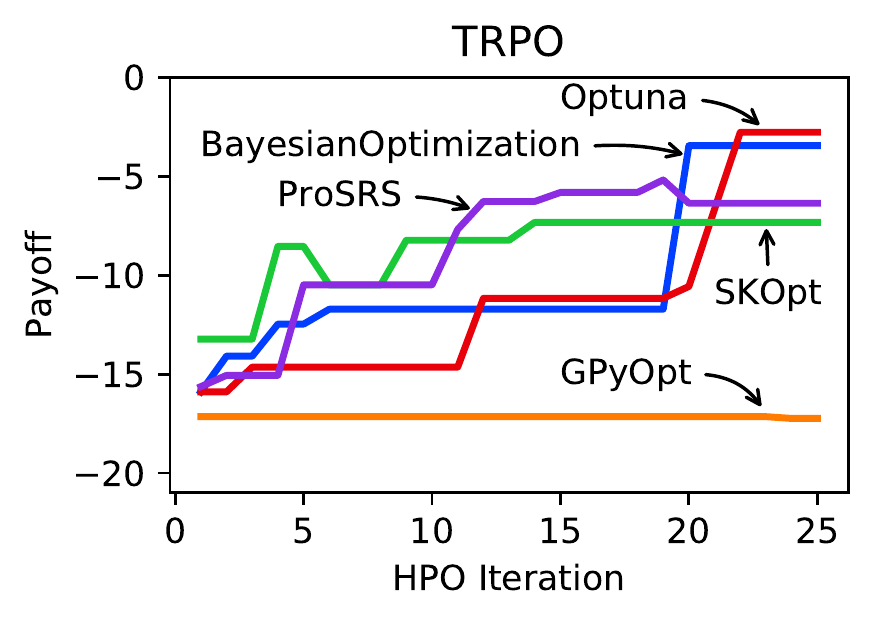}
	\includegraphics[width=0.33\linewidth]{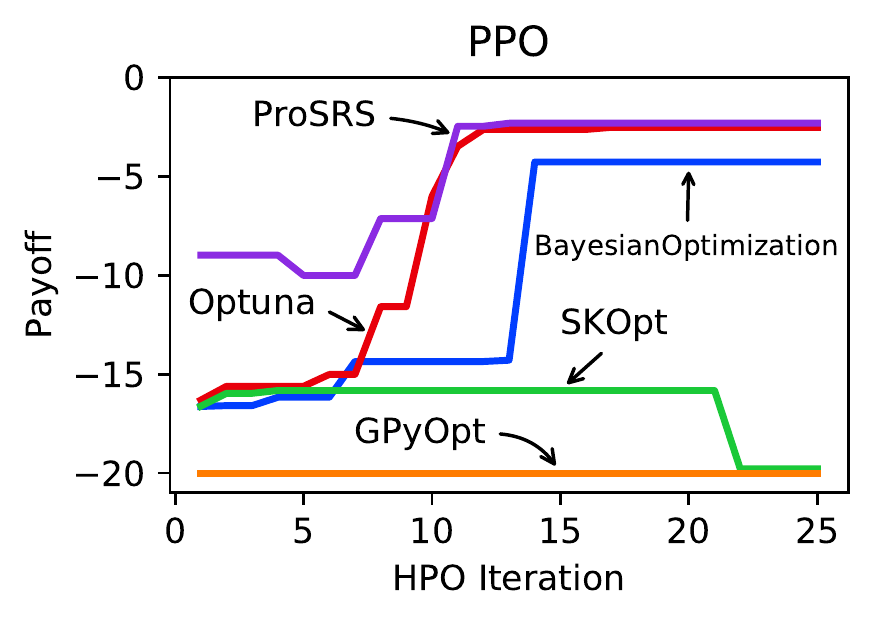}
	\includegraphics[width=0.33\linewidth]{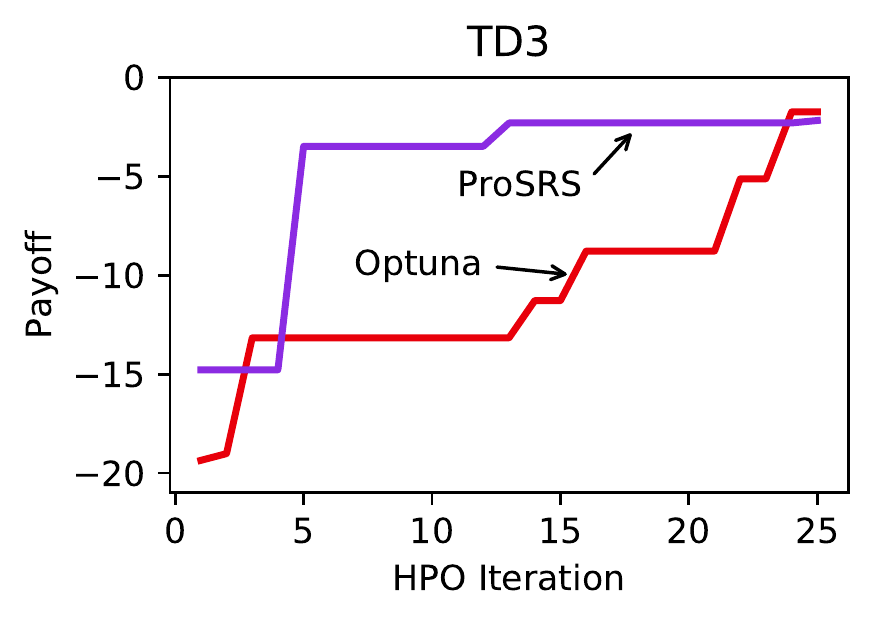}
	\caption{The best payoff vs. the Hyper-Parameter Optimization (HPO) iteration on a short-horizon variant of the legged robotic environment. The HPOs are performed for each of the TRPO, PPO, and TD3 methods in a separate panel. DDPG is a special case of TD3 with HPO. Since TD3 was considerably more expensive, we only show Optuna and ProSRS for it.}
	\label{fig:hpopicking}
\end{figure}
\begin{figure}[t]
	\includegraphics[width=0.98\linewidth]{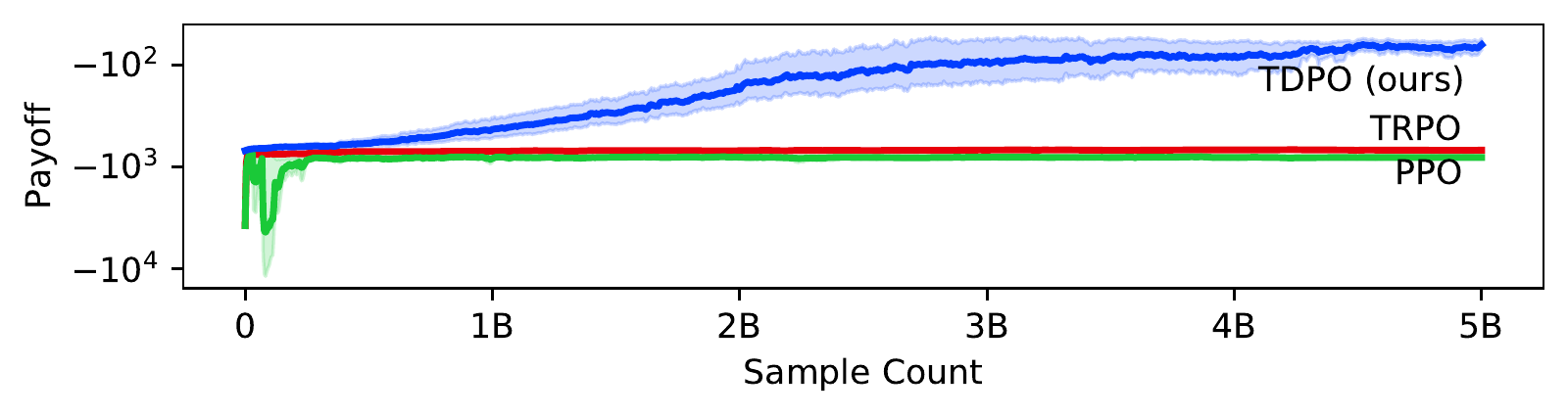}
	\caption{Post Hyper-Parameter Optimization (HPO) training curves with the best settings found for TRPO and PPO compared to the advanced variant of our method (TDPO with adaptive exploration scales and line search). TD3 had a significantly poor performance in the initial parameter sweeps. Due to resource limitations and poor initial performance, we excluded TD3 from this experiment.}
	\label{fig:legsto}
	\vspace{-0mm}
\end{figure}

The second environment that we consider is a single leg from a quadruped robot \citep{park2017high}. This leg has two joints, a ``hip'' and a ``knee,'' about which it is possible to exert torques. The hip is attached to a slider that confines motion to a vertical line above flat ground. We assume the leg is dropped from some height above the ground and the task is to recover from this drop and to stand upright at rest after impact. States given to the agent are the angle and velocity of each joint (slider position and velocity are hidden), and actions are the joint torques. The reward function penalizes difference from an upright posture, slipping or chattering at the contact between the foot and the ground, non-zero joint velocities, and steady-state joint torque deviations. We use the open-source MuJoCo software for simulation \citep{todorov2012mujoco}, with high-fidelity models of ground compliance, motor nonlinearity, and joint friction. The control loop rate is $4000\;\text{Hz}$ and the rollout length is $2\;\text{s}$, resulting in a horizon of 8000 steps. Implementation details are left to the Supplementary Material.

Figure~\ref{fig:leg} shows training curves for TDPO (our method) as compared to TRPO, PPO, DDPG and TD3. These results were averaged over 75 experiments. A discount factor of $\gamma=0.99975$ was chosen for all methods, where $(1-\gamma)^{-1}$ is half the trajectory length. Similarly, the GAE factors for PPO and TRPO were scaled up to 0.99875 and 0.9995, respectively, in proportion to the trajectory length. Figure \ref{fig:leg} also shows trajectories obtained by the best agents from each method. TDPO (our method) was able to learn high-reward behavior. TRPO, PPO, DDPG, and TD3 were not.

We hypothesize that the reason for this difference in performance is that TDPO better handles the combination of two challenges presented by the leg environment---an unusually long time horizon (8000 steps) and the existence of a resonant frequency that results from compliance between the foot and the ground (note the oscillations at a frequency of about $100\;\text{Hz}$ that appear in the trajectories after impact). Both high-speed control loops and resonance due to ground compliance are common features of real-world legged robots to which TDPO seems to be more resilient.

\subsection{Practical Training and Hardware Implementation}\label{sec:legsto}

For the most realistic setting, we take the environment from the previous section and make it highly stochastic by (a) injecting noise into the transition dynamics $P$, and (b) making the initial state distribution $\mu$ as random as physically possible. We also systematically perform Hyper-Parameter Optimization (HPO) on all methods to allow the most fair comparison.



\begin{wrapfigure}{r}{0.3\textwidth}
    \centering
    \vspace{-\intextsep}
    \hspace*{-.5\columnsep}\includegraphics[width=0.485\linewidth]{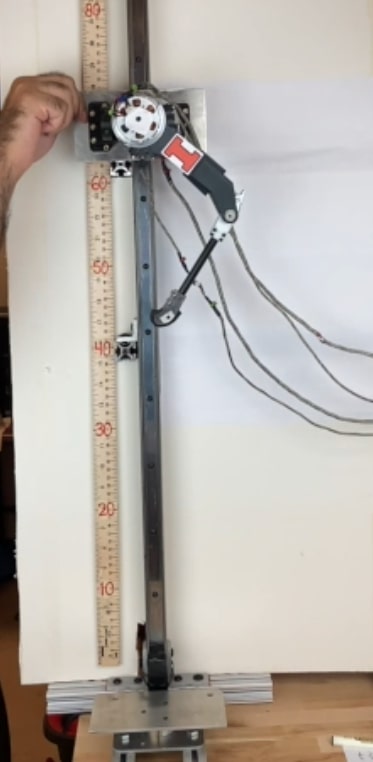} \includegraphics[width=0.48\linewidth]{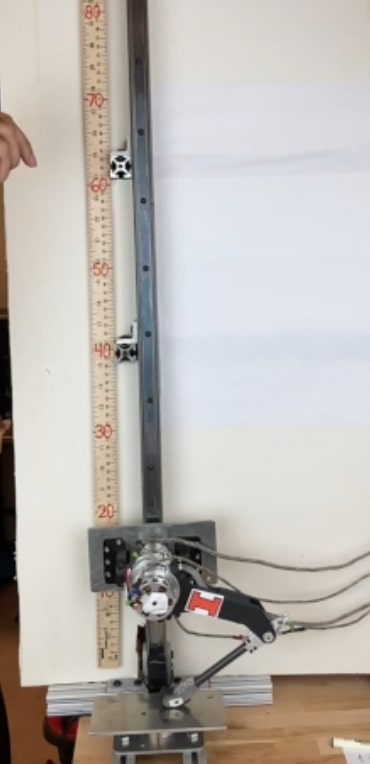}
\caption{\label{fig:sim2realleg} The simulation-to-real transfer of the best TDPO agent to perform a successful drop test at 4~kHz control rate.}
\end{wrapfigure}



The choice of the HPO method can have a significant impact on the RL agent's performance. We consider a list of five off-the-shelf HPO implementations and run them in their default settings: Optuna~\citep{optuna_2019}, BayesianOptimization~\citep{bayeopt}, Scikit-Optimize~\citep{tim_head_2018_1207017}, GPyOpt~\citep{gonzalez2016batch}, and ProSRS~\citep{shou2019tree}. These implementations include a range of HPO methods, including Gaussian processes and tree Parzen estimators. For better performance, HPO methods need a reasonable set of initial hyper-parameter guesses. For this, we perform a one-variable-at-a-time parameter sweep along every hyper-parameter near the RL method's default hyper-parameters. These parameter sweep results are then input to each HPO method for full optimization. Using all HPO algorithms for all RL methods in the long-horizon environment (where each full training run takes 5 billion samples) is computationally infeasible. To pick the best HPO method, we benchmark a short-horizon environment with only 200 time-steps in a trajectory. The result is shown in Figure~\ref{fig:hpopicking} (see the Supplementary Material for full details on the HPO methods). Overall, we found that Optuna and ProSRS are the best HPO methods on the test problem. Since Optuna is widely-tested and arguably the most popular HPO library, we pick it as the main HPO method for our long-horizon environment.

We repeat the same HPO procedure on the long-horizon environment using Optuna, and pick the best hyper-parameters found in the course of HPO for a final training. Figure~\ref{fig:legsto} shows this final training. TDPO shows superior performance in this highly stochastic environment, and such benefits cannot be obtained by merely performing HPO on other methods. To showcase the practicality of our method, we picked the best TDPO trained agent, and implemented it on the physical hardware. The transferred agent was able to successfully perform drop-and-catch tests on the robot system at 4~kHz, with both global control and suppression of high-frequency transients. Figure~\ref{fig:sim2realleg} shows a glimpse of this test, and a short video is also included in the code repository.

\section{Discussion}

We proposed a deterministic policy gradient method (TDPO: Truly Deterministic Policy Optimization) based on the use of a deterministic Vine (DeVine) gradient estimator and the Wasserstein metric. We proved monotonic payoff guarantees for our method, and defined a novel surrogate for policy optimization. We showed numerical evidence for superior performance with non-local rewards defined in the frequency domain and a realistic long-horizon resonant environment. This method enables applications of policy gradient to customize frequency response characteristics of agents. Our work assumed continuous environments, and future work should include the adaptation of our method to environments with discrete state and action spaces.

\bibliography{refs_tdpo}

\newpage
\appendix
\section{Appendix}

\subsection{Tables of Notation}
The same mathematical definitions and notations used in the paper were re-introduced and summarized in two tables; Table~\ref{tab:mathnotation} describes the mathematical functions and operators used throughout the paper, and Table~\ref{tab:mdpnotation} describes the notations needed to define the Markov Decision Process (MDP). The tables consist of two columns; one showing or defining the notation, and the other includes the name in which the same notation was called in the paper.

\renewcommand{\arraystretch}{1.6}
\begin{table}[h]
	\centering	
	\begin{tabular}{p{3.3cm}|p{9.2cm}}
		Name & Mathematical Definition or Description\\ \midrule\midrule\midrule
		\multirow{2}{*}{Value function} & $\displaystyle V^\pi(s):= \frac{1}{1-\gamma}\EE_{\substack{s_t\sim \rho_\mu^\pi\\ a_t\sim \pi(s_t)}}[R(s_t, a_t)]$ \\
		& $=\EE[\sum_{t=1}^{\infty}\gamma^{t-1}R(s_t, a_t)|s_1=s, a_t\sim \pi(s_t), s_{t+1}\sim P(s_t, a_t)].$\vspace{0.1cm} \\\hline\hline 
		Q-Value function & $Q^\pi(s,a) := R(s,a) + \gamma \cdot \EE_{s'\sim P(s,a)} [V^\pi(s')]$ \\\hline\hline 
		Advantage function & $A^\pi(s,a) := Q^\pi(s,a) - V^\pi(s)$. \\\hline
		Advantage function & $A^\pi(s,\pi') := \EE_{a\sim \pi'(s)}[A^\pi(s,a)]$. \\\hline\hline
		Arbitrary functions &  $f$ and $g$ are arbitrary functions used next.\\\hline\hline
		Arbitrary distributions &  $\nu$ and $\zeta$ are arbitrary probability distributions used next.\\\hline\hline
		Hilbert inner product & $\langle f, g\rangle_x :=\int f(x)g(x)\mathrm{d}x$\\\hline\hline
		Kulback-Liebler (KL) divergence & $\KL(\zeta|\nu):=\langle \zeta(x), \log(\frac{\zeta(x)}{\nu(x)})\rangle_x =\int \zeta(x) \log(\frac{\zeta(x)}{\nu(x)})\mathrm{d}x$\\\hline\hline
		Total Variation (TV) distance & $\TV(\zeta,\nu):=\frac{1}{2}\langle |\zeta(x)-\nu(x)|, 1\rangle_x = \frac{1}{2}\int |\zeta(x)-\nu(x)|\mathrm{d}x$.\\\hline\hline
		Coupling set & $\Gamma(\zeta, \nu)$ is the set of couplings for $\zeta$ and $\nu$.\\\hline\hline
		Wasserstein distance & $W(\zeta, \nu)=\inf_{\gamma \in \Gamma(\zeta, \nu)} \langle \|x-y\|, \gamma(x,y) \rangle_{x,y}$.\\\hline
		Policy Wasserstein distance & $W(\pi_1, \pi_2) := \sup_{s\in\mathcal{S}}W(\pi_1(\cdot|s), \pi_2(\cdot|s))$.\\\hline\hline
		Lipschitz Constant & $\Lip(f(x,y);x):=\sup_{x} \|\nabla_x f(x,y)\|_2$.\\\hline\hline
		Rubinstein-Kantrovich (RK) duality & $|\ip{\zeta(x)-\nu(x), f(x)}_x| \leq  W(\zeta, \nu) \cdot \Lip(f;x)$.\\\hline\hline
		
	\end{tabular}%
	\caption{The mathematical notations used throughout the paper.}
	\label{tab:mathnotation}%
\end{table}%

\renewcommand{\arraystretch}{1.5}
\begin{table}[t]
	\centering	
	\begin{tabular}{c|p{9.5cm}}
		Mathematical Notation & Name and Description \\ \midrule\midrule\midrule
		$\mathcal{S}$ & This is the state space of the MDP. \\\hline\hline 
		$\mathcal{A}$ & This is the action space of the MDP. \\\hline\hline 
		$\gamma$ & This is the discount factor of the MDP. \\\hline\hline 
		$R: \mathcal{S} \times \mathcal{A}\rightarrow \R$ & This is the reward function of the MDP. \\\hline\hline
		$\mu$ & This is the initial state distribution of the MDP over the state space. \\\hline\hline 
		$\Delta$ & $\Delta(\mathcal{F})$ is the set of all probability distributions over the arbitrary set $\mathcal{F}$ (otherwise known as the Credal set of $\mathcal{F}$).\\\hline\hline
		
		$\pi$ & In general, $\pi$ denotes the policy of the MDP. However, the output argument type could vary in the text. See the next lines.\\\hline 
		\multirow{2}{*}{$\pi: \mathcal{S} \rightarrow \Delta(\mathcal{A})$} & Given a state $s\in \mathcal{S}$, $\pi(s)$ and $\pi(\cdot|s)$ denote the action distribution suggested by the policy $\pi$. \\
		& In other words, $a\sim\pi(s)$ and $a\sim\pi(\cdot | s)$.\\\hline 
		\multirow{2}{*}{$\pi_{\text{det}}: \mathcal{S} \rightarrow \mathcal{A}$} & For a deterministic policy $\pi_{\text{det}}$, the unique action $a$ suggested by the policy given the state $s$ can be denoted by $\pi(s)$ specially. \\
		& In other words, $a=\pi_{\text{det}}(s)$.\\\hline 
		$\Pi$ & $\Pi$ is the set of all policies (i.e., $\forall \pi: \pi\in \Pi$).\\\hline 
		\hline
		
		$P$ & In general, $P$ denotes the transition dynamics model of the MDP. However, the input argument types could vary throughout the text. See the next lines for more clarification.\\\hline 
		$P: \mathcal{S} \times \mathcal{A} \rightarrow \Delta(\mathcal{S})$ & Given a particular state $s$ and action $a$, $P(s, a)$ will be the next state distribution of the transition dynamics (i.e. $s'\sim P(s, a)$ where $s'$ denotes the next state after applying $s$, $a$ to the transition $P$).\\\hline 
		
		
		$P: \Delta(\mathcal{S}) \times \mathcal{A} \rightarrow \Delta(\mathcal{S})$ & This is a generalization of the transition dynamics to accept state distributions as input. In other words, $P(\nu_s, a) := \EE_{s\sim \nu_s}[P(s,a)]$.\vspace{0.1cm}\\\hline 
		
		$P: \mathcal{S} \times \Delta(\mathcal{A}) \rightarrow \Delta(\mathcal{S})$ & This is a generalization of the transition dynamics to accept action distributions as input. In other words, $P(s, \nu_a) := \EE_{a\sim \nu_a}[P(s,a)]$.\vspace{0.1cm}\\\hline 
		
		\multirow{2}{*}{$\PP: \Delta(\mathcal{S}) \times \Pi \rightarrow \Delta(\mathcal{S})$} & This is a generalization of the transition dynamics to accept a state distribution and a policy as input. Given an arbitrary state distribution $\nu_s$ and a policy $\pi$, and $\PP(\nu_s, \pi)$ will be the next state distribution given that the state is sampled from $\nu_s$ and the action is sampled from the $\pi(s)$ distribution. \\
		& In other words, we have $\PP(\nu_s, \pi) := \EE_{\substack{s\sim \nu_s \\a\sim \pi(s)}}[P(s,a)]$.\vspace{0.1cm}\\\hline 
		
		$\PP^t: \Delta(\mathcal{S}) \times \Pi \rightarrow \Delta(\mathcal{S})$ & This is the $t$-step transition dynamics generalization. Given an arbitrary state distribution $\nu_s$ and a policy $\pi$ and non-negative integer $t$, one can define $\PP^t$ recursively as $\PP^0(\nu_s,\pi):=\nu_s$ and $\PP^t(\nu_s,\pi):=\PP(\PP^{t-1}(\nu_s, \pi), \pi)$. \\\hline\hline
		$\rho_{\mu}^{\pi}$ & The discounted visitation frequency $\rho_{\mu}^{\pi}$ is a distribution over  $\mathcal{S}$, and can be defined as $\rho_{\mu}^{\pi}:=(1-\gamma)\sum_{t=0}^\infty \gamma^{t}\PP^{t}(\mu, \pi)$.\\\hline\hline
		
	\end{tabular}%
	\caption{The MDP notations used throughout the paper.}
	\label{tab:mdpnotation}%
\end{table}%

\clearpage

\subsection{Brief Introduction to Policy Gradient Methods}
Conservative Policy Iteration (CPI) \citep{kakade2002approximately} was one of the early dimensionally consistent methods with a surrogate of the form $\mathcal{L}_{\pi_1}(\pi_2) = \eta_{\pi_1} + \frac{1}{1-\gamma}\cdot\EE_{s\sim \rho_\mu^{\pi_1}}[A^{\pi_1}(s, \pi_2)] - \frac{C}{2} \TV^2(\pi_1, \pi_2)$. The $C$ coefficient guarantees non-decreasing payoffs. However, CPI is limited to linear function approximation classes due to the update rule $\pi_{\text{new}}\leftarrow (1-\alpha) \pi_{\text{old}} + \alpha \pi'$. This lead to the design of the Trust Region Policy Optimization (TRPO) \citep{schulman2015trust} algorithm.

TRPO exchanged the bounded squared TV distance with the KL divergence by lower bounding it using the Pinsker inequality. This made TRPO closer to the Natural Policy Gradient algorithm\citep{kakade2002natural}, and for Gaussian policies, the modified terms had similar Taylor expansions within small trust regions. Confined trust regions are a stable way of making large updates and avoiding pessimistic coefficients. For gradient estimates, TRPO used Importance Sampling (IS) with a baseline shift:
\begin{align}
\nabla_{\theta_2} \EE_{s\sim \rho_\mu^{\pi_1}}[A^{\pi_1}(s, \pi_2)]\Big|_{\theta_2=\theta_1} = \nabla_{\theta_2} \EE_{\substack{{s\sim \rho_\mu^{\pi_1}} \\ {a\sim \pi_1(\cdot|s)}}}\bigg[ Q^{\pi_1}(s, a) \frac{\pi_2(a|s)}{\pi_1(a|s)}\bigg]\Big|_{\theta_2=\theta_1}.
\label{eq:impsampQ}
\end{align}
While empirical $\EE[A^{\pi_1}(s,\pi_2)]$ and $\EE[Q^{\pi_1}(s,\pi_2)]$ estimates yield identical variances in principle, the importance sampling estimator in~(\ref{eq:impsampQ}) imposes larger variances. Later, Proximal Policy Optimization (PPO) \citep{schulman2015high} proposed utilizing the Generalized Advantage Estimation (GAE) method for variance reduction and incorporated first-order smoothing like ADAM \citep{kingma2014adam}. GAE employed Temporal-Difference (TD) learning \citep{bhatnagar2009convergent} for variance reduction. Although TD-learning was not theoretically guaranteed to converge and could add bias, it improved the estimation accuracy.

As an alternative to IS, deterministic policy gradient estimators were also utilized in an actor-critic fashion. Deep Deterministic Policy Gradient (DDPG) \citep{lillicrap2015continuous} generalized deterministic gradients by employing Approximate Dynamic Programming (ADP) \citep{mnih2015human} for variance reduction. Twin Delayed Deterministic Policy Gradient (TD3) \citep{fujimoto2018addressing} improved DDPG's approximation to build an even better policy optimization method. Although both methods used deterministic policies, they still performed stochastic search by adding stochastic noise to the deterministic policies to force exploration.

Other lines of stochastic policy optimization were later proposed. \citet{wu2017scalable} used a Kronecker-factored approximation for curvatures. \citet{haarnoja2018soft} proposed a maximum entropy actor-critic method for stochastic policy optimization.

\subsection{Reinforcement Learning Challenges} We will briefly describe a few challenges in modern reinforcement learning: (a) the problem of non-local rewards, (b) scalability to longer horizons, and (c) observation or action delay.

\subsubsection{Non-local Rewards} 

An underlying assumption in the MDP framework is that the desired payoff can be decomposed into a (discounted) sum of time-step rewards. This leaves out practical payoff functions that cannot be expressed in this form. Non-local rewards are reward functions of the entire trajectory whose payoffs cannot be decomposed into the sum of terms such as $\eta=\sum_{t} f_t(s_t,a_t)$, where functions $f_t$ only depend on nearby states and actions. An example non-local reward is one that depends on the Fourier transform of the complete trajectory signal. Other examples include trajectory statistics (e.g., the median or maximum observation in a trajectory). In both examples, the reward cannot be determined without collecting the entire trajectory. While approximating non-local rewards with local ones is possible, such approximations usually come at significant reward engineering costs and undesired behavior in the policy. Although policy gradient methods are designed under the MDP framework and theoretically under-equipped for such challenges, being resilient to them is a desired property. 

\subsubsection{Scalability to Longer Horizons} 

In its simplified and un-discounted form, reinforcement learning aims at optimizing the $\eta=\sum_{t=1}^{T} r_t$ payoff by determining the proper actions. It is insightful to contemplate the difficulty of this goal relative to the time-horizon $T$. With $T=1$ the optimal policy is to take the greedy action at each time-step. However, with larger $T$ finding the optimal policy becomes more challenging. 

In infinite-horizon discounted MDPs, $1/(1-\gamma)$ plays the same effect as $T$. In particular, we have ${\sum_{i=1}^{T} \gamma^i}/{\sum_{i=1}^{\infty} \gamma^i} = 1-\gamma^{T}\simeq 1-e^{-T(1-\gamma)}$. In other words, although the framework optimizes over infinite time-steps, the cumulative weight of time-steps after $T$ decays exponentially with a time constant of $1/(1-\gamma)$. For instance, with $\gamma=0.99$ the first 200 steps constitute 87\% of the infinite-length trajectory payoff, whereas with $\gamma=0.999$ the first 2000 steps constitute the same portion of the payoff. This is why $1/(1-\gamma)$ appears in most theoretical sample-complexity analyses and higher bounds in an exponential capacity~\citep{kakade2003sample,kearns2000approximate,kearns2002sparse}. Practically, longer horizons can appear in at least two forms: (a) preserving the task complexity but increasing the decision-making frequency, and (b) increasing the task complexity. We call these forms \textit{soft} and \textit{hard} horizon scalability, respectively. 


\begin{figure}[t]
\includegraphics[width=0.98\linewidth]{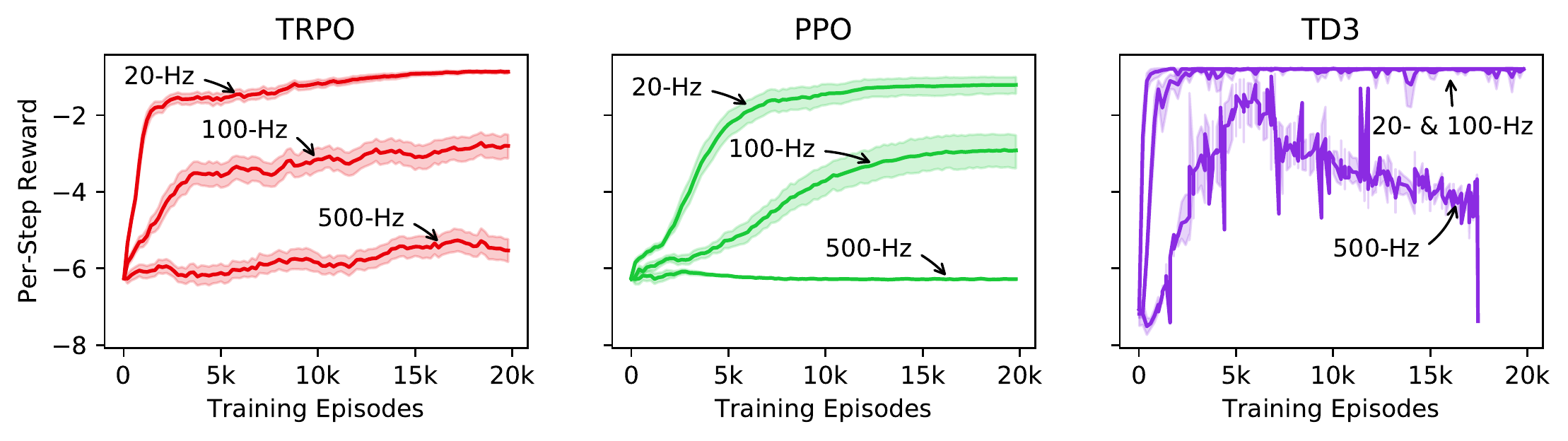}
\caption{The effect of soft horizon scaling on the typical pendulum continuous control task. Three environments are defined at different control frequencies. All environments try to achieve the same goal of making the same pendulum stand up-right within 10 seconds. The original environment runs at 20 Hz control frequency. We also show two similar environments running at 100 Hz and 500 Hz control frequency.  To make the tasks comparable, the horizontal axis shows the training episodes and the vertical axis shows the normalized reward per time-step. The environment and training hyper-parameters are given in Table~\ref{tab:hpspendlh}. Evidently, all methods suffer from the curse of horizon.}
\label{fig:pendhorizon}
\end{figure}

\renewcommand{\arraystretch}{0.8}
\begin{table}[t]
\parbox{.45\linewidth}{
\small
\centering
    \begin{tabular}{|c|c|c|c|}
    \toprule
    \multirow{2}{2.5cm}{General Hyper-Parameters} & \multicolumn{3}{c|}{Control Frequency} \\
\cmidrule{2-4}          & 20 hz & 100 hz & 500 hz \\
    \midrule
    Control Time-step & 50 ms & 10 ms & 2 ms \\
    \midrule
    Trajectory Duration & 10 s  & 10 s  & 10 s \\
    \midrule
    Parallel Workers & 4     & 4     & 4 \\
    \midrule
    Training Episodes & 20K   & 20K   & 20K \\
    \midrule
    Episode Time-steps & 200   & 1000  & 5000 \\
    \bottomrule
    \end{tabular}%
  \label{tab:generalpendlh}%
}\qquad
\parbox{.45\linewidth}{
\small
  \centering
    \begin{tabular}{|c|c|c|c|}
    \toprule
    \multirow{2}{2.5cm}{TD3 Hyper-Parameters} & \multicolumn{3}{c|}{Control Frequency} \\
\cmidrule{2-4}          & 20 hz & 100 hz & 500 hz \\
    \midrule
    MDP Discount & 0.99  & 0.998 & 0.9996 \\
    \midrule
    Replay Buffer Size & 50K   & 250K  & 1.25M \\
    \midrule
    Initial Random Steps & 100   & 500   & 2500 \\
    \midrule
    Training Interval & 100   & 500   & 2500 \\
    \midrule
    Opt. Batch Size & 128   & 640   & 3200 \\
    \bottomrule
    \end{tabular}%
  \label{tab:td3pendlh}%
}\vspace{0.3cm}\\
\parbox{.45\linewidth}{
\small
  \centering
    \begin{tabular}{|c|c|c|c|}
    \toprule
    \multirow{2}{2.5cm}{TRPO Hyper-Parameters} & \multicolumn{3}{c|}{Control Frequency} \\
\cmidrule{2-4}          & 20 hz & 100 hz & 500 hz \\
    \midrule
    Sampling Batch Size & 1024  & 5120  & 25600 \\
    \midrule
    MDP Discount & 0.99  & 0.998 & 0.9996 \\
    \midrule
    GAE Discount & 0.98  & 0.996 & 0.9992 \\
    \midrule
    Value Batch Size & 128   & 640   & 3200 \\
    \bottomrule
    \end{tabular}%
  \label{tab:trpopendlh}%
}\qquad
\parbox{.45\linewidth}{
\small
  \centering
    \begin{tabular}{|c|c|c|c|}
    \toprule
    \multirow{2}{2.5cm}{PPO Hyper-Parameter} & \multicolumn{3}{c|}{Control Frequency} \\
\cmidrule{2-4}          & 20 hz & 100 hz & 500 hz \\
    \midrule
    Sampling Batch Size & 256   & 1280  & 6400 \\
    \midrule
    MDP Discount & 0.99  & 0.998 & 0.9996 \\
    \midrule
    GAE Discount & 0.95  & 0.99  & 0.998 \\
    \midrule
    Opt. Batch Size & 64    & 320   & 1600 \\
    \bottomrule
    \end{tabular}%
  \label{tab:ppo1pendlh}
}
\caption{The settings and hyper-parameters used to produce Figure~\ref{fig:pendhorizon}.  The top-left table shows the common settings used to define the environment and the run the training. The scaled hyper-parameters for each of the TD3, TRPO, and PPO methods are given in the top-right, bottom-left, and bottom-right corner, respectively. Other hyper-parameters were set to their default value in all methods.}
 \label{tab:hpspendlh}
\end{table}
\renewcommand{\arraystretch}{1.0}

\paragraph{Soft Horizon Scalability:} In physical systems, one can preserve the task complexity but increase the control frequency. This increases the policy's agility in adapting to changes in the observation. Each time-step can be divided into $k$ smaller time-steps, resulting in a $k$-fold increase of time-steps per trajectory. This is what we call \textit{soft horizon scalability}. Intuitively, soft horizon scaling should only improve the optimal policy's performance; the smaller time-step policy can be faster in response to observation changes, and it has the freedom to choose different actions in the $k$ smaller time-steps rather than being constrained to apply the same action during the entire $k$ smaller time-steps. Unfortunately, that's not the case for existing policy optimization methods.

To showcase this effect, we consider a typical pendulum benchmark problem. Existing methods can solve this task with their default settings in less than a million serial time-steps. By making the control time-steps 5 or 25 times smaller, one may hope to achieve better per-step rewards. Of course, some hyper-parameters (e.g., the sampling batch-size) must be scaled proportionally. Table~\ref{tab:hpspendlh} summarizes such scalings. Figure~\ref{fig:pendhorizon} shows the training curves for each method and control frequency. Evidently, all methods suffer from the curse of horizon. In particular, on-policy methods (TRPO and PPO) are most vulnerable to soft-horizon scaling. TD3, on the other hand, is closer to TD(0) algorithms. This problem has usually been addressed with the \textit{frame-skip} trick, where same action is zero-held for multiple time-steps. However, this trick reduces the policy's agility or impacts the optimal behavior negatively. The performance deterioration with soft horizon scaling can be attributed to the higher variance of estimated gradients in reinforcement learning methods with longer horizons.

\paragraph{Hard Horizon Scalability} When multiple tasks are stacked in the time horizon and are conditioned upon the completion of each other, \textit{hard horizon scalability} is achieved. Consider a treasure hunt game where the next clue is conditioned upon solving the current task. Stacking more tasks makes winning the game exponentially more difficult; a single mistake in any step can result in failure. Reinforcement learning methods can suffer the same way with composite tasks. It is difficult to resolve hard horizon scalability without being resilient to soft horizon scalability in the first place. Overall, hard horizon scalability is a difficult challenge and beyond the scope of our work.

\begin{figure}[t]
\includegraphics[width=0.98\linewidth]{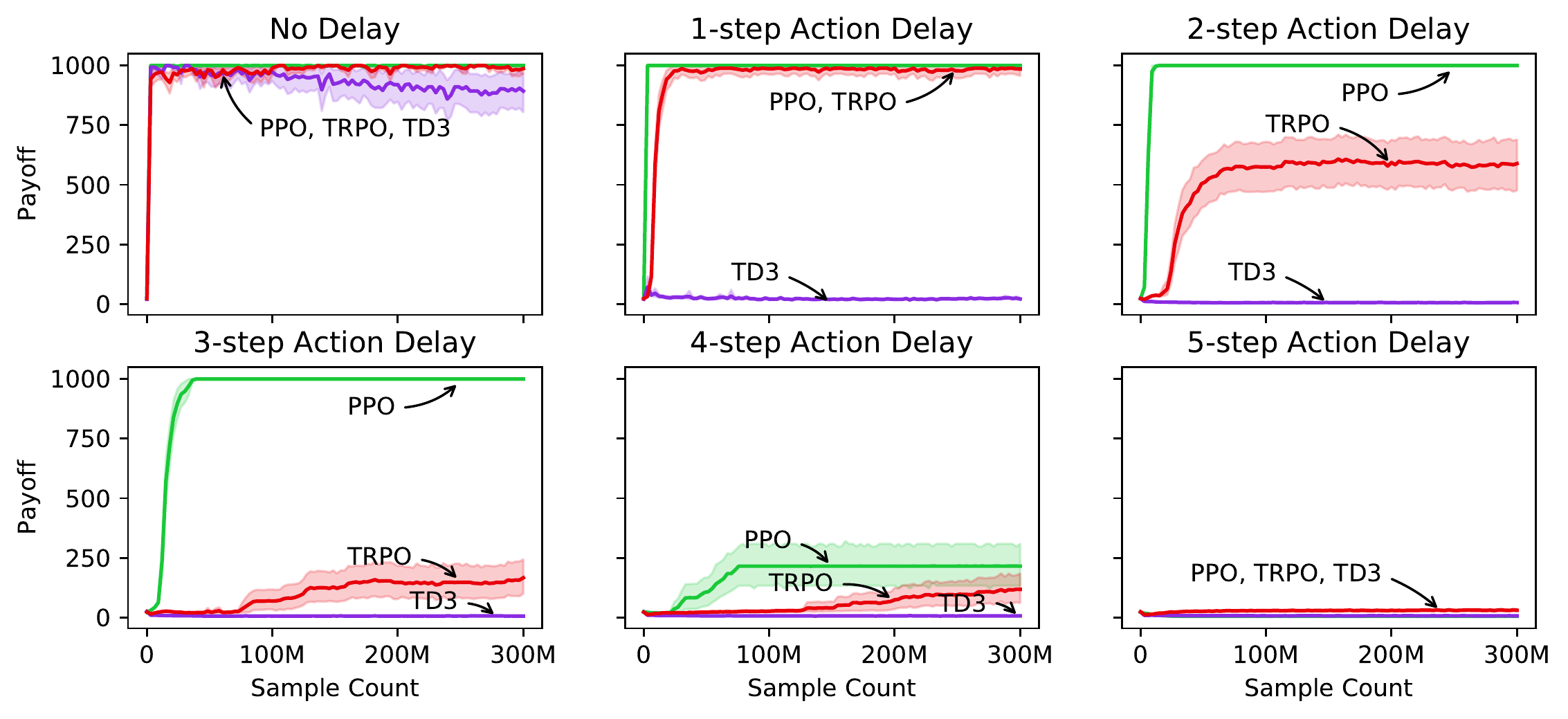}
\caption{The effect of action delay on the inverted pendulum continuous control task. The top left plot shows each method's performance on the original problem, and the rest simulate different amounts of action delay time-steps. All methods' lose performance with 5 time-steps of action delay.}
\label{fig:actdelayinvpend}
\vspace{-3mm}
\end{figure}

\subsubsection{Action or Observation Delay}

Delay in sensing the observation or applying the desired actuation is a challenging artifact in physical systems. Such delays have been a favorite topic of research in traditional control theory~\citet{golnaraghi2002auto}. Such delays are most influential in high-bandwidth control systems, where such delays are non-negligible compared to the high-frequency control rate. Although the MDP framework does not address observation or action delays, being resilient to them is a desirable feature for all reinforcement learning methods. To show the effect of delay on PG methods, we simulated a typical inverted pendulum task and delayed the proposed actions by the agent for different numbers of time-steps. The training curves are shown in Figure~\ref{fig:actdelayinvpend}. With no action delay, all methods produce high-performance agents. However, at 5 time-steps of action delay, the resulting agents are almost indistinguishable from the initial policies. In particular, TD3 is most vulnerable to delay, while PPO and TRPO could tolerate a few time-steps of delay. We speculate that this is due to TD3 being closely related to the TD(0) methods, whereas PPO and TRPO are closer to TD(1) when estimating the state values in their training processes. Overall, observation and action delays are unresolved topics in modern reinforcement learning.

\clearpage
\newpage

\section{Theoretical Proofs and Derivations}

Two theoretical results from the main paper were left to be discussed here. The bulk of our theoretical derivations (Sections~\ref{sec:ptboundwrtpolicy}-\ref{sec:finalizedpayoffbound}) belongs to the payoff improvement lower-bound of Theorem~\ref{thm:finalpracticallowerbound}, which was used in Algorithm~\ref{alg:wppoalg} of the main paper to regulate the policy updates. Figure~\ref{fig:proofflowchart} shows a flow-chart of the theoretical steps necessary to prove Theorem~\ref{thm:finalpracticallowerbound}, and Section~\ref{ss:allassumptions} discusses the assumptions used in the theoretical derivations. On a separate note, Section~\ref{sec:proofdevine} is dedicated to proving Theorem~\ref{thm:accuratevinepg}, which shows that the DeVine advantage estimator (Algorithm~\ref{alg:vineAdv} of the main paper) can provide exact policy gradient estimates under certain conditions. 

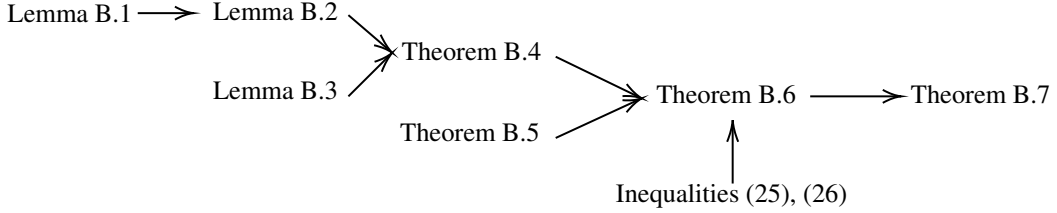
\begin{figure}[h]
\centering
\tikzset{every picture/.style={line width=0.75pt}} 
\begin{tikzpicture}[x=0.75pt,y=0.75pt,yscale=-1,xscale=0.85]
\draw    (94,19) -- (125,19) ;
\draw [shift={(127,19)}, rotate = 180] [color={rgb, 255:red, 0; green, 0; blue, 0 }  ][line width=0.75]    (10.93,-3.29) .. controls (6.95,-1.4) and (3.31,-0.3) .. (0,0) .. controls (3.31,0.3) and (6.95,1.4) .. (10.93,3.29)   ;
\draw    (219,20) -- (243.39,37.82) ;
\draw [shift={(245,39)}, rotate = 216.16] [color={rgb, 255:red, 0; green, 0; blue, 0 }  ][line width=0.75]    (10.93,-3.29) .. controls (6.95,-1.4) and (3.31,-0.3) .. (0,0) .. controls (3.31,0.3) and (6.95,1.4) .. (10.93,3.29)   ;
\draw    (219,61) -- (243.47,40.29) ;
\draw [shift={(245,39)}, rotate = 139.76] [color={rgb, 255:red, 0; green, 0; blue, 0 }  ][line width=0.75]    (10.93,-3.29) .. controls (6.95,-1.4) and (3.31,-0.3) .. (0,0) .. controls (3.31,0.3) and (6.95,1.4) .. (10.93,3.29)   ;
\draw    (342,41) -- (391.15,61.24) ;
\draw [shift={(393,62)}, rotate = 202.38] [color={rgb, 255:red, 0; green, 0; blue, 0 }  ][line width=0.75]    (10.93,-3.29) .. controls (6.95,-1.4) and (3.31,-0.3) .. (0,0) .. controls (3.31,0.3) and (6.95,1.4) .. (10.93,3.29)   ;
\draw    (342,82) -- (391.14,62.73) ;
\draw [shift={(393,62)}, rotate = 158.59] [color={rgb, 255:red, 0; green, 0; blue, 0 }  ][line width=0.75]    (10.93,-3.29) .. controls (6.95,-1.4) and (3.31,-0.3) .. (0,0) .. controls (3.31,0.3) and (6.95,1.4) .. (10.93,3.29)   ;
\draw    (447,105) -- (447,75) ;
\draw [shift={(447,73)}, rotate = 90] [color={rgb, 255:red, 0; green, 0; blue, 0 }  ][line width=0.75]    (10.93,-3.29) .. controls (6.95,-1.4) and (3.31,-0.3) .. (0,0) .. controls (3.31,0.3) and (6.95,1.4) .. (10.93,3.29)   ;
\draw    (493,61) -- (545,61) ;
\draw [shift={(547,61)}, rotate = 180] [color={rgb, 255:red, 0; green, 0; blue, 0 }  ][line width=0.75]    (10.93,-3.29) .. controls (6.95,-1.4) and (3.31,-0.3) .. (0,0) .. controls (3.31,0.3) and (6.95,1.4) .. (10.93,3.29)   ;

\draw (15,13) node [anchor=north west][inner sep=0.75pt]   [align=left] {Lemma~\ref{lemma:fullasumption}};
\draw (137,12) node [anchor=north west][inner sep=0.75pt]   [align=left] {Lemma~\ref{lemma:ptbound}};
\draw (137,52) node [anchor=north west][inner sep=0.75pt]   [align=left] {Lemma~\ref{lemma:lincombwass}};
\draw (249,32) node [anchor=north west][inner sep=0.75pt]   [align=left] {Theorem~\ref{thm:visitwass}};
\draw (248,73) node [anchor=north west][inner sep=0.75pt]   [align=left] {Theorem~\ref{thm:gradtwoconcnt}};
\draw (400,54) node [anchor=north west][inner sep=0.75pt]   [align=left] {Theorem~\ref{thm:wgpayoffbound}};
\draw (551,54) node [anchor=north west][inner sep=0.75pt]   [align=left] {Theorem~\ref{thm:finalpracticallowerbound}};
\draw (376,104) node [anchor=north west][inner sep=0.75pt]   [align=left] {Inequalities~\eqref{eq:rksecorderbound},~\eqref{eq:secondarybound}};
\end{tikzpicture}
\caption{The theoretical derivations flow-chart to prove Theorem~\ref{thm:finalpracticallowerbound}.}
\label{fig:proofflowchart}
\end{figure}

\subsection{Bounding $W(\PP^t(\mu, \pi_1), \PP^t(\mu, \pi_2))$}\label{sec:ptboundwrtpolicy}
To review, the dynamical smoothness assumptions were 
\begin{align*}
W(\PP(\mu, \pi_1), \PP(\mu, \pi_2))&\leq L_\pi \cdot W(\pi_1, \pi_2), \\
W(\PP(\mu_1, \pi), \PP(\mu_2, \pi))&\leq L_\mu \cdot W(\mu_1, \mu_2).
\end{align*}
The following lemma states that these two assumptions are equivalent to a more concise assumption. This will be used to bound the $t$-step visitation distance and prove Lemma~\ref{lemma:ptbound}.
\begin{lemma}
	\label{lemma:fullasumption}
	Assumptions~\eqref{eq:piassumption} and~\eqref{eq:sassumption} are equivalent to having
	\begin{equation}
	\label{eq:fullassumption}
	W(\PP(\mu_1, \pi_1), \PP(\mu_2, \pi_2))\leq L_\mu \cdot W(\mu_1, \mu_2) + L_\pi \cdot W(\pi_1, \pi_2).
	\end{equation}
\end{lemma}
\begin{proof}
	To prove the $\eqref{eq:piassumption}, \eqref{eq:sassumption}\Rightarrow \eqref{eq:fullassumption}$ direction, the triangle inequality for the Wasserstein distance gives
	\begin{equation}
	W(\PP(\mu_1, \pi_1), \PP(\mu_2, \pi_2))\leq W(\PP(\mu_1, \pi_1), \PP(\mu_2, \pi_1)) + W(\PP(\mu_2, \pi_1), \PP(\mu_2, \pi_2))
	\label{eq:wasstrian}
	\end{equation}
	and using \eqref{eq:piassumption}, \eqref{eq:sassumption}, and \eqref{eq:wasstrian} then implies
	\begin{equation}
	W(\PP(\mu_1, \pi_1), \PP(\mu_2, \pi_2))\leq L_\mu \cdot W(\mu_1, \mu_2) + L_\pi \cdot W(\pi_1, \pi_2).
	\end{equation}
	The other direction is trivial.
\end{proof}


\begin{lemma}
	\label{lemma:ptbound}
	Under Assumptions~\eqref{eq:piassumption} and~\eqref{eq:sassumption} we have the bound
	\begin{equation}
	W(\PP^t(\mu, \pi_1), \PP^t(\mu, \pi_2)) \leq L_\pi \cdot (1 + L_\mu +\cdots + L_\mu^{t-1}) \cdot W(\pi_1, \pi_2),
	\label{eq:ptbound}
	\end{equation}
	where $\PP^t(\mu, \pi)$ denotes the state distribution after running the MDP for $t$ time-steps with the initial state distribution $\mu$ and policy $\pi$.
\end{lemma}
\begin{proof}
	For $t=1$, the lemma is equivalent to Assumption~\eqref{eq:piassumption}. This paves the way for the lemma to be proved using induction. The hypothesis is
	\begin{equation}
	W(\PP^{t-1}(\mu, \pi_1), \PP^{t-1}(\mu, \pi_2)) \leq L_\pi \cdot (1 + L_\mu +\cdots + L_\mu^{t-2}) \cdot W(\pi_1, \pi_2),
	\label{eq:induchypo}
	\end{equation} and for the induction step we write
	\begin{equation}
	\label{eq:inductstepeq}
	W(\PP^t(\mu, \pi_1), \PP^t(\mu, \pi_2)) = W(\PP(\PP^{t-1}(\mu, \pi_1), \pi_1) , \PP(\PP^{t-1}(\mu, \pi_2), \pi_2)).
	\end{equation}
	Using Assumption~\eqref{eq:fullassumption}, which according to Lemma~\ref{lemma:fullasumption} is equivalent to Assumptions~\eqref{eq:piassumption} and~\eqref{eq:sassumption}, we can combine~\eqref{eq:induchypo} and~\eqref{eq:inductstepeq} into
	\begin{equation}
	W(\PP^t(\mu, \pi_1), \PP^t(\mu, \pi_2)) \leq L_\pi \cdot W(\pi_1, \pi_2) + L_\mu \cdot W(\PP^{t-1}(\mu, \pi_1), \PP^{t-1}(\mu, \pi_2)).
	\end{equation}
	Thus, by applying the induction Hypothesis~\eqref{eq:induchypo}, we have
	\begin{equation}
	W(\PP^t(\mu, \pi_1), \PP^t(\mu, \pi_2)) \leq L_\pi \cdot W(\pi_1, \pi_2) + L_\mu \cdot  L_\pi \cdot (1 + L_\mu +\cdots + L_\mu^{t-2}) \cdot W(\pi_1, \pi_2),
	\end{equation} which can be simplified into the lemma statement (i.e., Inequality~\eqref{eq:ptbound}).
\end{proof}

\subsection{Bounding $W(\rho_\mu^{\pi_1}, \rho_\mu^{\pi_2})$}
Lemma~\ref{lemma:ptbound} suggests making the $\gamma L_{\mu}<1$ assumption and paves the way for Theorem~\ref{thm:visitwass}. The $\gamma L_{\mu}<1$ assumption is overly restrictive and unnecessary but makes the rest of the proof easier to follow. This assumption can be relaxed by a general transition dynamics stability assumption which is discussed in more detail later in section~\ref{ss:dynamicstability}, and an equivalent $\gamma \bar{L}_{\mu}<1$ assumption is introduced to replace $\gamma L_{\mu}<1$. 

First, we need to introduce Lemma~\ref{lemma:lincombwass} first, which will be used in the proof of Theorem~\ref{thm:visitwass}.

\begin{lemma}
	The Wasserstein distance between linear combinations of distributions can be bounded as $W(\beta\cdot\mu_1 + (1-\beta)\cdot\nu_1, \beta\cdot\mu_2 + (1-\beta)\cdot\nu_2) \leq \beta\cdot W(\mu_1, \mu_2) + (1-\beta)\cdot W(\nu_1, \nu_2)$.
	\label{lemma:lincombwass}
\end{lemma}
\begin{proof}
	Plugging $\gamma = \beta \cdot \gamma_{(\mu_1,\mu_2)} + (1-\beta) \cdot  \gamma_{(\nu_1,\nu_2)}$ in the Wasserstein definition yields the result.
\end{proof}

\begin{theorem}
	\label{thm:visitwass}
	Assuming (\ref{eq:piassumption}), (\ref{eq:sassumption}), and $\gamma L_\mu<1$, we have the inequality
	\begin{equation}
	W(\rho_\mu^{\pi_1}, \rho_\mu^{\pi_2}) \leq \frac{\gamma L_\pi}{1-\gamma L_\mu} \cdot W(\pi_1, \pi_2).
	\end{equation}
\end{theorem}

\begin{proof}
	Using Lemma~\ref{lemma:lincombwass} and the definition of $\rho_\mu^\pi$, we can write
	\begin{equation}
	\label{eq:wassvisfirst}
	W(\rho_\mu^{\pi_1}, \rho_\mu^{\pi_2}) \leq (1-\gamma)\sum_{t=0}^\infty \gamma^t\cdot W(\PP^t(\mu, \pi_1), \PP^t(\mu, \pi_2)).
	\end{equation}
	Using Lemma~\ref{lemma:ptbound}, we can take another step to relax the inequality (\ref{eq:wassvisfirst}) and write
	\begin{equation}
	\label{eq:wassvisticonvsum}
	W(\rho_\mu^{\pi_1}, \rho_\mu^{\pi_2}) \leq \frac{L_\pi (1-\gamma) W(\pi_1, \pi_2)}{(L_\mu-1)} \sum_{t=0}^\infty ((\gamma L_\mu)^t-\gamma^t).
	\end{equation}
	Due to the $\gamma L_\mu<1$ assumption, the right-hand summation in (\ref{eq:wassvisticonvsum}) is convergent. This leads us to
	\begin{equation}
	\label{eq:prefinalviswass}
	W(\rho_\mu^{\pi_1}, \rho_\mu^{\pi_2}) \leq \frac{L_\pi (1-\gamma) W(\pi_1, \pi_2)}{(L_\mu-1)} (\frac{1}{1 - \gamma L_\mu}-\frac{1}{1 - \gamma}).
	\end{equation}
	Inequality (\ref{eq:prefinalviswass}) can then be simplified to give the result.
\end{proof}

\subsection{Steps to Bound the Second-order Term}\label{sec:bound2ndorder}
The RK duality yields the following bound:
\begin{equation}
|\ip{\rho_\mu^{\pi_2}-\rho_\mu^{\pi_1}, A^{\pi_1}(\cdot, \pi_2)}_s| \leq W(\rho_\mu^{\pi_1}, \rho_\mu^{\pi_2}) \cdot \sup_{s} \|\nabla_s A^{\pi_1}(s, \pi_2)\|_2.
\label{eq:rksecorderbound}
\end{equation}
To facilitate the further application of the RK duality, any advantage can be rewritten as the following inner product: $A^{\pi_1}(s, \pi_2) = \ip{\pi_2(a|s) - \pi_1(a|s),  Q^{\pi_1}(s, a)}_{a}$. Taking derivatives of both sides with respect to the state variable and applying the triangle inequality produces the bound
\begin{align}
\sup_{s} \|\nabla_s A^{\pi_1}(s, \pi_2)\|_2 &\leq \sup_s  \| \ip{\nabla_s (\pi_2(a|s) - \pi_1(a|s)),  Q^{\pi_1}(s, a)}_{a} \|_2 \nonumber\\
&\qquad +\sup_s
\|\ip{\pi_2(a|s) - \pi_1(a|s), \nabla_s Q^{\pi_1}(s, a)}_{a}\|_2.
\label{eq:secondarybound}
\end{align}
The second term of the RHS in~(\ref{eq:secondarybound}) is compatible with the RK duality. However, the form of the first term does not warrant an easy application of RK. For this, we introduce Theorem~\ref{thm:gradtwoconcnt}.

\begin{theorem}
	\label{thm:gradtwoconcnt}
	Assuming the existence of $\Lip(Q^{\pi_1}(\bs, a);a)$, we have the bound
	\begin{align}
	&\bigg\| \bip{\nabla_\bs (\pi_2(a|\bs) - \pi_1(a|\bs)),  Q^{\pi_1}(\bs, a)}_{a} \bigg\|_2  \\
	&\qquad \qquad \leq 2\cdot \Lip(Q^{\pi_1}(\bs, a);a) \cdot \bigg\| \nabla_{\bs'} W\bigg(\frac{\pi_2(a|\bs') + \pi_1(a|\bs)}{2}, \frac{\pi_2(a|\bs) + \pi_1(a|\bs')}{2}\bigg)\bigg|_{\bs'=\bs} \bigg\|_2 \nonumber.
	\end{align}
\end{theorem}

\begin{proof}
	By definition, we have
	\begin{align}\label{eq:graddimdecomp}
	\big\| \bip{\nabla_\bs (\pi_2(a|\bs) - \pi_1(a|&\bs)),  Q^{\pi_1}(\bs, a)}_{a} \big\|_2 \nonumber\\
	&= \sqrt{\sum_{j=1}^{\dim(\mathcal{S})} \bigg( \Bip{\frac{\partial}{\partial s^{(j)}} (\pi_2(a|\bs) - \pi_1(a|\bs)),  Q^{\pi_1}(\bs, a)}_{a} \bigg)^2}.
	\end{align}
	For better insight, we will write the derivative using finite differences:
	\begin{align}
	\Bip{\frac{\partial}{\partial s^{(j)}} (\pi_2(a|\bs) - &\pi_1(a|\bs)),  Q^{\pi_1}(\bs, a)}_{a} \nonumber \\
	=\lim_{\delta s\rightarrow 0}\frac{1}{\delta s} \bigg[ & \bip{(\pi_2(a|\bs+\delta s\cdot \mathbf{e}_j)& - &\pi_1(a|\bs+\delta s\cdot \mathbf{e}_j)&), Q^{\pi_1}(\bs, a)}_{a} & \nonumber\\
	- & \bip{(\pi_2(a|\bs)& - &\pi_1(a|\bs)&), Q^{\pi_1}(\bs, a)}_{a} &\bigg].
	\end{align}
	We can rearrange the finite difference terms to get
	\begin{align}
	\Bip{\frac{\partial}{\partial s^{(j)}} (\pi_2(a|\bs) - & \pi_1(a|\bs)),  Q^{\pi_1}(\bs, a)}_{a} \nonumber\\
	=\lim_{\delta s\rightarrow 0}\frac{1}{\delta s} \bigg[ & \bip{(\pi_2(a|\bs+\delta s\cdot \mathbf{e}_j)& + &\pi_1(a|\bs) &),  Q^{\pi_1}(\bs, a)}_{a} & \nonumber\\
	- & \bip{(\pi_2(a|\bs)& + &\pi_1(a|\bs+\delta s\cdot \mathbf{e}_j) &),  Q^{\pi_1}(\bs, a)}_{a} &\bigg] .
	\end{align}
	Equivalently, we can divide and multiply the inner products by a factor of 2, to make the inner product arguments resemble mixture distributions:
	\begin{align}
	\Bip{\frac{\partial}{\partial s^{(j)}} (\pi_2(a|\bs) - \pi_1(a|\bs)),  Q^{\pi_1}(\bs, a)}_{a} \nonumber \\
	=\lim_{\delta s\rightarrow 0}\frac{2}{\delta s} \bigg[ & \Bip{\frac{\pi_2(a|\bs+\delta s\cdot \mathbf{e}_j) + \pi_1(a|\bs)}{2},  Q^{\pi_1}(\bs, a)}_{a} \nonumber\\
	- & \Bip{\frac{\pi_2(a|\bs) + \pi_1(a|\bs+\delta s\cdot \mathbf{e}_j)}{2},  Q^{\pi_1}(\bs, a)}_{a} \bigg] .
	\end{align}
	The RK duality can now be used to bound this difference as
	\begin{align}
	&\bigg|\Bip{\frac{\partial}{\partial s^{(j)}} (\pi_2(a|\bs) - \pi_1(a|\bs)),  Q^{\pi_1}(\bs, a)}_{a}\bigg| \\
	& \leq \lim_{\delta s\rightarrow 0}\frac{2}{\delta s} \bigg[ W\bigg(\frac{\pi_2(a|\bs+\delta s\cdot \mathbf{e}_j) + \pi_1(a|\bs)}{2}, \frac{\pi_2(a|\bs) + \pi_1(a|\bs+\delta s\cdot \mathbf{e}_j)}{2}\bigg) \cdot \Lip(Q^{\pi_1}(\bs, a);a) \bigg] \nonumber,
	\end{align}which can be simplified as
	\begin{align}\label{eq:dimconversionbound}
	&\bigg|\Bip{\frac{\partial}{\partial s^{(j)}} (\pi_2(a|\bs) - \pi_1(a|\bs)),  Q^{\pi_1}(\bs, a)}_{a}\bigg| \nonumber \\
	& \leq 2\cdot \Lip(Q^{\pi_1}(\bs, a);a) \cdot \frac{\partial}{\partial s'^{(j)}} W\bigg(\frac{\pi_2(a|\bs') + \pi_1(a|\bs)}{2}, \frac{\pi_2(a|\bs) + \pi_1(a|\bs')}{2}\bigg)\bigg|_{\bs'=\bs}.
	\end{align}
	Combining Inequality~\eqref{eq:dimconversionbound} and Equation~\eqref{eq:graddimdecomp}, we obtain the bound in the theorem.
\end{proof}

\subsection{The Preliminary Payoff Improvement Bound}\label{sec:boundstatement}
Combining the results of Inequality \eqref{eq:secondarybound} and Theorems~\ref{thm:gradtwoconcnt} and~\ref{thm:visitwass} leads us to define the regularization terms and coefficients:

$$C_1':=\sup_{s} \frac{2\cdot \Lip(Q^{\pi_1}(s, a);a)\cdot \gamma \cdot L_\pi}{(1-\gamma)(1-\gamma L_\mu)}, \qquad C_2':=\sup_{s} \frac{\| \Lip(\nabla_s Q^{\pi_1}(s, a); a) \|_2 \cdot \gamma \cdot L_\pi}{(1-\gamma)(1-\gamma L_\mu)},$$
\begin{align}
\mathcal{L}_{WG}(\pi_1, \pi_2; s) &:=W(\pi_2(a|s), \pi_1(a|s)) \nonumber\\
&\qquad \times \bigg\| \nabla_{s'} W\bigg(\frac{\pi_2(a|s') + \pi_1(a|s)}{2}, \frac{\pi_2(a|s)  +\pi_1(a|s')}{2}\bigg)\bigg|_{s'=s} \bigg\|_2.
\label{eq:c1c2lwgdef}
\end{align}

This gives us the following novel lower bound for payoff improvement.
\begin{theorem}
	\label{thm:wgpayoffbound}
	Defining $C_1'$, $C_2'$, and $\mathcal{L}_{WG}(\pi_1, \pi_2; s)$ as in~\eqref{eq:c1c2lwgdef}, we have  $\eta_{\pi_2} \geq \mathcal{L}_{\pi_1}^{\sup'}(\pi_2)$, where
	\begin{align}
         \label{eq:wtrpompibound}
         \mathcal{L}_{\pi_1}^{\sup'}(\pi_2) &:= \eta_{\pi_1} +\frac{1}{1-\gamma} \mathbb{E}_{s\sim \rho^{\pi_1}_\mu}[A^{\pi_1}(s,\pi_2)] - C_1' \cdot \sup_{s} \bigg[ \mathcal{L}_{WG}(\pi_1, \pi_2; s) \bigg]\nonumber\\
          &\qquad - C_2' \cdot \sup_{s} \bigg[ W(\pi_2(a|s), \pi_1(a|s))^2 \bigg].
         \end{align} 
\end{theorem}
\begin{proof}
By first inserting Theorem~\ref{thm:gradtwoconcnt} into Inequality~\eqref{eq:secondarybound}, and then applying it to Inequality~\eqref{eq:rksecorderbound} along with Theorem~\ref{thm:visitwass}, one can obtain the result according to the advantage decomposition lemma given in Equation~\eqref{eq:advdecompmainpaper} of the main paper.
\end{proof}

\subsection{Quadratic Modeling of Policy Sensitivity Regularization}\label{ss:quadmodelsurr}
First, we will build insight into the nature of the 
\begin{align}
\mathcal{L}_{WG}(\pi_1, \pi_2; s) = & W(\pi_2(a|s), \pi_1(a|s)) \times \nonumber \\
&\bigg\| \nabla_{s'} W\bigg(\frac{\pi_2(a|s') + \pi_1(a|s)}{2}, \frac{\pi_2(a|s) + \pi_1(a|s')}{2}\bigg)\bigg|_{s'=s} \bigg\|_2
\end{align}
term. It is fairly obvious that 
\begin{equation}
W(\pi_2(a|s), \pi_1(a|s)) \big|_{\pi_2=\pi_1} = 0 .
\end{equation}
If $\pi_2=\pi_1$, then the two distributions $\frac{\pi_2(a|s') + \pi_1(a|s)}{2}$ and $\frac{\pi_2(a|s) + \pi_1(a|s')}{2}$ will be the same no matter what $s'$ is. In other words, 
\begin{equation}
\label{eq:zeroginsqder}
\pi_1 = \pi_2 \Rightarrow \forall s': W\bigg(\frac{\pi_2(a|s') + \pi_1(a|s)}{2}, \frac{\pi_2(a|s) + \pi_1(a|s')}{2}\bigg) = 0 .
\end{equation}

This means that
\begin{equation}
\bigg\| \nabla_{s'} W\bigg(\frac{\pi_2(a|s') + \pi_1(a|s)}{2}, \frac{\pi_2(a|s) + \pi_1(a|s')}{2}\bigg)\bigg|_{s'=s} \bigg\|_2 \bigg|_{\pi_2=\pi_1} = 0 .
\end{equation}

The Taylor expansion of the squared Wasserestein distance can be written as
\begin{equation}
W(\pi_2(a|s), \pi_1(a|s))^2 \big|_{\theta_2 = \theta_1 + \dth} = \frac{1}{2} \dth^T H_2 \dth + \text{h.o.t.} .
\end{equation}

Considering \eqref{eq:zeroginsqder} and similar to the previous point, one can write the following Taylor expansion
\begin{equation}
\bigg\| \nabla_{s'} W\bigg(\frac{\pi_2(a|s') + \pi_1(a|s)}{2}, \frac{\pi_2(a|s) + \pi_1(a|s')}{2}\bigg)\bigg|_{s'=s} \bigg\|_2^2 \bigg|_{\theta_2 = \theta_1 + \dth} = \dth^T H_1 \dth + \text{h.o.t.} .
\end{equation}

According to above, $\mathcal{L}_{WG}$ is the geometric mean of two functions of quadratic order. Although this makes $\mathcal{L}_{WG}$ of quadratic order (i.e., $\lim_{\dth\rightarrow 0} \frac{\mathcal{L}_{WG}(\alpha \dth)}{\mathcal{L}_{WG}(\dth)} =\alpha^2$ holds for any constant $\alpha$), this does not guarantee that $\mathcal{L}_{WG}$ is twice continuously differentiable w.r.t. the policy parameters, and may not have a defined Hessian matrix (e.g., $f(x_1, x_2) = |x_1x_2|$ is of quadratic order, yet is not twice differentiable). To avoid this issue, we compromise on the local model. According to the arithmetic mean and geometric mean inequality, for all $x_1, x_2 \geq 0$ and any non-zero $\alpha$, we have
\begin{equation}
x_1 x_2 \leq \frac{\alpha^2 x_1^2 + \alpha^{-2} x_2^{2}}{2}.
\end{equation}
Therefore, we can bound the $\mathcal{L}_{WG}$ term into two quadratic terms:
\begin{align}
\mathcal{L}_{WG}(\pi_1, \pi_2; s) \leq \frac{1}{2} \bigg(&\alpha^2 \cdot \bigg\| \nabla_{s'} W\bigg(\frac{\pi_2(a|s') + \pi_1(a|s)}{2}, \frac{\pi_2(a|s) + \pi_1(a|s')}{2}\bigg)\bigg|_{s'=s} \bigg\|_2^2  +\nonumber\\
& \alpha^{-2} \cdot W(\pi_2(a|s), \pi_1(a|s))^2 \bigg).
\label{eq:amgmineqlwg}
\end{align}

\subsection{The Final Payoff Improvement Guarantee}\label{sec:finalizedpayoffbound}
Inequality~\eqref{eq:amgmineqlwg} paves the way for our final payoff lower bound theorem.
\begin{theorem}
	\label{thm:finalpracticallowerbound}
	By defining $C_1 := \frac{C_1' \cdot \alpha^2}{2}$, $C_2:=(C_2'+\frac{C_1'}{2\alpha^2})$ with any non-zero $\alpha$,  and
	\begin{equation}
	\mathcal{L}_{G^2}(\pi_1, \pi_2; s) := \bigg\| \nabla_{s'} W\bigg(\frac{\pi_2(a|s') + \pi_1(a|s)}{2}, \frac{\pi_2(a|s) + \pi_1(a|s')}{2}\bigg)\bigg|_{s'=s} \bigg\|_2^2,
	\end{equation}
	we have $\eta_{\pi_2} \geq \mathcal{L}_{\pi_1}^{\sup}(\pi_2)$, where 
	\begin{align}
	\mathcal{L}_{\pi_1}^{\sup}(\pi_2) := & \frac{1}{1-\gamma} \cdot \mathbb{E}_{s\sim \rho^{\pi_1}_\mu}[A^{\pi_1}(s,\pi_2)] -C_1 \cdot \sup_s \bigg[ \mathcal{L}_{G^2}(\pi_1, \pi_2; s) \bigg] \nonumber \\
	& - C_2 \cdot \sup_s \bigg[ W(\pi_2(a|s), \pi_1(a|s))^2 \bigg].
	\label{eq:objformalnew}
	\end{align}
\end{theorem}
\begin{proof}
Applying Inequality~\eqref{eq:amgmineqlwg} into Theorem~\ref{thm:wgpayoffbound} gives the result.
\end{proof}

$C_1$ and $C_2$ will be the corresponding regularization coefficients to the ones defined in Theorem~\ref{thm:wgpayoffbound}. Due to the arbitrary $\alpha$ used in the bounding process, no constrain governs the $C_1$ and $C_2$ coefficients. Therefore, $C_1$ and $C_2$ can be chosen without constraining each other.

\subsection{Proof of Theorem~\ref{thm:accuratevinepg}}\label{sec:proofdevine}
Essentially, DeVine rolls out a trajectory and computes the values of each state. Since the transition dynamics and the policy are deterministic, these values are exact. Then, it picks a perturbation state $s_t$ according to the visitation frequencies using importance sampling. A state-reset to $s_t$ is made, a $\sigma$-perturbed action is applied for a single time-step, followed by $\pi_1$ policy. This exactly produces $Q^{\pi_1}(s_t,a_t+\sigma)$. Then, $A^{\pi_1}(s_t,a_t+\sigma)$ can be computed by subtracting the value baseline. Finally, $A^{\pi_1}(s_t,a_t)=0$ and $A^{\pi_1}(s_t,a_t+\sigma)$ define a two-point linear $A^{\pi_1}(s_t,a)$ model with respect to the action. Parallelization can be used to have as many states of the first roll-out included in the estimator as desired. The parameter $\sigma$ acts as an exploration parameter and a finite difference to establish derivatives. While $\sigma\simeq 0$ can produce exact gradients, larger $\sigma$ can build stabler interpolations.

We restate Theorem~\ref{thm:accuratevinepg} below for reference and now prove it.

\begin{customthm}{4.1}
	Assume a finite horizon MDP with both deterministic transition dynamics $P$ and initial distribution $\mu$, with maximal horizon length of $H$. Define $K=H\cdot \dim(\mathcal{A})$, a uniform $\nu$, and $q(s;\sigma)=\pi_1(s) + \sigma \mathbf{e}_j$ in Algorithm 2 with $\mathbf{e}_j$ being the $j^{th}$ basis element for $\mathcal{A}$. If the $(j,t_k)$ pairs are sampled to exactly cover $\{1,\ldots,\dim(\mathcal{A})\} \times \{1,\ldots,H\}$, then we have
	\begin{equation}
	\lim_{\sigma\rightarrow 0}\nabla_{\pi_2} \mathbb{A}^{\pi_1}(\pi_2)\big|_{\pi_2=\pi_1} = \nabla_{\pi_2} \eta_{\pi_2}\big|_{\pi_2=\pi_1}.
	\end{equation}
\end{customthm}

\begin{proof}
	According to the advantage decomposition lemma, we have
	\begin{equation}
	\label{eq:advstep1}
	\nabla_{\pi_2} \eta_{\pi_2}\big|_{\pi_2=\pi_1} = \frac{1}{1-\gamma} \EE_{s\sim \rho_\mu^{\pi_1}}[\nabla_{\pi_2} A^{\pi_1}(s,\pi_2)]\big|_{\pi_2=\pi_1}.
	\end{equation} Due to the fact that the transition dynamics, policies $\pi_1$ and $\pi_2$, and initial state distribution are all deterministic, we can simplify Equation~\eqref{eq:advstep1} to 
	\begin{equation}
	\nabla_{\pi_2} \eta_{\pi_2}\big|_{\pi_2=\pi_1} = \sum_{t=0}^{H-1} \gamma^t \cdot \nabla_{\pi_2} A^{\pi_1}(s_t,\pi_2) \big|_{\pi_2=\pi_1},
	\label{eq:advstep2}
	\end{equation}
	where $s_t$ is the state after applying the policy $\pi_1$ for $t$ time-steps. We can use the chain rule to write 
	\begin{align}
	\nabla_{\pi_2} A^{\pi_1}(s_t,\pi_2) \big|_{\pi_2=\pi_1} &= \nabla_{\pi_2} A^{\pi_1}(s_t,\ba_t) \big|_{\substack{\ba_t = \pi_2(s_t)\\ \pi_2=\pi_1}}\nonumber\\
	&= \sum_{j=1}^{\dim(\mathcal{A})} \nabla_{\pi_2} a_t^{(j)} \cdot \frac{\partial}{\partial a_t^{(j)}} A^{\pi_1}(s_t,\ba_t) \big|_{\substack{\ba_t = \pi_2(s_t)\\ \pi_2=\pi_1}}.
	\label{eq:advstep3}
	\end{align}
	To recap, Equations~\eqref{eq:advstep2},~\eqref{eq:advstep2}, and~\eqref{eq:advstep3} can be summarized as 
	\begin{equation}
	\label{eq:advrhs}
	\nabla_{\pi_2} \eta_{\pi_2}\big|_{\pi_2=\pi_1} = \sum_{t=0}^{H-1} \gamma^t \sum_{j=1}^{\dim(\mathcal{A})} \nabla_{\pi_2} a_t^{(j)} \cdot \frac{\partial}{\partial a_t^{(j)}} A^{\pi_1}(s_t,\ba_t) \big|_{\substack{\ba_t = \pi_2(s_t)\\ \pi_2=\pi_1}}.
	\end{equation}
	
	Under the assumption that the $(j,t)$ pairs are sampled to exactly cover $\{1,\ldots,\dim(\mathcal{A})\} \times \{1,\ldots,H\}$, we can simplify the DeVine oracle to
	\begin{align}
	\label{eq:advlhsstep1}
	\mathbb{A}^{\pi_1}(\pi_2) = \frac{1}{K} \sum_{t=0}^{H-1} \sum_{j=1}^{\dim(\mathcal{A})} \bigg[ & \frac{\dim(\mathcal{A})\cdot\gamma^t}{\nu(t)} \cdot \frac{(\pi_2(s_t) -\pi_1(s_t))^T(q(s_t;j,\sigma) -\pi_1(s_t))}{(q(s_t;j,\sigma) -\pi_1(s_t))^T(q(s_t;j,\sigma) -\pi_1(s_t))} \nonumber\\
	& \cdot A^{\pi_1}(s_t, q(s_t;j,\sigma)) \bigg].
	\end{align} From the $q$ definition, we have $q(s_t;j,\sigma) -\pi_1(s_t) = \sigma \mathbf{e}_j$ and $(q(s_t;j,\sigma) -\pi_1(s_t))^T(q(s_t;j,\sigma) -\pi_1(s_t)) = \sigma^2$. Since $\nu$ is uniform (i.e., $\nu(t) = \frac{1}{H}$) and $K=H\cdot \dim(\mathcal{A})$, we can take the policy gradient of Equation~\eqref{eq:advlhsstep1} and simplify it into 
	\begin{align}
	\label{eq:advlhsstep2}
	\nabla_{\pi_2} \mathbb{A}^{\pi_1}(\pi_2)\big|_{\pi_2=\pi_1} = \sum_{t=0}^{H-1} \sum_{j=1}^{\dim(\mathcal{A})} \bigg[ \gamma^t \cdot \nabla_{\pi_2}  (\pi_2(s_t) -\pi_1(s_t))^T\mathbf{e}_j 
	\cdot \frac{A^{\pi_1}(s_t, q(s_t;j,\sigma))}{\sigma} \bigg].
	\end{align} Since, $A^{\pi_1}(s_t, \pi_1(s_t))=0$, we can write
	\begin{align}
	\lim_{\sigma\rightarrow 0}\frac{A^{\pi_1}(s_t, q(s_t;j,\sigma))}{\sigma} &= \lim_{\sigma\rightarrow 0}\frac{A^{\pi_1}(s_t, \pi_1(s_t) + \sigma \mathbf{e}_j) - A^{\pi_1}(s_t, \pi_1(s_t))}{\sigma} \nonumber\\
	&= \frac{\partial}{\partial a_t^{(j)}} A^{\pi_1}(s_t,\ba_t) \big|_{\ba_t = \pi_1(s_t)}.
	\label{eq:advlhsstep3}
	\end{align} Also, by the definition of the gradient, we can write 
	\begin{equation}
	\nabla_{\pi_2}  (\pi_2(s_t) -\pi_1(s_t))^T\mathbf{e}_j = \nabla_{\pi_2} a_t^{(j)}.
	\label{eq:advlhsstep4}
	\end{equation} Combining Equations~\eqref{eq:advlhsstep3} and~\eqref{eq:advlhsstep4}, and applying them to Equation~\eqref{eq:advlhsstep2}, yields
	\begin{align}
	\label{eq:advlhsstep5}
	\lim_{\sigma\rightarrow 0} \nabla_{\pi_2} \mathbb{A}^{\pi_1}(\pi_2)\big|_{\pi_2=\pi_1} = \sum_{t=0}^{H-1} \sum_{j=1}^{\dim(\mathcal{A})} \gamma^t \cdot \nabla_{\pi_2} a_t^{(j)} \cdot \frac{\partial}{\partial a_t^{(j)}} A^{\pi_1}(s_t,\ba_t) \big|_{\substack{\ba_t = \pi_2(s_t)\\ \pi_2=\pi_1}}.
	\end{align}Finally, the theorem can be obtained by comparing Equations~\eqref{eq:advlhsstep5} and~\eqref{eq:advrhs}.
	
\end{proof}

\subsection{The Discussion of Assumptions}\label{ss:allassumptions}
There are three key groups of assumptions made in the derivation of our policy improvement lower bound. First is the existence of $Q^\pi$-function Lipschitz constants. Second is the transition dynamics Lipschitz-continuity assumptions. Finally, we make an assumption about the stability of the transition dynamics. Next, we will discuss the meaning and the necessity of these assumptions.

\subsubsection{On the Existence of the $\text{Lip}(Q^\pi,a)$ Constant}\label{ss:qlipcontassumption}
The $\text{Lip}(Q^\pi,a)$ constant may be undefined when either the reward function or the transition dynamics are discontinuous. Examples of known environments with undefined $\text{Lip}(Q^\pi,a)$ constants include those with grazing contacts which define a discontinuous transition dynamics. In practice, even for environments that do not satisfy Lipschitz continuity assumptions, there are mitigating factors; practical $Q^\pi$ functions are reasonably narrow-bounded in a small trust-region neighborhood, and since we use non-vanishing exploration scales and trust regions, a bounded interpolation slope can still model the $Q$-function variation effectively. We should also note that a slightly stronger version of this assumption is frequently used in the context of Lipschitz MDPs \citep{pirotta2015policy,oatao17977,asadi2018lipschitz}. In practice, we have not found this to be a substantial limitation.

\subsubsection{The Transition Dynamics Lipschitz Continuity Assumption}\label{ss:lipcontassumptions}
Assumptions~\ref{eq:piassumption} and~\ref{eq:sassumption} of the main paper essentially represent the Lipschitz continuity assumptions of the transition dynamics with respect to actions and states, respectively. If the transition dynamics and the policy are deterministic, then these assumptions are exactly equivalent to the Lipschitz continuity assumptions. Assumptions~\ref{eq:piassumption} and~\ref{eq:sassumption} only generalize the Lipschitz continuity assumptions in a distributional sense. 

The necessity of these assumptions is a consequence of using metric measures for bounding errors. Traditional non-metric bounds force the use of full-support stochastic policies where all actions have non-zero probabilities (e.g., for the KL-divergence of two policies to be defined, TRPO needs to operate on full-support policies such as the Gaussian policies). In those analyses, since all policies share the same support, the next state distribution automatically becomes smooth and Lipschitz continuous with respect to the policy measure even if the transition dynamics were not originally smooth with respect to the input actions. However, metric measures are also defined for policies of non-overlapping support. To be able to provide closeness bounds for future state visitations of two similar policies with non-overlapping support, it becomes necessary to assume that close-enough actions or states must be yielding close-enough next states. In fact, this is a very common assumption in the framework of Lipschitz MDPs (See Section 2.2 of \citet{oatao17977}, Section 3 of \citet{asadi2018lipschitz}, and Assumption 1 of \citet{pirotta2015policy}).

\subsubsection{The Transition Dynamics Stability Assumption}\label{ss:dynamicstability}
Before moving to relax the $\gamma L_\mu < 1$ assumption, we will make a few definitions and restate the previous lemmas and theorems under these definitions. We define $L_{\mu_1,\mu_2,\pi}$ to be the infimum non-negative value that makes the equation $W(\PP(\mu_1, \pi), \PP(\mu_2, \pi)) = L_{\mu_1,\mu_2,\pi} W(\mu_1, \mu_2)$ hold. Similarly, $L_{\mu_1,\mu_2,\pi}$ is defined as the infimum non-negative value that makes the equation $W(\PP(\mu, \pi_1), \PP(\mu, \pi)) = L_{\mu,\pi_1,\pi_2} W(\pi_1(\cdot|\mu), \pi_2(\cdot|\mu))$ hold. For notation brevity, we will also denote $L_{\PP^{t}(\mu, \pi_1), \PP^{t}(\mu, \pi_2), \pi_2}$ and $L_{\PP^{t}(\mu, \pi_1), \pi_1, \pi_2}$ by $\tilde{L}_{\mu, \pi_1, \pi_2}^{(t)}$ and $\hat{L}_{\mu, \pi_1, \pi_2}^{(t)}$, respectively.

Under these definitions, Lemma~\ref{lemma:fullasumption} evolves into 
\begin{equation}
W(\PP(\mu_1, \pi_1), \PP(\mu_2, \pi_2)) \leq L_{\mu_1,\mu_2,\pi} W(\mu_1, \mu_2) + L_{\mu_1, \pi_1, \pi_2} W(\pi_1, \pi_2).
\end{equation}
We can apply a time-point recursion to this lemma and have
\begin{align}
&W(\PP(\PP^{t}(\mu, \pi_1), \pi_1), \PP(\PP^{t}(\mu, \pi_2), \pi_2))\nonumber\\
&\leq L_{\PP^{t}(\mu, \pi_1), \pi_1, \pi_2}W(\pi_1, \pi_2) + L_{\PP^{t}(\mu, \pi_1), \PP^{t}(\mu, \pi_2), \pi_2} W(\PP^{t}(\mu, \pi_1), \PP^{t}(\mu, \pi_2))
\end{align}
, which can be notationally simplified to
\begin{equation}
W(\PP^{t}(\mu, \pi_1), \PP^{t}(\mu, \pi_2)) \leq \hat{L}_{\mu, \pi_1, \pi_2}^{(t-1)} W(\pi_1, \pi_2) + \tilde{L}_{\mu, \pi_1, \pi_2}^{(t-1)} W(\PP^{t-1}(\mu, \pi_1), \PP^{t-1}(\mu, \pi_2)).
\end{equation}

These modifications lead Lemma~\ref{lemma:ptbound} to be updated accordingly into
\begin{equation}
W(\PP^t(\mu, \pi_1), \PP^t(\mu, \pi_2)) \leq C_{L;\mu,\pi_1,\pi_2}^{(t)}\cdot W(\pi_1, \pi_2)
\end{equation}
, where we have
\begin{equation}
C_{L;\mu,\pi_1,\pi_2}^{(t)} := \sum_{k=1}^t \hat{L}_{\mu, \pi_1, \pi_2}^{(t-k)} \prod_{i=1}^{k-1} \tilde{L}_{\mu, \pi_1, \pi_2}^{(t-i)}.
\end{equation}
By a simple change of variables, we can have the equivalent definition of
\begin{equation}
C_{L;\mu,\pi_1,\pi_2}^{(t)} := \sum_{k=1}^t \hat{L}_{\mu, \pi_1, \pi_2}^{(k-1)} \prod_{i=k+1}^{t-1} \tilde{L}_{\mu, \pi_1, \pi_2}^{(i)}.
\end{equation}

Now, we would replace the $\gamma L_\mu < 1$ assumption with the following assumption.

\textbf{The Transition Dynamics Stability Assumption}: A transition dynamics $P$ is called stable if and only if the induced $\{\tilde{L}_{\mu, \pi_1, \pi_2}^{(t)}\}_{t\geq 0}$ and $\{\hat{L}_{\mu, \pi_1, \pi_2}^{(t)}\}_{t\geq 0}$ sequences satisfy
\begin{equation}
C_L := \sup_{\mu, \pi_1, \pi_2, t} C_{L;\mu,\pi_1,\pi_2}^{(t)} = \sup_{\mu, \pi_1, \pi_2, t} \sum_{k=1}^t \hat{L}_{\mu, \pi_1, \pi_2}^{(k-1)} \prod_{i=k+1}^{t-1} \tilde{L}_{\mu, \pi_1, \pi_2}^{(i)} < \infty .
\end{equation}

To help understand which $\{\tilde{L}_{\mu, \pi_1, \pi_2}^{(t)}\}_{t\geq 0}$ and $\{\hat{L}_{\mu, \pi_1, \pi_2}^{(t)}\}_{t\geq 0}$ sequences can satisfy this assumption, we will provide some examples:

\begin{itemize}
\item Having $\forall t: \tilde{L}_{\mu, \pi_1, \pi_2}^{(t)} = c_1 > 1$, $\hat{L}_{\mu, \pi_1, \pi_2}^{(t)} = c_2$ violates the dynamics stability assumption.
\item Having $\forall t: \tilde{L}_{\mu, \pi_1, \pi_2}^{(t)} \leq 1, \hat{L}_{\mu, \pi_1, \pi_2}^{(t)} = O(\frac{1}{t^2})$ sequences satisfy the dynamics stability assumption.
\item Having $\sup_t \tilde{L}_{\mu, \pi_1, \pi_2}^{(t)} < 1$ guarantees the dynamics stability assumption.
\item Having $\forall t\geq t_0: \tilde{L}_{\mu, \pi_1, \pi_2}^{(t)} < 1$ guarantees the dynamics stability assumption no matter (1) how big $t_0$ is (as long as it is finite), or (2) how big the members of the finite set $\{\tilde{L}_{\mu, \pi_1, \pi_2}^{(t)} | t < t_0\}$ are.
\end{itemize}

If the dynamics stability assumption holds with a constant $C_L$, one can define a $\bar{L}_\mu$ constant such that $C_L = L_\pi \sum_{t=0}^{\infty}(\gamma\bar{L}_\mu)^t$.  Then, we can replace all the $L_\mu$ instances in the rest of the proof with the corresponding $\bar{L}_\mu$ constant, and the results will remain the same without any change of format.

The $\tilde{L}_{\mu, \pi_1, \pi_2}^{(t)}$ and  $\hat{L}_{\mu, \pi_1, \pi_2}^{(t)}$ constants can be thought as tighter versions of $L_\mu$ and $L_\pi$, but with dependency on $\pi_1$, $\pi_2$, $\mu$ and the time-point of application. Having $\gamma L_{\mu}<1$ is a sufficient yet unnecessary condition for this dynamics stability assumption to hold. Vaguely speaking, $L_\mu$ is an expansion rate for the state distribution distance; it tells you how much the divergence in the state distribution will expand after a single application of the transition dynamics. Having effective expansion rates that are larger than one throughout an infinite horizon trajectory is a sign of the system instability; some change in the initial state's distribution could cause the observations to diverge exponentially. While controlling unstable systems is an important and practical challenge, none of the existing reinforcement learning methods is capable of learning effective policies in such environments. Roughly speaking, having the dynamics stability assumption guarantees that the expansion rates cannot be consistently larger than one for infinite time-steps.

\newpage
\section{Implementation Details and Supplementary Results}

\subsection{Implementation Details for the Environment with Non-local Rewards}
\label{sec:impl-deta-envir}

We used the stable-baselines implementation \citep{stablebaselines}, which has the same structure as the original OpenAI baselines \citep{baselines} implementation. We used the ``ppo1'' variant since no hardware acceleration was necessary for automatic differentiation and MPI parallelization was practically efficient. TDPO, TRPO, and PPO used the same function approximation architecture with two hidden layers, 64 units in each layer, and the tanh activation. TRPO, PPO, DDPG, and TD3 used their default hyper-parameter settings.  We used the same method of network initialization as TRPO and PPO (Xavier initialization~\citep{glorot2010understanding} with default gains for the inner layers, and smaller gain for the output layer). TD3's baseline implementation was amended to support MPI parallelization just like TRPO, PPO, and DDPG. To produce the results for DDPG and TD3, we used hyper-parameter optimization both with and without the $\tanh$ final activation function that is common for DDPG and TD3 (this causes the difference in initial payoff in the figures). However, under no conditions were DDPG and TD3 able to solve these environments effectively, suggesting that the deterministic search used by TDPO is operating in a qualitatively different way than the stochastic policy optimization used by DDPG and TD3. Note that we made a thorough attempt to compare DDPG and TD3 fairly, including trying different initializations, different final layer scalings/activations, different network architectures, and performing hyperparameter optimization. Mini-batch selection was unnecessary for TDPO since optimization for samples generated by DeVine was fully tractable. The confidence intervals in all figures were generated using 1000 samples of the statistics of interest.


For designing the environment, we used \citet{baselines}'s pendulum dynamics and relaxed the torque thresholds to be as large as $40\rm~N~m$. The environment also had the same episode length of 200 time-steps. We used the reward function described by the following equations:
\begin{align}
R(s_t,a_t) &= C_R \cdot R(\tau) \cdot \mathbf{1}\{t=200\}\nonumber\\
R(\tau) &=  R_{\text{Freq}}(\tau) + R_{\text{Offset}}(\tau) + R_{\text{Amp}}(\tau)\nonumber\\
R_{\text{Freq}}(\tau) &= 0.1 \cdot \bigg[\sum_{f=f_{\min}}^{f_{\max}} \Theta^+_{\text{std}}(f)^2 -1\bigg]\nonumber\\
R_{\text{Offset}}(\tau) &= -\bigg|\frac{\Theta(f=0)}{200} - \theta_{\text{Target Offset}}\bigg| = -\bigg|\bigg(\frac{1}{200}\sum_{t=0}^{199}\theta_t\bigg)  - \theta_{\text{Target Offset}}\bigg|\nonumber\\
R_{\text{Amp}}(\tau) &=  h_{\text{piecewise}}\bigg(\frac{\Theta_{\text{AC}}}{\theta_{\text{Target Amp.}}} - 1\bigg)\nonumber\\
\end{align}
where
\begin{itemize}
	\item $\theta$ is the pendulum angle signal in the time domain.
	\item $\Theta$ is the magnitude of the Fourier transform of $\theta$.
	\item $\Theta^{+}$ is the same as $\Theta$ only for the positive frequency components.
	\item $\Theta_{\text{AC}}$ is the normalized oscillatory spectrum of $\Theta$:
	\begin{equation}
	\Theta_{\text{AC}}= \frac{\sqrt{{\Theta^{+}}^T{\Theta^{+}}}}{200}.
	\end{equation}
	\item $h_{\text{piecewise}}$ is a piece-wise linear error penalization function:
	\begin{equation}
	h_{\text{piecewise}}(x) = -x\cdot \mathbf{1}\{x\geq 0\} + 10^{-4} x \cdot\mathbf{1}\{-x\geq 0\} .
	\end{equation}
	\item $\Theta^+_{\text{std}}$ is the standardized positive amplitudes vector:
	\begin{equation}
	\Theta^+_{\text{std}} = \frac{\Theta^+}{\sqrt{{\Theta^{+}}^T{\Theta^{+}}} + 10^{-6}} .
	\end{equation}
	\item $C_R=1.3\times 10^{4}$ is a reward normalization coefficient and was chosen to yield approximately the same payoff as a null policy would yield in the typical pendulum environment of \citet{baselines}.
	\item $\theta_{\text{Target Offset}}$ is the target offset, $\theta_{\text{Target Amp.}}$ is the target amplitude, and $[f_{\min}, f_{\max}]$ is the target frequency range of the environment.
\end{itemize} 

All methods used 48 parallel workers. The machines used Xeon E5-2690-v3 processors and 256 GB of memory. Each experiment was repeated 25 times for each method, and each run was given 6 hours or 500 million samples to finish.

\subsection{Implementation Details for the Environment with Long Horizon and Resonant Frequencies}\label{legexpdetails}

For the robotic leg, we used exactly the same algorithms with the same parameters as described in Section~\ref{sec:impl-deta-envir} above.

We used the reward function described by the following equations:
\begin{equation}
R = R_{\text{posture}} + R_{\text{velocity}} + R_{\text{foot offset}} + R_{\text{foot height}} + R_{\text{ground force}} + R_{\text{knee height}} + R_{\text{on-air torques}}
\end{equation}
with
\begin{align}
R_{\text{posture}} &= -1 \times \bigg[\bigg|\theta_{\text{knee}} + \frac{\pi}{2}\bigg| + \bigg|\theta_{\text{hip}} + \frac{\pi}{4}\bigg|\bigg]\nonumber\\
R_{\text{velocity}} &= -0.08\times \big[|\omega_{\text{knee}}| + |\omega_{\text{hip}}|\big]\nonumber\\
R_{\text{foot offset}} &= -10\times \big[|x_{\text{foot}}| \cdot \mathbf{1}\{z_{\text{knee}} < 0.2\}\big]\nonumber\\
R_{\text{ground force}} &= -1\times \big[|f_z-m g| \cdot \mathbf{1}\{f_z < m g\} \cdot \mathbf{1}_{\text{touchdown}}\big]\nonumber\\
R_{\text{foot height}} &= -1 \times \big[|z_{\text{foot}}|\cdot\mathbf{1}_{\text{touchdown}}\big]\nonumber\\
R_{\text{knee height}} &= -15 \times \big[ \big|z_{\text{knee}}-z^{\text{target}}_{\text{knee}}\big| \cdot \mathbf{1}_{\text{touchdown}} \big]\nonumber\\
R_{\text{on-air torques}} &= -10^{-4} \times \big[ (\tau_{\text{knee}}^2+\tau_{\text{hip}}^2) \cdot (1-\mathbf{1}_{\text{touchdown}}) \big]
\end{align}
where 
\begin{itemize}
	\item $\theta_{\text{knee}}$ and $\theta_{\text{hip}}$ are the knee and hip angles in radians, respectively.
	\item $\omega_{\text{knee}}$ and $\omega_{\text{hip}}$ are the knee and hip angular velocities in radians per second, respectively.
	\item $x_{\text{foot}}$ and $z_{\text{foot}}$ are the horizontal and vertical foot offsets in meters from the desired standing point on the ground, respectively.
	\item $x_{\text{knee}}$ and $z_{\text{knee}}$ are the horizontal and vertical knee offsets in meters from the desired standing point on the ground, respectively.
	\item $f_z$ is the vertical ground reaction force on the robot in Newtons.
	\item $m$ is the robot weight in kilograms (i.e., $m=0.76$ kg).
	\item $g$ is the gravitational acceleration in meters per second squared.
	\item $\mathbf{1}_{\text{touchdown}}$ is the indicator function of whether the robot has ever touched the ground.
	\item $z^{\text{target}}_{\text{knee}}$ is a target knee height of 0.1~m.
	\item $\tau_{\text{knee}}$ and $\tau_{\text{hip}}$ are the knee and hip torques in Newton meters, respectively.
\end{itemize}

All methods used 72 full trajectories between each policy update, and each run was given 16 hours of wall time, which corresponded to almost 500 million samples. This experiment was repeated 75 times for each method. The empirical means of the discounted payoff values were reported without any performance or seed filtration. The same hardware as the non-local rewards experiments (i.e., Xeon E5-2690-v3 processors and 256 GB of memory) was used.

\subsection{Gym Suite Benchmarks}\label{ss:gymbench}

While it is clear that our deterministic policy gradient performs well on the new control environments we consider, one may naturally wonder about its performance on existing RL control benchmarks. To show our method's core capability, we ran the basic variant of our method on a suite of Gym environments and include four representative examples in Figure~\ref{fig:gym}. Broadly speaking, our method (TDPO) performs similar to others, but occasionally performs much better as seen in the Swimmer-v3 environment. Such an improvement is quite surprising since (a) we used the basic variant of our method without any line-search for the update coefficient or adaptively tuning the exploration scale, and (b) we did not perform any hyper-parameter optimization on our method. On the other hand, TRPO, PPO, and TD3 have been highly optimized in the prior work, and include many tuned adaptive processes internally. We speculate that many of these gym environments are reasonably robust to any injected noise, and this may mean that stochastic policy gradients can more rapidly and efficiently explore the policy space than in our new control environments. The experiments granted each method 72 parallel MPI workers for about 144 million steps (i.e., 2 million sequential steps), and the returns were averaged over 100 different seeds for each method. Since the computational cost of running both DDPG and TD3 was high, we only included TD3 since it was shown to outperform DDPG in earlier benchmarks.

\begin{figure}[h]
	\centering
	\includegraphics[width=0.95\linewidth]{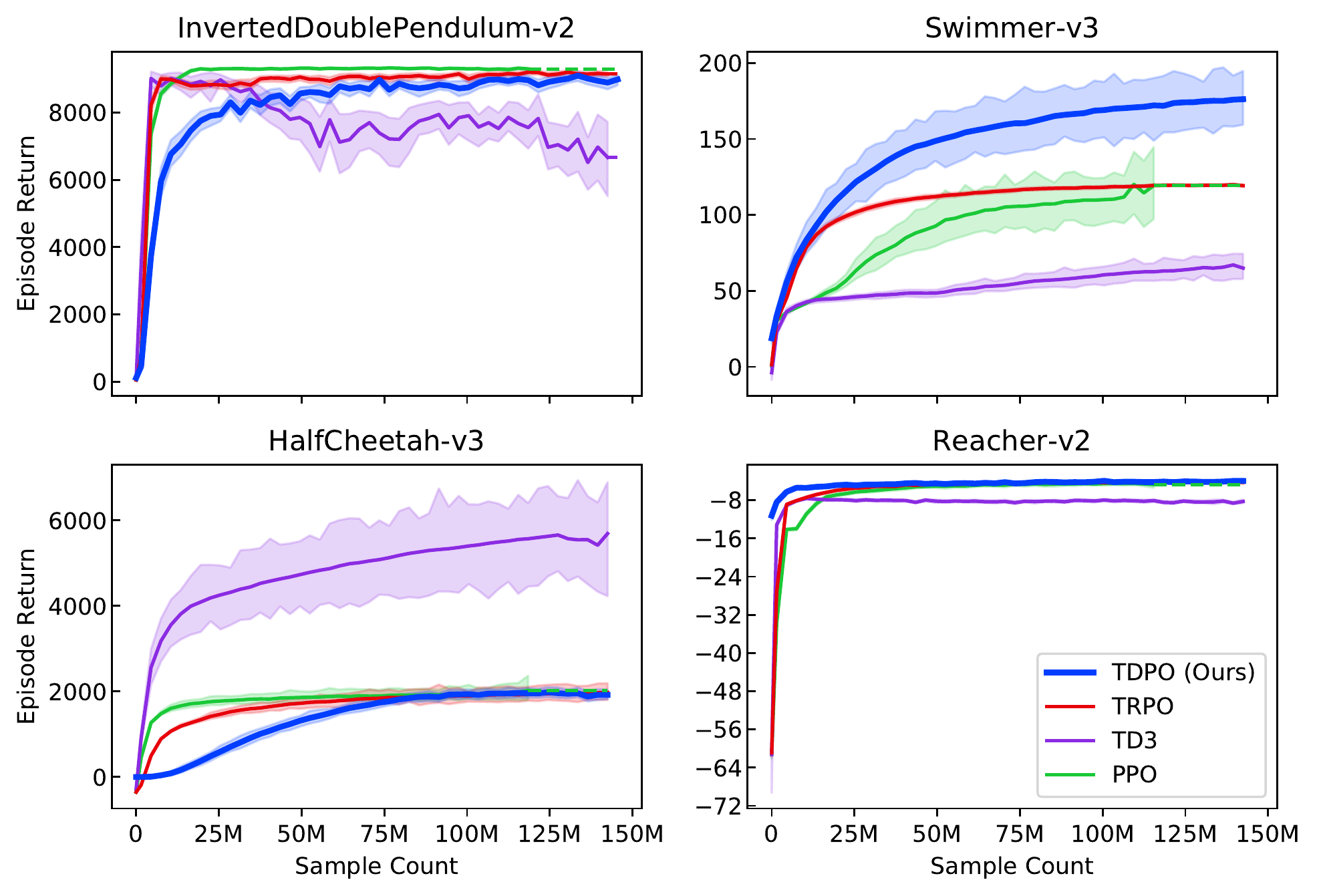}
	\caption{Results for the gym suite benchmarks. The basic variant of our method (TDPO with no line search processes or adaptive exploration scale parameters) with no hyper-parameter optimization was used in this experiment. Nevertheless, TRPO, PPO, and TD3 have many code-level optimizations and were meticulously tuned in the prior work. Still, TDPO managed to perform similar to other methods and even outperform the other methods in the Swimmer-v3 environment. While using the advanced variant of our method or performing hyper-parameter optimization can certainly improve our method further, we avoided them to showcase our core policy gradient method's capabilities.}
	\label{fig:gym}
	\vspace{-4mm}
\end{figure}


\subsection{Running Time Comparison}
Figure~\ref{fig:runningtime} depicts a comparison of each method's running time per million steps. These plots show the combination of both the simulation (i.e., environment sampling) and the optimization (i.e., computing the policy gradient and running the conjugate gradient solver) time. It is clear that our method (TDPO) is generally faster than the other algorithms. This is mainly due to the computational efficiency of the DeVine gradient estimator, which summarizes two full trajectories in a single state-action-advantage tuple which can significantly reduce the optimization time. That being said, these relative comparisons could vary to a large extent (1) under different processor architectures, (2) with more (or less) efficient implementations, or (3) when running environments whose simulation time constitutes a significantly larger (or smaller) portion of the total running time.
\begin{figure}[!ht]
	\centering
	\includegraphics[width=1\linewidth]{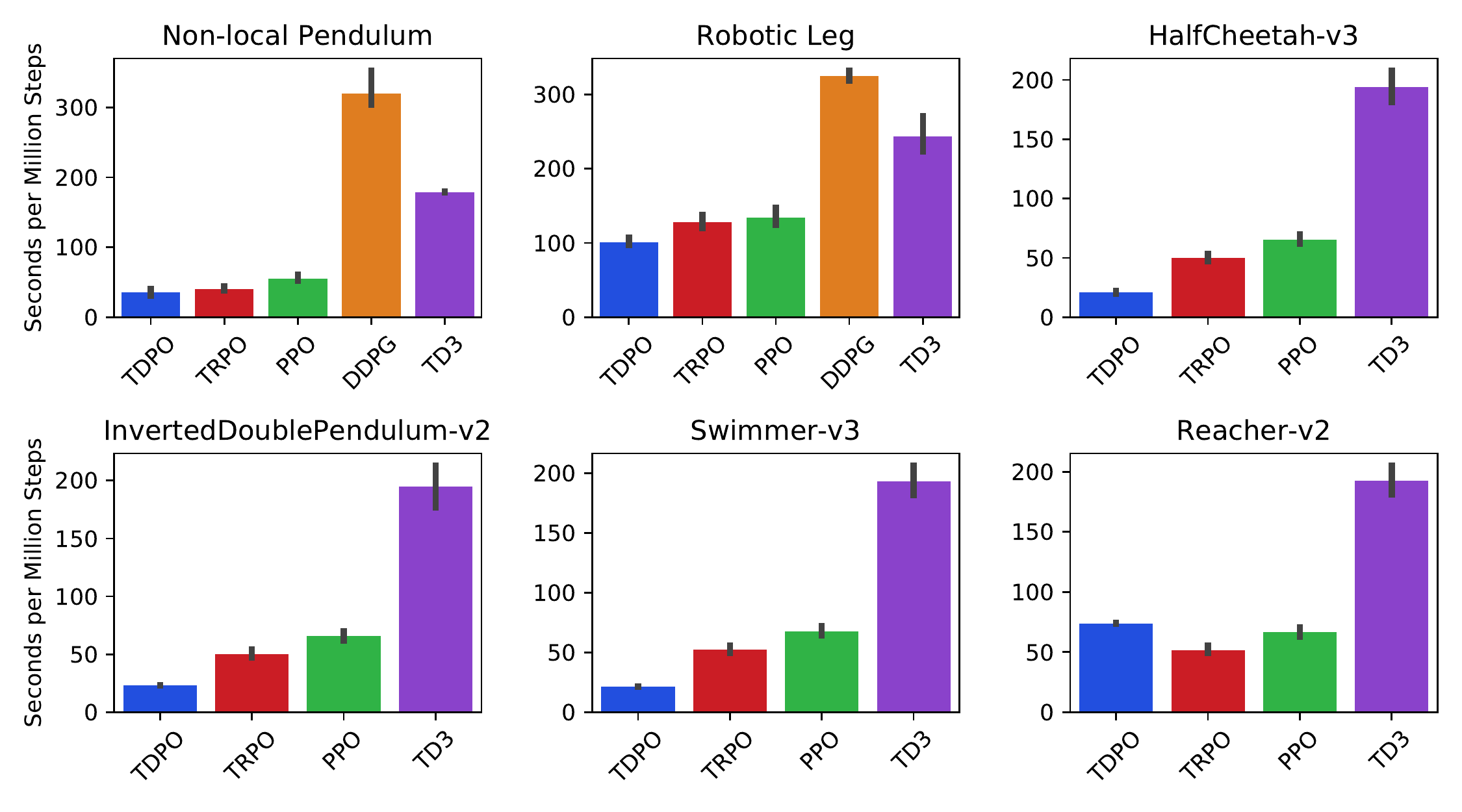}
	\caption{Training time comparison in different environments. The lower the bar, the faster the method. The vertical axis shows the time in seconds needed to consume one million state-action pairs for training. Each environment was shown separately in a different subplot.}
	\label{fig:runningtime}
\end{figure}

\subsection{Other Swinging Pendulum Variants}

Multiple variants of the pendulum with non-local rewards were used, each with different frequency targets and the same reward structure. Table~\ref{tab:nlrpendtarget} summarizes the target characteristics of each variant. The main variant was shown in the paper. Figures~\ref{fig:nlrv9}, ~\ref{fig:nlrv4}, ~\ref{fig:nlrv3},~\ref{fig:nlrv8},~\ref{fig:nlrv1},~\ref{fig:nlrv6},~\ref{fig:nlrv5}, and ~\ref{fig:nlrv7} show similar results for the second to ninth variants. To focus on our method's ability to solve all these variants efficiently, we only show the performance of our method (TDPO) on all variants in Figure~\ref{fig:nlrtdpo}. Overall, we found TRPO, PPO, DDPG, and TD3 to occasionally find the correct offset. They either excited the natural or the maximum (not the desired) frequency of the pendulum, but they were not able to drive the desired frequency and amplitude. TDPO was able to achieve the desired oscillations (and thus high rewards) in all variants.

\renewcommand{\arraystretch}{1.}
\begin{table}[h]
	\centering	
	\begin{tabular}{c|c|c|c}
		Pendulum Variant & Desired Frequency & Desired Offset & Desired Amplitude \\
		\hline
		Main & 1.7--2 Hz & 0.524 rad & 0.28 rad \\ 
		Second & 0.5--0.7 Hz & 1.571 rad & 1.11 rad \\ 
		Third & 2.5--3 Hz & 0.524 rad & 0.28 rad \\ 
		Fourth & 2--2.4 Hz & 0.785 rad & 0.28 rad \\ 
		Fifth & 2--2.4 Hz & 1.571 rad & 0.74 rad \\ 
		Sixth & 2--2.4 Hz & 0.524 rad & 0.28 rad \\ 
		Seventh & 2--2.4 Hz & 1.047 rad & 0.28 rad \\ 
		Eighth & 2--2.4 Hz & 0.785 rad & 0.74 rad \\ 
		Ninth & 2--2.4 Hz & 1.309 rad & 0.28 rad \\ 
	\end{tabular}%
	\caption{The target oscillation characteristics defining different pendulum swinging environments. }
	\label{tab:nlrpendtarget}%
\end{table}%

\insertnlrfig{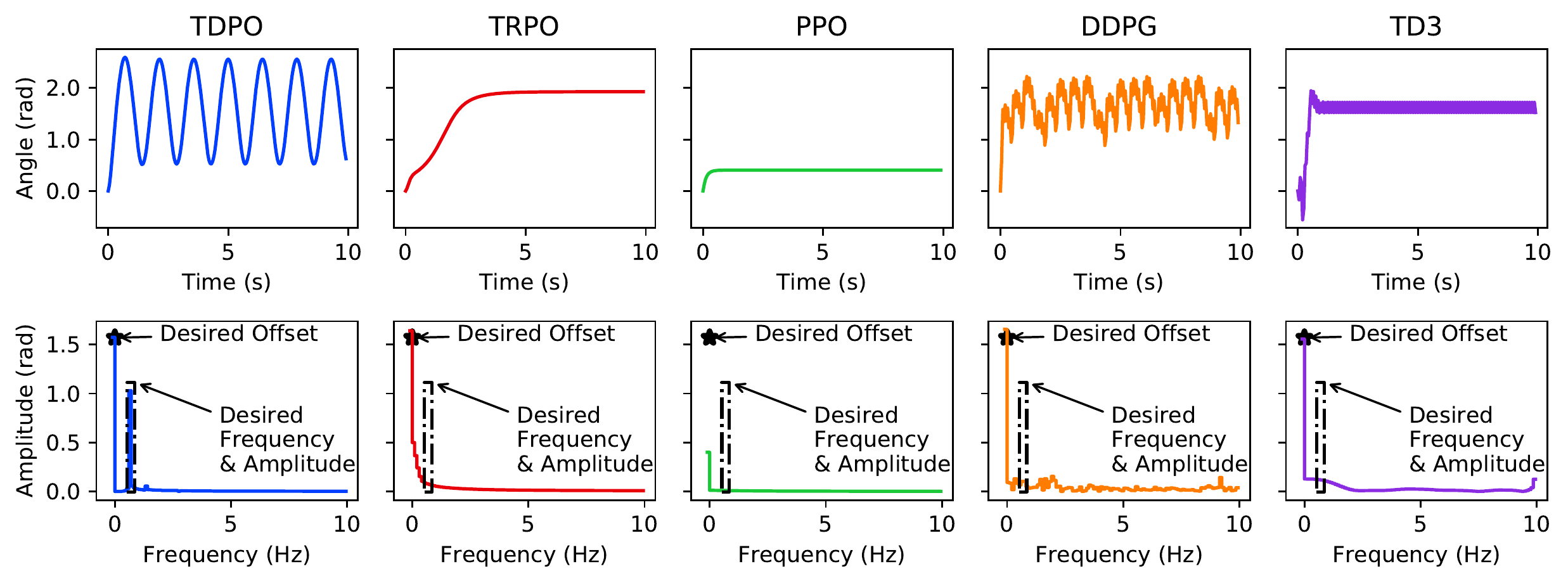}{second}
\insertnlrfig{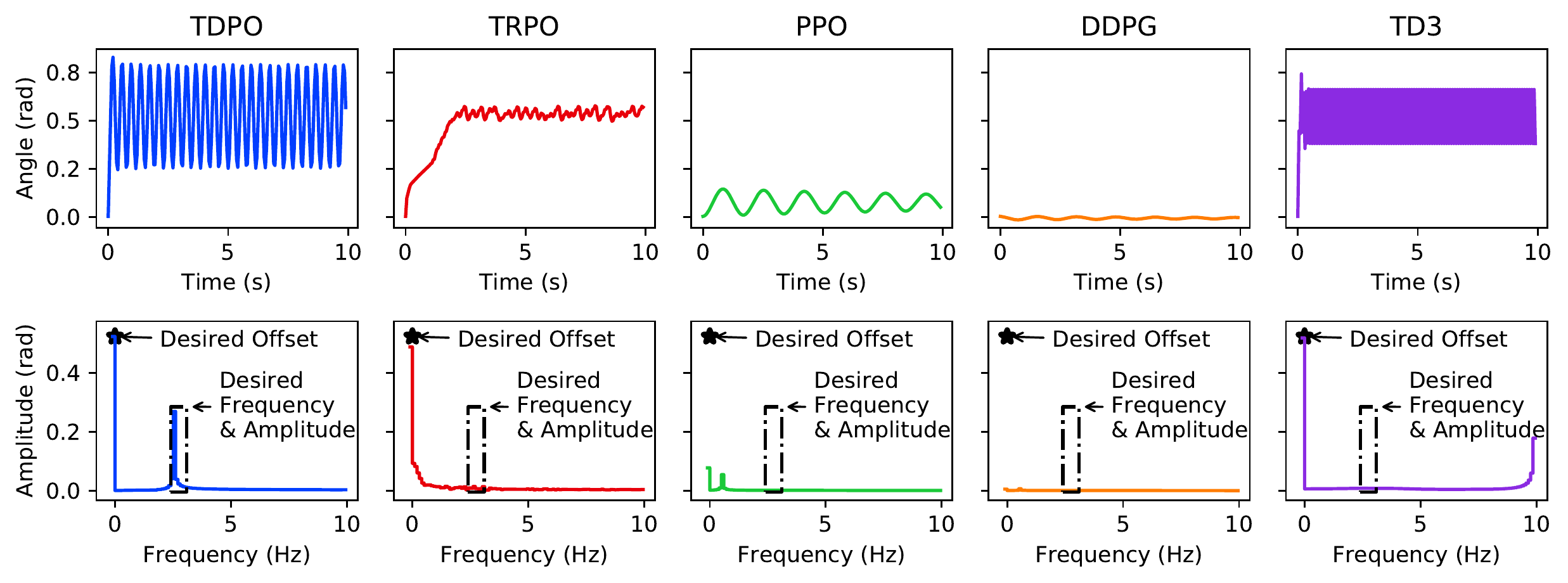}{third}
\insertnlrfig{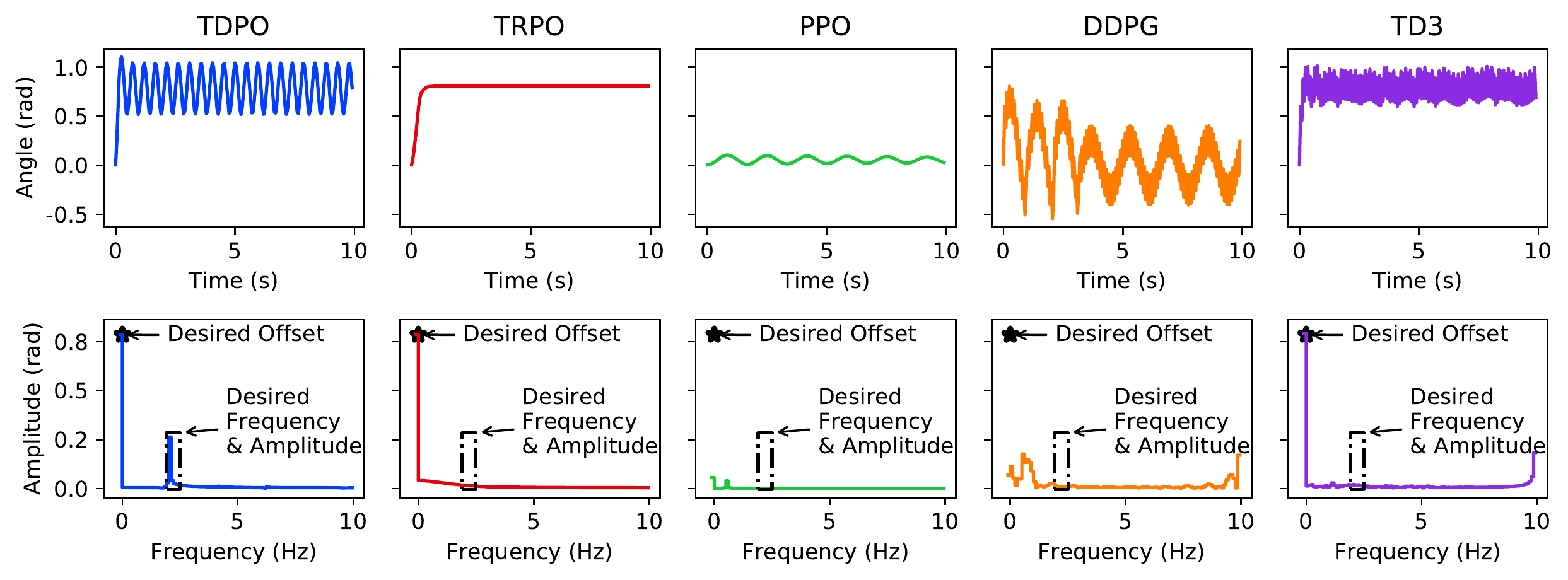}{fourth}
\insertnlrfig{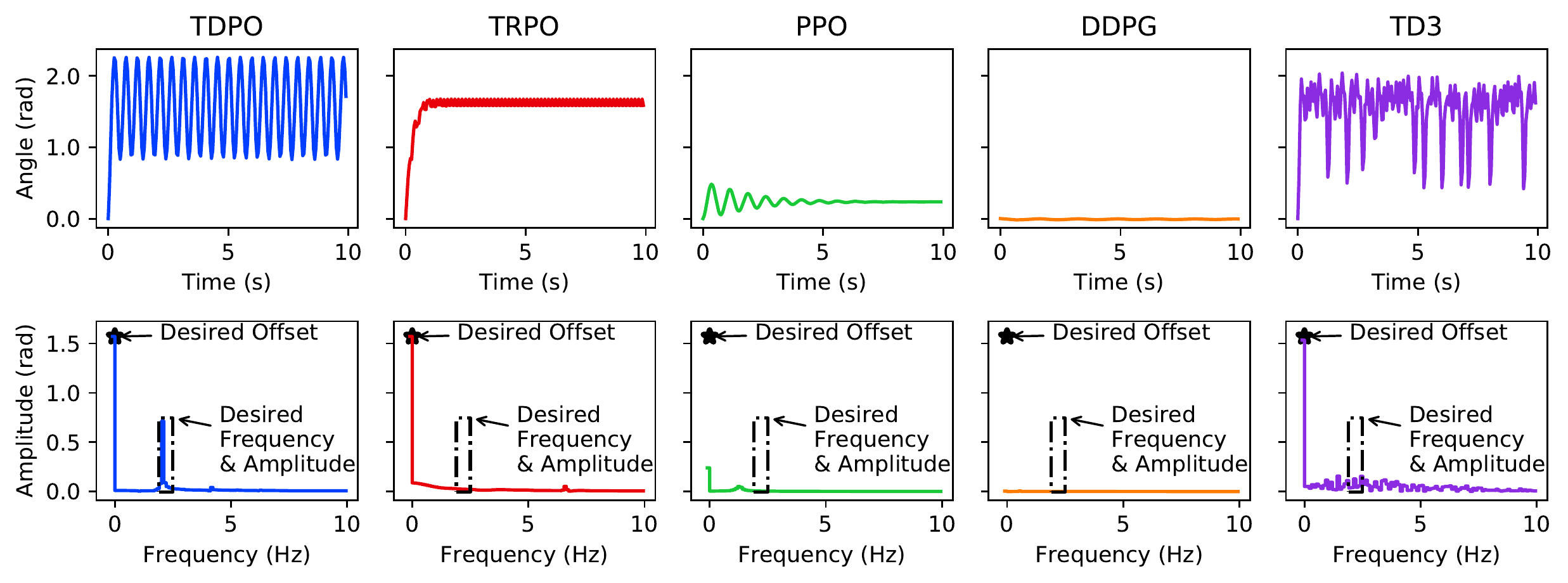}{fifth}
\insertnlrfig{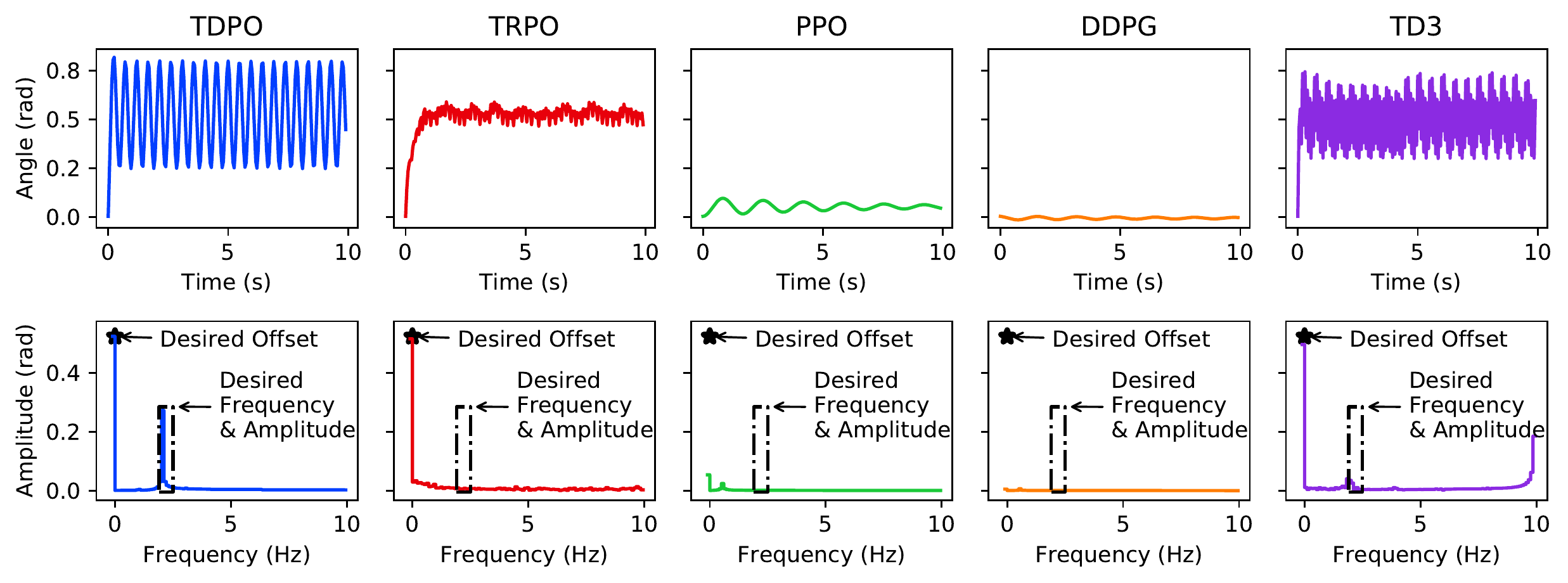}{sixth}

\begin{figure}
	
	\begin{minipage}[c]{1\textwidth}
		\centering
		\subfloat[]{
			\includegraphics[width=1\linewidth]{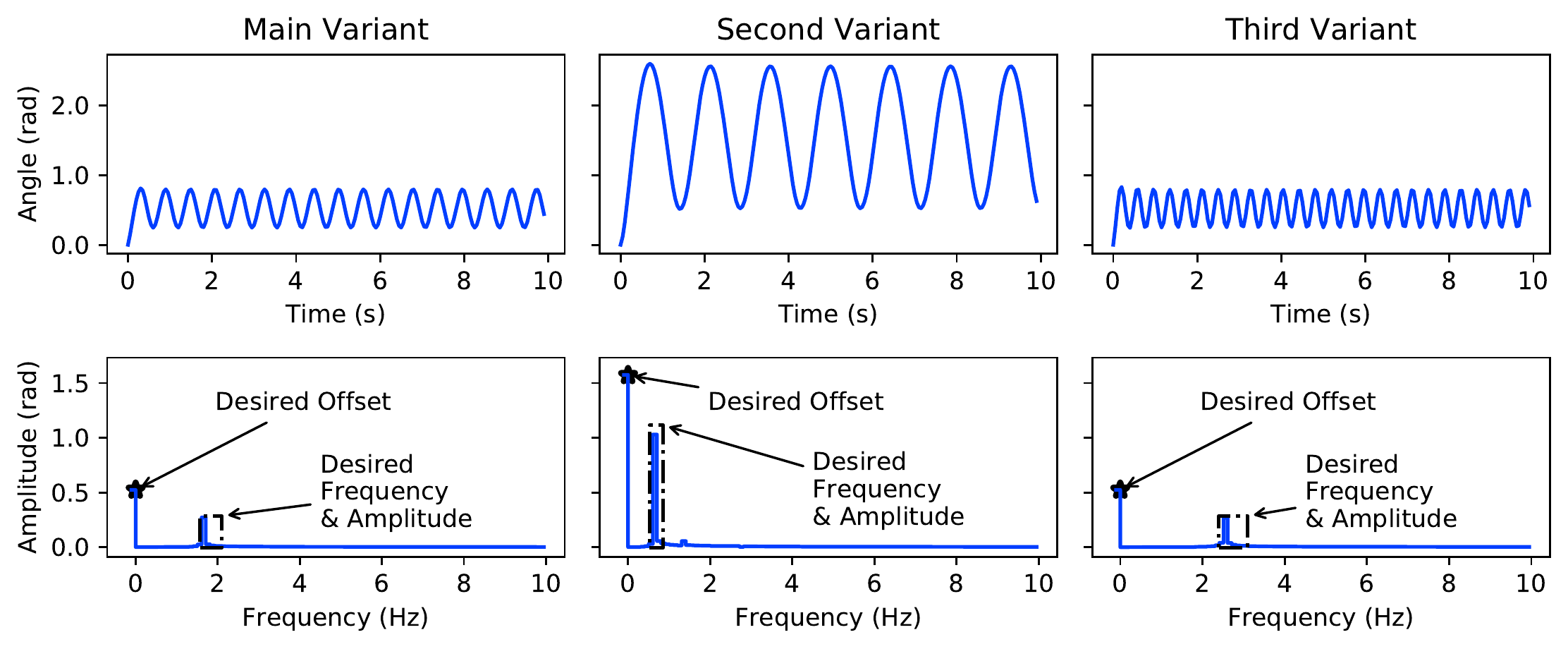}
			\label{fig:nlrtraincurvesv1}
		}
	\end{minipage}%
	\\
	\begin{minipage}[c]{1\textwidth}
		\centering
		\subfloat[]{
			\includegraphics[width=1\linewidth]{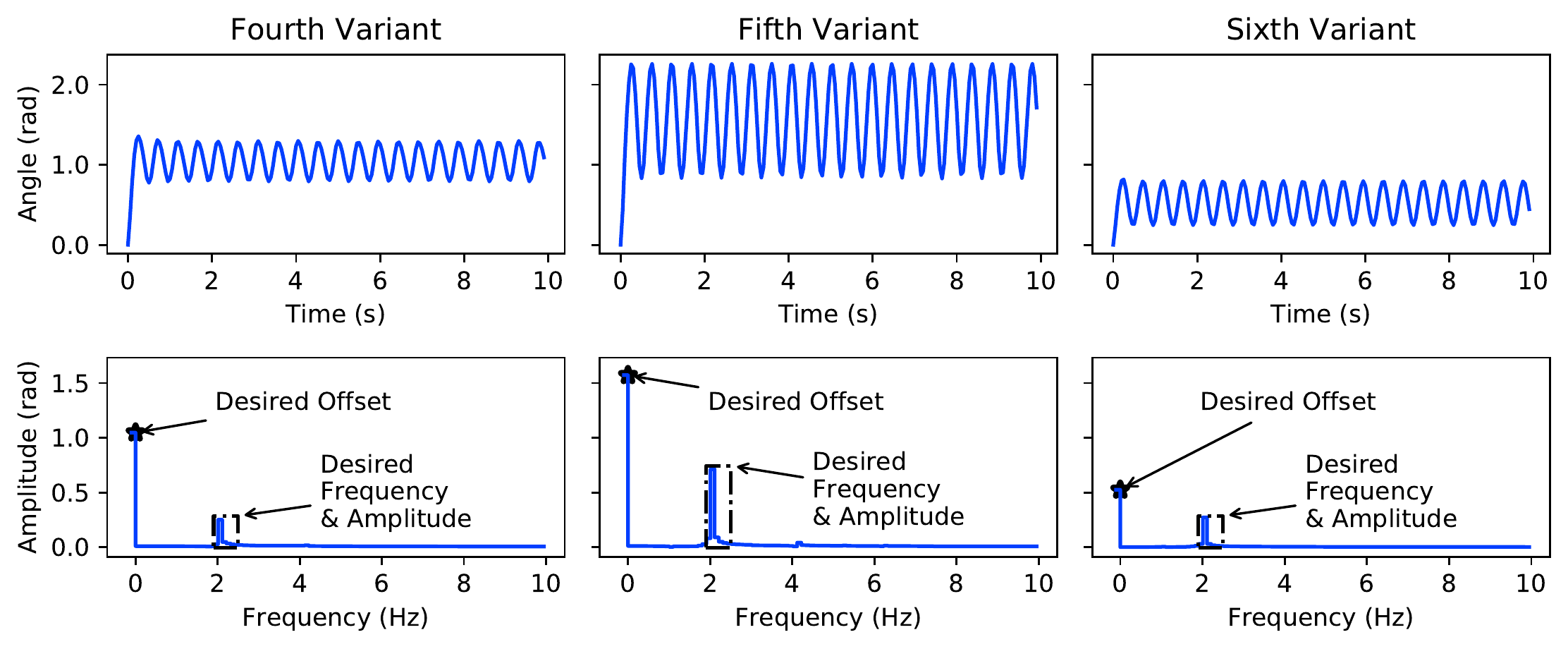}
			\label{fig:nlrtrajsv1}
		}
	\end{minipage}%
	\\
	\centering
	\begin{minipage}[c]{1\textwidth}
		\centering
		\subfloat[]{
			\includegraphics[width=1\linewidth]{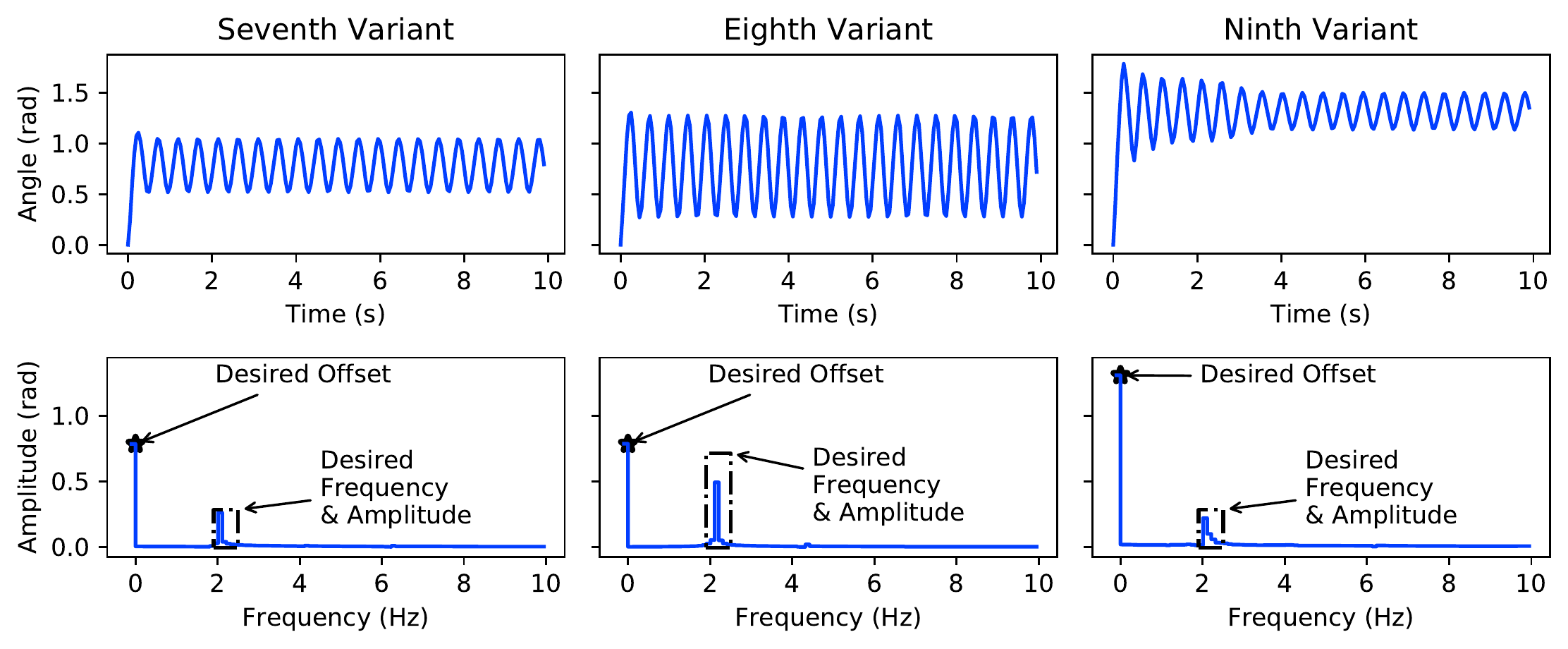}
			\label{fig:nlrtrajsv2}
		}
	\end{minipage}%
	\caption{Time and frequency domain trajectories for our method (TDPO) on multiple variants of the simple pendulum with non-local rewards. (a) The high-reward trajectories for the first group of variants, (b) the high-reward trajectories for the second group of variants, and (c) the high-reward trajectories for the third group of variants. Target values of oscillation frequency, amplitude, and offset were annotated in the frequency domain plots.}
	\label{fig:nlrtdpo}
\end{figure}

\insertnlrfig{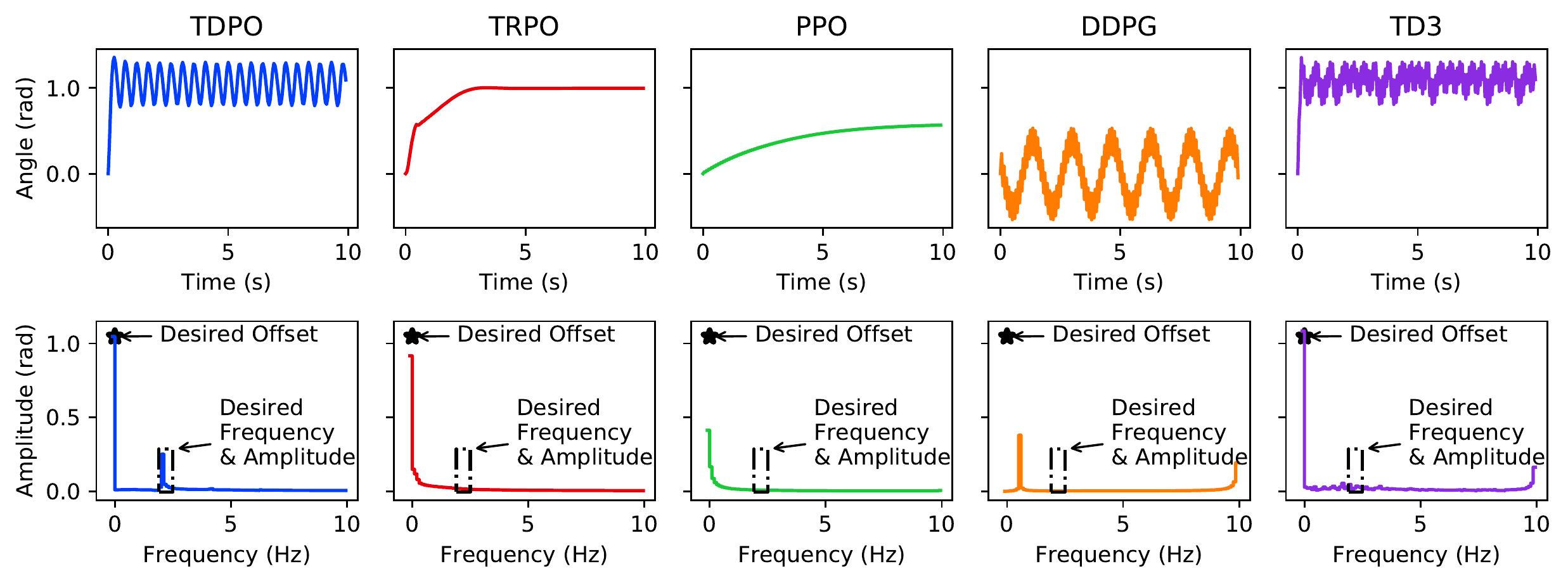}{seventh}
\insertnlrfig{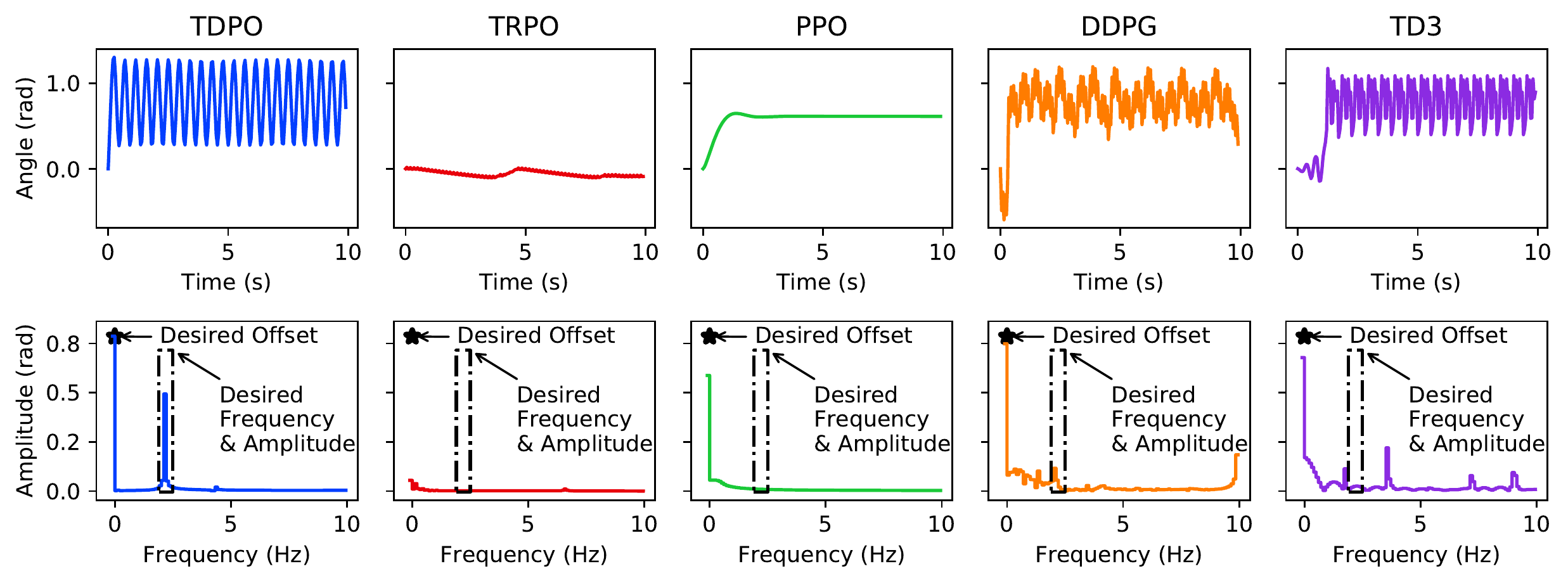}{eighth}
\insertnlrfig{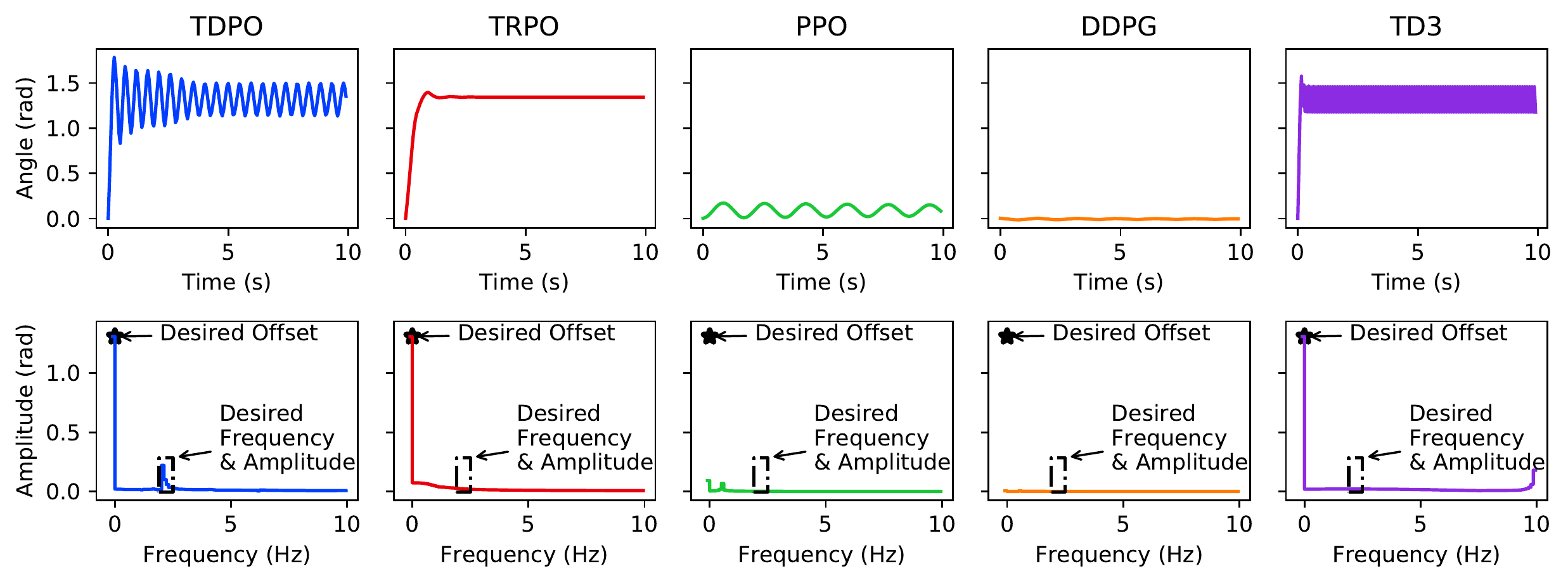}{ninth}

\subsection{Notes on How to Implement TDPO}
In short, our method (TDPO) is structured in the same way TRPO was structured; both TDPO and TRPO use policy gradient estimation, and a conjugate-gradient solver utilizing a Hessian-vector product machinery. On the other hand, there are some algorithmic differences that distinguish the basic variant of TDPO from TRPO. First of all, TRPO uses line-search heuristics to adaptively find the update scale; no such heuristics are applied in the basic variant of TDPO. Second, TDPO uses the DeVine advantage estimator, which requires storing and reloading pseudo-random generator states. Finally, the Hessian-vector product machinery used in TDPO computes Wasserstein-vector products, which is slightly different from those used in TRPO. The hyper-parameter settings and notes on how to choose them were discussed in Sections~\ref{sec:impl-deta-envir},~\ref{sec:c1c2}, and~\ref{legexpdetails}. We will describe how to implement TDPO, and focus on the subtle differences between TDPO and TRPO next.

As for the state-reset capability, our algorithm does not require access to a reset function for arbitrary states. Instead, we only require to be able to start from the prior trajectory's initial state. Many environments, including the Gym environments, instantiate their own pseudo-random generators and only utilize that pseudo-random generator for all randomized operations. This facilitates a straightforward implementation of the DeVine oracle; in such environments, implementing an arbitrary state-reset functionality is unnecessary, and only reloading the pseudo-random generator to its configuration prior to the trajectory would suffice. In other words, the DeVine oracle can store the initial configuration of the pseudo-random generator before asking for a trajectory reset and then start sampling. Once the main trajectory is finished, the pseudo-random generator can be reloaded, thus producing the same initial state upon a reset request. Other time-step states can then be recovered by applying the same proceeding action sequence.

To optimize the quadratic surrogate, the conjugate gradient solver was used. Implementing the conjugate gradient algorithm is fairly straightforward, and is already included in many common automatic differentiation libraries. The conjugate gradient solver is perfect for situations where (1) the Hessian matrix is larger than can efficiently be stored in the memory, and (2) the Hessian matrix includes many nearly identical eigenvalues. Both of these conditions apply well for TDPO, as well as for TRPO. Instead of requiring the full Hessian matrix to be stored, the conjugate gradient solver only requires a Hessian-vector product machinery $v\rightarrow Hv$, which must be specifically implemented for TDPO. Our surrogate function can be viewed as $$\mathcal{L}(\dth) = g^T \dth + \frac{C'_2}{2}\dth^T H \dth$$
where the Hessian matrix can be defined as
\begin{align}
	H = H_2 + \frac{C_1'}{C_2'} H_1, \qquad H_1 := \nabla^2_{\theta'} \mathbb{E}_{s\sim \rho^{\pi_k}_\mu} \bigg[ \mathcal{L}_{G^2}(\pi', \pi_k; s) \bigg], \nonumber
\end{align}
\begin{align}
	H_2 := \nabla^2_{\theta'} \mathbb{E}_{s\sim \rho^{\pi_k}_\mu} \bigg[ W(\pi'(a|s), \pi_k(a|s))^2 \bigg].
\end{align}


\begin{algorithm}[t]
	\caption{Wasserstein-Vector-Product Machinery}
	\begin{algorithmic}[1]
		\label{alg:wvp}
		\REQUIRE Current Policy $\pi_1$ with parameters $\theta_1$.
		\REQUIRE The vector $v$ with the same dimensions as  $\theta_1$.
		\REQUIRE An observation $s$.
		
		\STATE Compute the action for the observation $s$ with $|A|$ elements.
		\begin{equation}
		a_{|A|\times 1}:=\begin{bmatrix}
		\pi^{(1)}(s)\\
		\vdots \\
		\pi^{(|A|)}(s)\\
		\end{bmatrix} .
		\end{equation}
		This vector should be capable of propagating gradients back to the policy parameters when used in automatic differentiation software.
		\STATE Define $t$ to be a constant vector with the same shape as $a$. It could be populated with any values such as all ones.
		\STATE Define the scalar $\tilde{a} := a^T t$.
		\STATE Using back-propagation, find the gradient 
		\begin{equation}
		\nabla_\theta \tilde{a} = \sum_{i=1}^{|A|} t_{i} \nabla_\theta a_{i} = \sum_{i=1}^{|A|} t_{i} \begin{bmatrix} \frac{\partial a_i}{\partial \theta_1} &\cdots& \frac{\partial a_i}{\partial \theta_{|\Theta|}} \end{bmatrix} .
		\end{equation}
		\STATE Compute the following dot-product:
		\begin{equation}
		\ip{\nabla_\theta \tilde{a}, v} = (\sum_{i=1}^{|A|} t_{i} \cdot \frac{\partial a_{i}}{\partial \theta_1})\cdot v_1 + \cdots + (\sum_{i=1}^{|A|} t_{i} \cdot \frac{\partial a_{i}}{\partial \theta_{|\Theta|}})\cdot v_{|\Theta|}.
		\end{equation}
		\STATE Using automatic differentiation, take the gradient w.r.t. the $t$ vector.
		\begin{equation}
		\tilde{a}_{\theta,v} := \nabla_t \ip{\nabla_\theta \tilde{a}, v} = \begin{bmatrix}
		\frac{\partial a_{1}}{\partial \theta_1} \cdot v_1 + \cdots + \frac{\partial a_{1}}{\partial \theta_{|\Theta|}} \cdot v_{|\Theta|}\\
		\vdots \\
		\frac{\partial a_{|A|}}{\partial \theta_1} \cdot v_1 + \cdots + \frac{\partial a_{|A|}}{\partial \theta_{|\Theta|}} \cdot v_{|\Theta|}\\
		\end{bmatrix}
		= \begin{bmatrix}
		\frac{\partial a_{1}}{\partial \theta_1} & \cdots & \frac{\partial a_{1}}{\partial \theta_{|\Theta|}} \\
		\vdots & & \vdots\\
		\frac{\partial a_{|A|}}{\partial \theta_1} & \cdots & \frac{\partial a^{|A|}}{\partial \theta_{|\Theta|}} \\
		\end{bmatrix} v
		\end{equation}
		\STATE Compute the dot product $\ip{\tilde{a}_{\theta,v}, \tilde{a}}$.
		\STATE Using back-propagation, take the gradient w.r.t. $\theta$, and return it as the gain-vector-product.
		\begin{equation}
		\nabla_\theta \ip{\tilde{a}_{\theta,v}, \tilde{a}} = \begin{bmatrix}
		\frac{\partial a_{1}}{\partial \theta_1} & \cdots & \frac{\partial a_{1}}{\partial \theta_{|\Theta|}} \\
		\vdots & & \vdots\\
		\frac{\partial a_{|A|}}{\partial \theta_1} & \cdots & \frac{\partial a^{|A|}}{\partial \theta_{|\Theta|}} \\
		\end{bmatrix}^T \begin{bmatrix}
		\frac{\partial a_{1}}{\partial \theta_1} & \cdots & \frac{\partial a_{1}}{\partial \theta_{|\Theta|}} \\
		\vdots & & \vdots\\
		\frac{\partial a_{|A|}}{\partial \theta_1} & \cdots & \frac{\partial a^{|A|}}{\partial \theta_{|\Theta|}} \\
		\end{bmatrix} v
		\end{equation}

	\end{algorithmic}
\end{algorithm}



\begin{algorithm}
	\caption{Sensitivity-Vector-Product Machinery}
	\begin{algorithmic}[1]
		\label{alg:gvp2}
		\REQUIRE Current Policy $\pi_1$ with parameters $\theta_1$.
		\REQUIRE The vector $v$ with the same dimensions as  $\theta_1$.
		\REQUIRE An observation $s$.
		
		\STATE Compute the action to observation Jacobian matrix
		\begin{equation}
		J_{|A|\times |S|}:=\begin{bmatrix}
		\frac{\partial \pi^{(1)}(s)}{\partial s^{1}} & \cdots & \frac{\partial \pi^{(1)}(s)}{\partial s^{(|S|)}}\\
		\vdots & & \vdots \\
		\frac{\partial \pi^{(|A|)}(s)}{\partial s^{(1)}} & \cdots & \frac{\partial \pi^{(|A|)}(s)}{\partial s^{|S|}}\\
		\end{bmatrix} .
		\end{equation}
		This can either be done using finite-differences in the observation using
		\begin{equation}
		\frac{\partial \pi^{(i)}(s)}{\partial s^{(j)}} \simeq \frac{\pi^{(i)}(s+ds\cdot \mathbf{e_j}) - \pi^{(i)}(s)}{ds}
		\end{equation} (which may be a bit numerically inaccurate), or using automatic differentiation. In any case, this matrix should be a parameter tensor capable of propagating gradients back to the parameters when used in automatic differentiation software.
		
		\STATE Define $\tilde{J}$ to be the vectorized (i.e., reshaped into a column) $J$ matrix, with $|AS|=|A|\times|S|$ rows and one column.
		\STATE Define $t$ to be a constant vector with the same shape as $\tilde{J}$. It could be populated with any values such as all ones.
		\STATE Define the scalar $J_t := \tilde{J}^T t$.
		\STATE Using back-propagation, find the gradient 
		\begin{equation}
		\nabla_\theta J_t = \sum_{i=1}^{|A|} \sum_{j=1}^{|S|} t_{i,j} \nabla_\theta J_{i,j} = \sum_{i=1}^{|A|} \sum_{j=1}^{|S|} t_{i,j} \begin{bmatrix} \frac{\partial J_{i,j}}{\partial \theta_1} &\cdots& \frac{\partial J_{i,j}}{\partial \theta_{|\Theta|}} \end{bmatrix} .
		\end{equation}
		\STATE Compute the following dot-product.
		\begin{equation}
		\ip{\nabla_\theta J_t, v} = (\sum_{i=1}^{|A|} \sum_{j=1}^{|S|} t_{i,j} \cdot \frac{\partial J_{i,j}}{\partial \theta_1})\times v_1 + \cdots + (\sum_{i=1}^{|A|} \sum_{j=1}^{|S|} t_{i,j} \cdot \frac{\partial J_{i,j}}{\partial \theta_{|\Theta|}})\times v_{|\Theta|}
		\end{equation}
		\STATE Using automatic differentiation, take the gradient w.r.t. the $t$ vector.
		\begin{equation*}
		(\nabla_\theta J) v := \nabla_t \ip{\nabla_\theta J_t, v} = \begin{bmatrix}
		\frac{\partial J_{1,1}}{\partial \theta_1} \cdot v_1 + \cdots + \frac{\partial J_{1,1}}{\partial \theta_{|\Theta|}} \cdot v_{|\Theta|}\\
		\vdots \\
		\frac{\partial J_{|A|,|S|}}{\partial \theta_1} \cdot v_1 + \cdots + \frac{\partial J_{|A|,|S|}}{\partial \theta_{|\Theta|}} \cdot v_{|\Theta|}\\
		\end{bmatrix}
		\end{equation*}
		\begin{equation}
		= \begin{bmatrix}
		\frac{\partial J^{(1,1)}}{\partial \theta_1} & \cdots & \frac{\partial J^{(1,1)}}{\partial \theta_{|\Theta|}} \\
		\vdots & & \vdots\\
		\frac{\partial J^{(|A|,|S|)}}{\partial \theta_1} & \cdots & \frac{\partial J^{(|A|,|S|)}}{\partial \theta_{|\Theta|}} \\
		\end{bmatrix} v
		\end{equation}
		
		\STATE Reshape $(\nabla_\theta J) v$ into a column vector and name it $\tilde{J}_{\theta,v}$.
		\STATE Compute the dot product $\ip{\tilde{J}_{\theta,v}, \tilde{J}}$.
		\STATE Using back-propagation, take the gradient w.r.t. $\theta$, and return it as the gain-vector-product.
		\begin{equation}
		\nabla_\theta \ip{\tilde{J}_{\theta,v}, \tilde{J}} = \begin{bmatrix}
		\frac{\partial J^{(1,1)}}{\partial \theta_1} & \cdots & \frac{\partial J^{(1,1)}}{\partial \theta_{|\Theta|}} \\
		\vdots & & \vdots\\
		\frac{\partial J^{(|A|,|S|)}}{\partial \theta_1} & \cdots & \frac{\partial J^{(|A|,|S|)}}{\partial \theta_{|\Theta|}} \\
		\end{bmatrix}^T \begin{bmatrix}
		\frac{\partial J^{(1,1)}}{\partial \theta_1} & \cdots & \frac{\partial J^{(1,1)}}{\partial \theta_{|\Theta|}} \\
		\vdots & & \vdots\\
		\frac{\partial J^{(|A|,|S|)}}{\partial \theta_1} & \cdots & \frac{\partial J^{(|A|,|S|)}}{\partial \theta_{|\Theta|}} \\
		\end{bmatrix} v
		\end{equation}
		
	\end{algorithmic}
\end{algorithm}


In order to construct a Hessian-vector product machinery $v\rightarrow Hv$, one can design an automatic-differentiation procedure that returns the Hessian-vector product. Many automatic-differentiation packages already include functionalities that can provide a Hessian-vector product machinery of a given scalar loss function without computing the Hessian matrix. This can be used to implement the Hessian-vector product machinery in a straightforward manner; one only needs to provide the scalar quadratic terms of our surrogate and would obtain the Hessian-vector product machinery in return. On the other hand, this may not be the most computationally efficient approach, as our problem exhibits a more specific structure. Alternatively, one can implement a more elaborate and specifically designed Hessian-vector product machinery by following these three steps:

\begin{itemize}
	\item Compute the Wasserstein-vector product $v\rightarrow H_2v$ according to Algorithm~\ref{alg:wvp}.
	\item Compute the Sensitivity-vector product $v\rightarrow H_1v$ according to Algorithm~\ref{alg:gvp2}.
	\item Return the weighted sum of $H_1v$ and $H_2v$ as the final Hessian-vector product $Hv$.
\end{itemize}

One may also need to add a conjugate gradient damping to the conjugate gradient solver (i.e., return $\beta v + Hv$ for some small $\beta$ as opposed to returning $Hv$ purely), which is also done in the TRPO method. This may be important when the number of policy parameters is much larger than the sample size. Setting $\beta=0$ may yield poor numerical stability if $H$ had small eigenvalues, and setting large $\beta$ will cause the conjugate gradient optimizer to mimic the gradient descent optimizer by making updates in the same direction as the gradient. The optimal conjugate gradient damping may depend on the problem and other hyper-parameters such as the sample size. However, it can easily be picked to be a small value that ensures numerical stability.

Once the conjugate gradient solver returned the optimal update direction $H^{-1}g$, it must be scaled down by a factor of $C_2'$ (i.e., $\dth^* = H^{-1}g/C_2'$). If $\dth^*$ satisfied the trust region criterion (i.e., $\frac{1}{2} \dth^{*^T} H \dth^* \leq \delta_{\max}^2$), then one can make the parameter update (i.e., $\theta_{\text{new}} = \theta_{\text{old}} + \dth^*$) and proceed to the next iteration. Otherwise, the proposed update $\dth^*$ must be scaled down further, namely by $\alpha$, such that the trust region condition would be satisfied (i.e., $\frac{1}{2} (\alpha \delta\theta^*)^T H (\alpha \delta\theta^*) = \delta_{\max}^2$) before making the update $\theta_{\text{new}} = \theta_{\text{old}} + \alpha\dth^*$.

\subsection{Manual Choice of $C_1$ and $C_2$}\label{sec:c1c2}
Since the TDPO algorithm operates using the metric Wasserstein distance, thinking about how normalizing actions and rewards affect the corresponding optimization objective builds insight into how to set these coefficients properly. Say we use the same dynamics, only the new actions are scaled up by a factor of $\beta$, and the rewards are scaled up by a factor of $\alpha$:
\begin{equation}
a_{\text{new}} = \beta\cdot a_{\text{old}}\qquad r_{\text{new}}=\alpha\cdot r_{\text{old}}.
\end{equation}
If the policy function approximation class remained the same, the policy gradient would be scaled by a factor of $\frac{\alpha}{\beta}$ (i.e., $\frac{\partial \eta_{\text{new}}}{\partial a_{\text{new}}} = \frac{\alpha}{\beta} \cdot \frac{\partial \eta_{\text{old}}}{\partial a_{\text{old}}}$). Therefore, one can easily show that the corresponding new regularization coefficient and trust region sizes can be obtained by
\begin{equation}
C_{\text{new}}=\frac{\alpha}{\beta^2} \cdot C_{\text{old}}
\end{equation}
and
\begin{equation}
\delta_{\max}^{\text{new}} = \beta\cdot  \delta_{\max}^{\text{old}}.
\end{equation}

We used equal regularization coefficients (i.e., $C_1=C_2=C$), and the process to choose them can be summarized as follows: (1) Define $C=3600\cdot\alpha\cdot\beta^{-2}$, $\delta_{\max} = \beta/600$ and $\sigma_q=\beta/60$ (where $\sigma_q$ is the action disturbance parameter used for DeVine), (2) using prior knowledge or by trial and error determine appropriate action and reward normalization coefficients. The reward normalization coefficient $\alpha$ was sought to be approximately the average per-step discounted reward difference between a null policy and an optimal policy. We used a reward scaling value of $\alpha=5$ and an action scaling value of $\beta=5$ for the non-locally rewarded pendulum and $\beta=1.5$ for the long-horizon legged robot. Both environments had a per-step discounted reward of approximately $-5$ for a null policy and non-positive rewards, justifying the choice of $\alpha$.

\subsection{Implementation Details for the Practical Training of the Robotic Leg}

\begin{figure}[t]
		\includegraphics[width=0.19\linewidth]{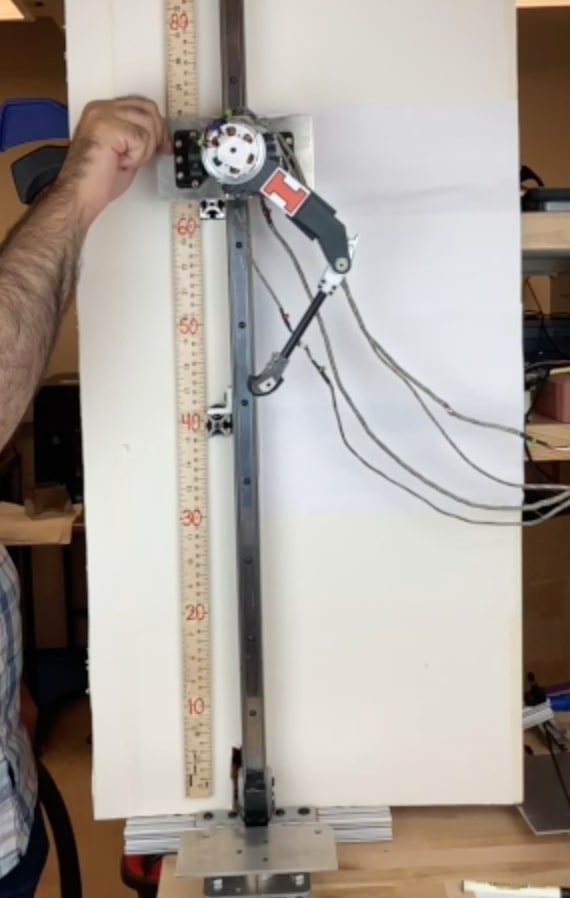}
		\includegraphics[width=0.19\linewidth]{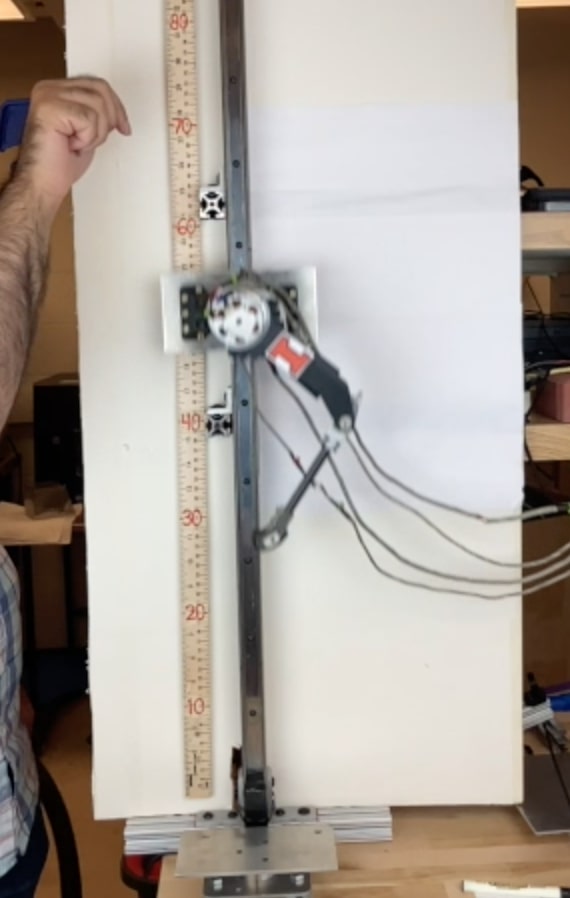}
		\includegraphics[width=0.19\linewidth]{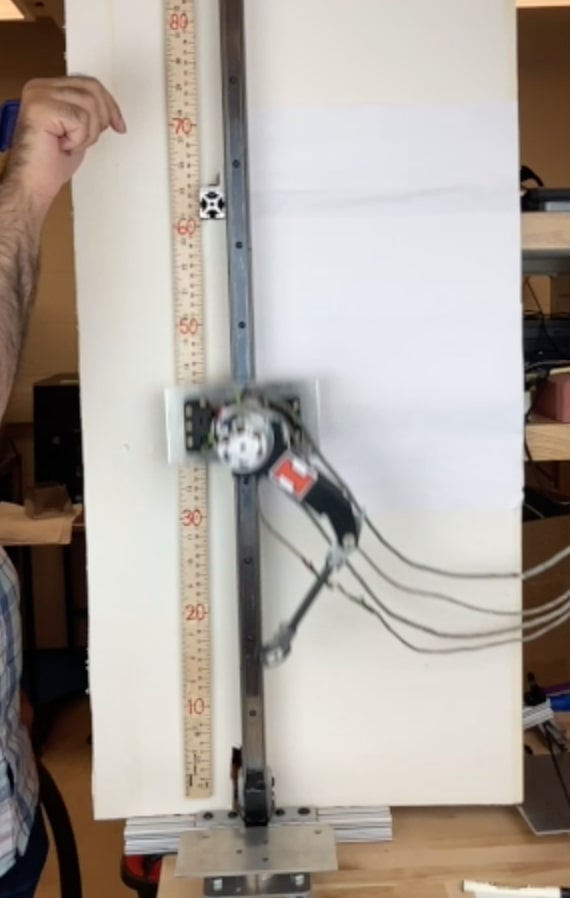}
		\includegraphics[width=0.19\linewidth]{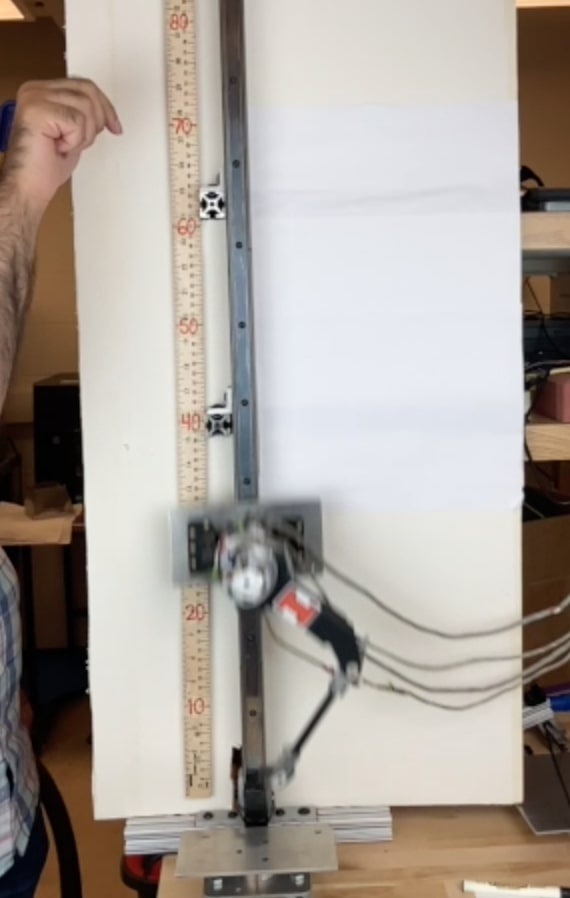}
		\includegraphics[width=0.19\linewidth]{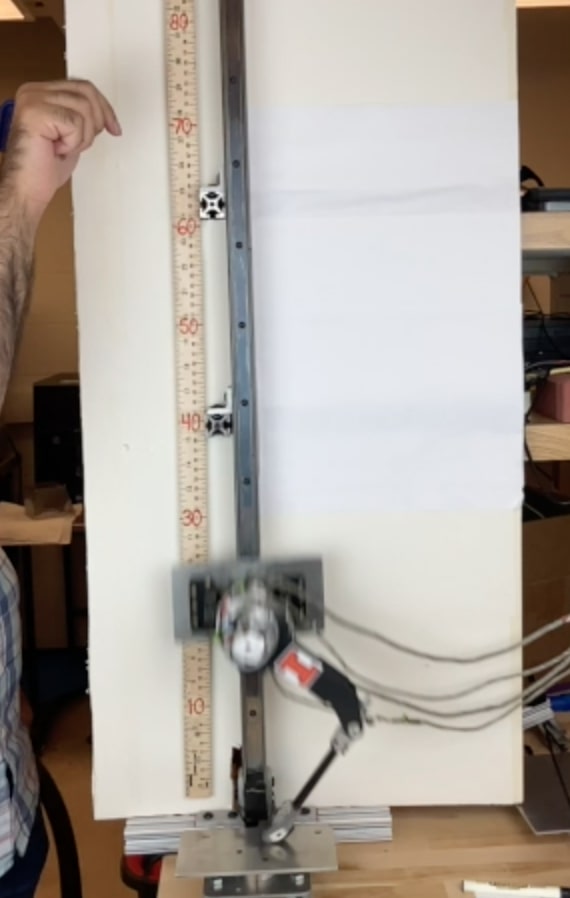}\\
		\includegraphics[width=0.19\linewidth]{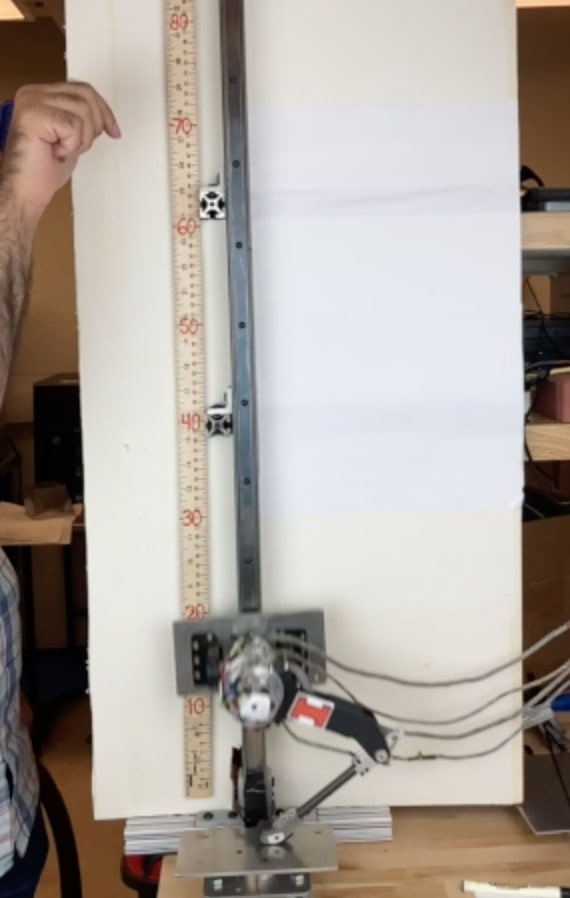}
		\includegraphics[width=0.19\linewidth]{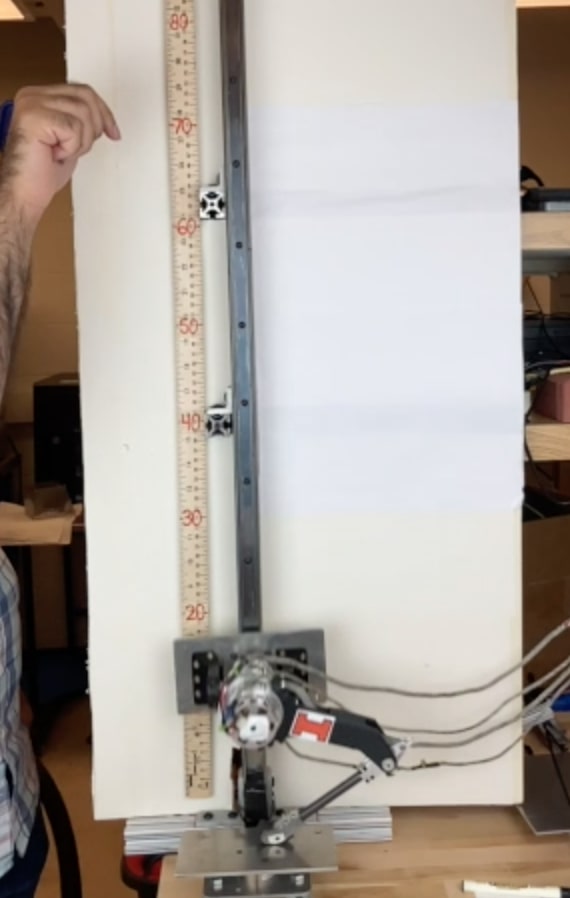}
		\includegraphics[width=0.19\linewidth]{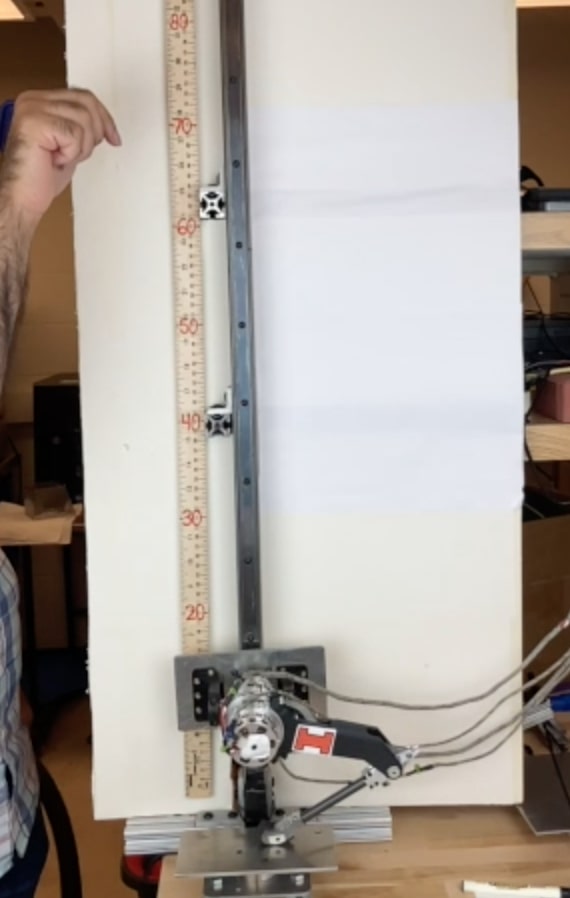}
		\includegraphics[width=0.19\linewidth]{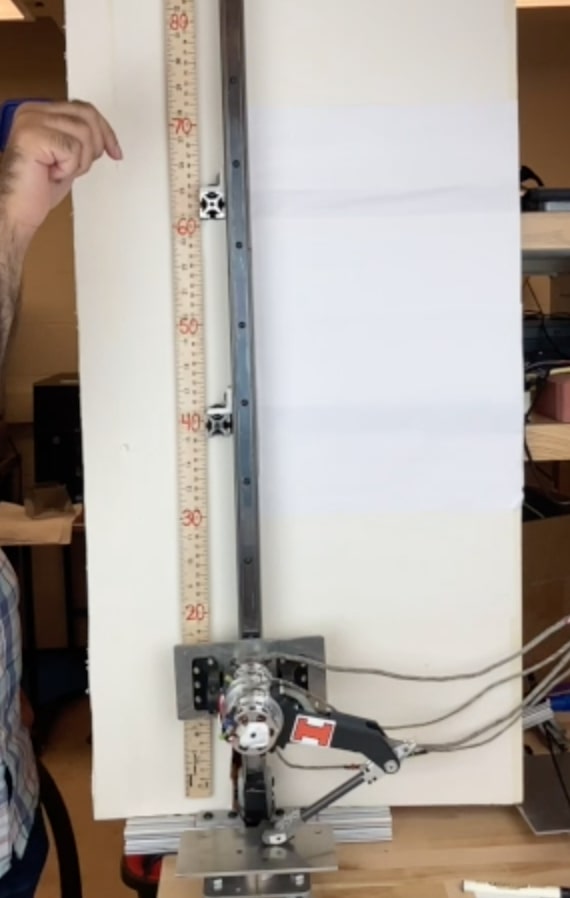}
		\includegraphics[width=0.19\linewidth]{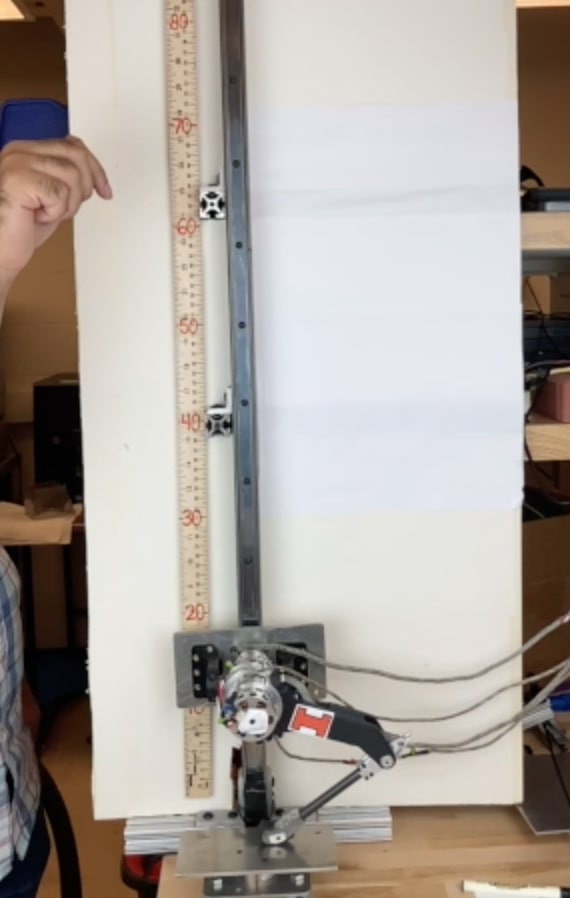}
		\caption{The frame sequence of a physical drop-and-catch test from a height of 0.7 m using the best agent trained by our method (TDPO) at 4 kHz control frequency. A short video of multiple drops from different heights is also included in our code repository. This unique simulation-to-real transfer is evidence of the practicality of our method, and the high-fidelity simulated dynamics in our environments. It is worth noting that TRPO, PPO, and TD3 could not produce any agents safe enough for physical tests neither at 4 kHz nor at 100 Hz control frequencies; the long-horizon agents could not land safely in simulation, and the best short-horizon agents failed at controlling the resonant oscillations due to the slow control frequency leading to physical damage to the hardware. On the other hand, TDPO's agent was capable of landing smoothly with no over- or under-shoots, resonant oscillations, or physical damage to the hardware.}
		\label{fig:camseq}
	\end{figure}


\begin{figure}[t]
	\begin{minipage}[c]{0.98\textwidth}
		\centering
		\subfloat[HPO curves for PPO and TRPO]{
			\includegraphics[width=1.00\linewidth]{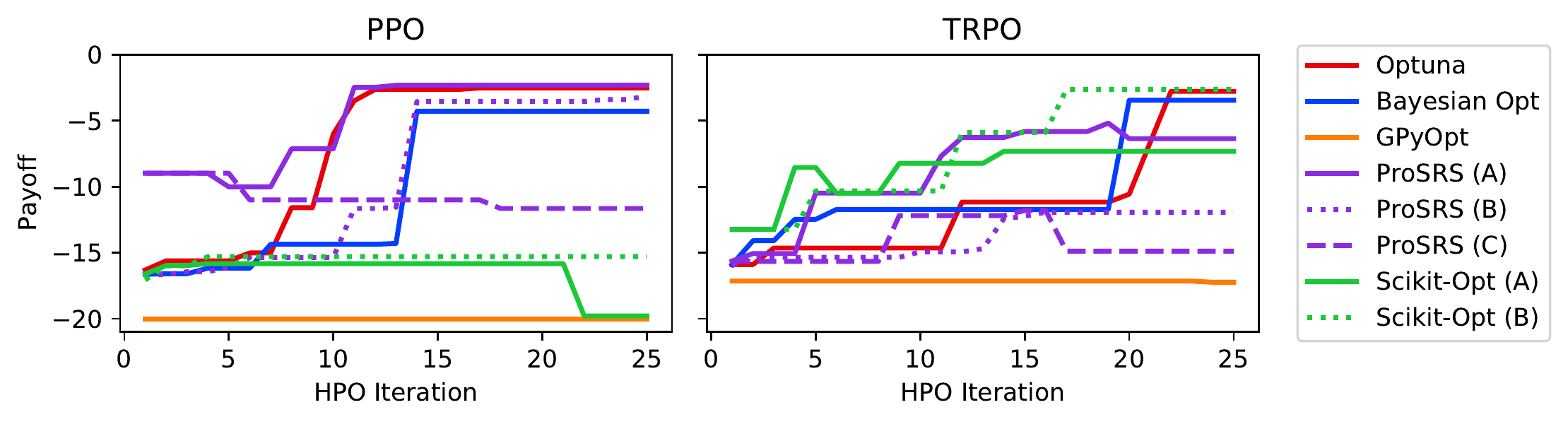}
		\label{fig:hpocurvesupp}
		}
	\end{minipage}
	\\
	\begin{minipage}[c]{0.48\textwidth}
		\centering
		\subfloat[PPO performance after 12 HPO iterations]{
			\includegraphics[width=1\linewidth]{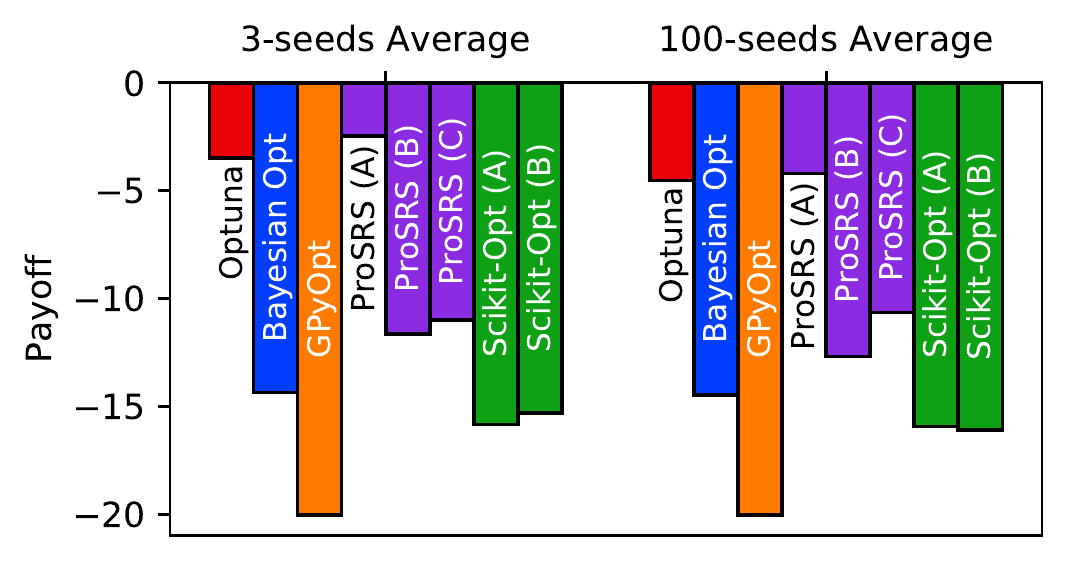}
		}
	\end{minipage}%
	\quad
	\begin{minipage}[c]{0.48\textwidth}
		\centering
		\subfloat[PPO performance after 25 HPO iterations]{
			\includegraphics[width=1\linewidth]{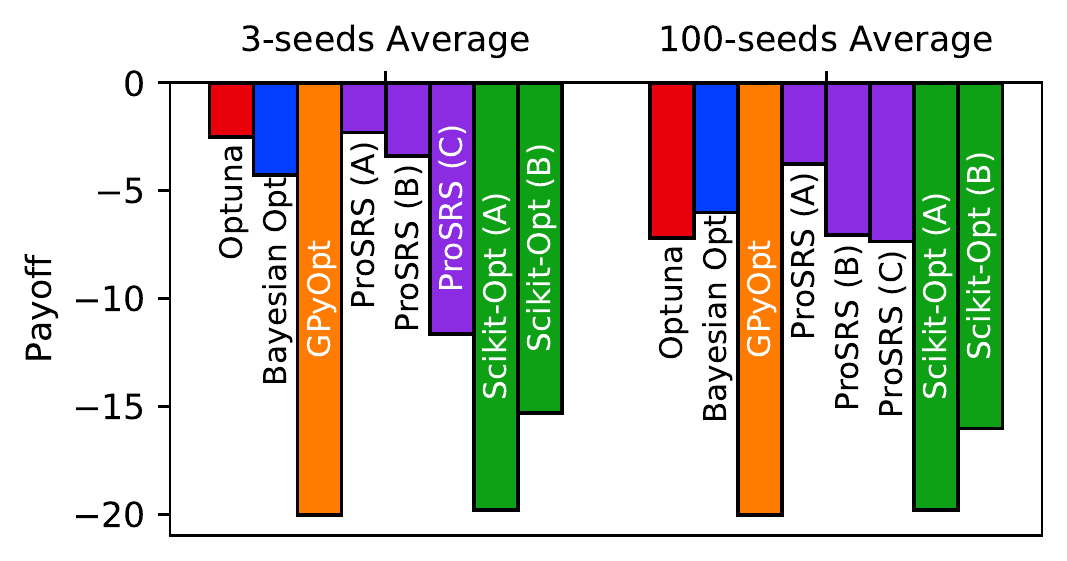}
		}
	\end{minipage}
	\\
	\begin{minipage}[c]{0.48\textwidth}
		\centering
		\subfloat[TRPO performance after 12 HPO iterations]{
			\includegraphics[width=1\linewidth]{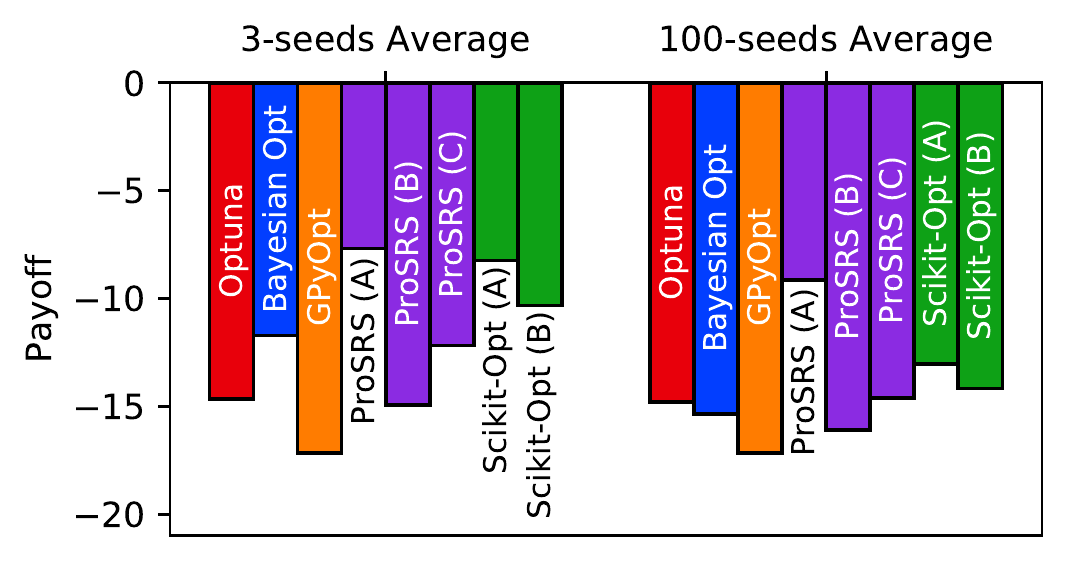}
		}
	\end{minipage}%
	\quad
	\begin{minipage}[c]{0.48\textwidth}
		\centering
		\subfloat[TRPO performance after 25 HPO iterations]{
			\includegraphics[width=1\linewidth]{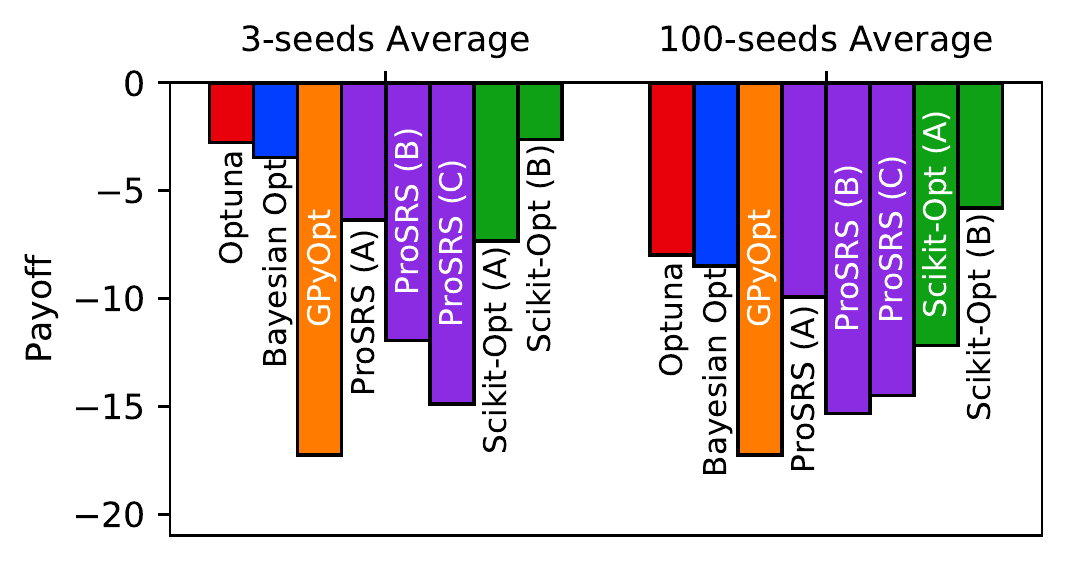}
		}
	\end{minipage}%
	
	\caption{The top line-plots show the HPO payoff curves for PPO and TRPO. For each iteration, the best agent obtained so far was evaluated using the 3 HPO seeds. In the top-left bar plot, the best agent after 12 iterations of HPO was selected and evaluated twice for each method; once only with the 3 HPO seeds, and once with 100 different random seeds. The shorter the bars, the better the agent's performance. Three ProSRS and two Scikit-Opt variants were considered. The A variants made 30 parallel proposals per iteration, where each proposal was evaluated with a single seed randomly chosen from the 3 HPO seeds. The B Variants made 10 parallel proposals per iteration, where each proposal was evaluated with all 3 seeds. Variant C made 30 parallel proposals per iteration, where each proposal was evaluated with a single seed randomly chosen between 1 and 100. Similarly, Sub-figures the top-right, bottom-left, and bottom-right bar plots were drawn for PPO after 25 iterations, TRPO after 12 iterations, and TRPO after 25 iterations, respectively. Overall, Variant A offers the best balance between the number of proposals and the induced stochasticity. The bar plots suggest a high ordering correlation between the evaluations on 3 seeds and 100 seeds, which further validates the "optimize the hyper-parameters for 3-seeds and finally evaluate on many seeds" approach.}
	\label{fig:supphpolf}
	\vspace{-6mm}
\end{figure}

Section~\ref{sec:ovatdetails} details the one variable at a time parameter sweep experiments which served as the initial guess set for HPO and defined the hyper-parameter search domain. Section~\ref{sec:indtraininglegsto} details the individual training settings for each set of proposed hyper-parameters. Section~\ref{sec:hpodetailssupp} discusses the HPO and the log-space search details. Finally, Section~\ref{sec:stolegexpdetails} describes the specifications of the transition dynamics, the reward definition, and the imposed stochasticity in the observations and the initial state distribution. We used the advanced variant of our method (TDPO) in this experiment, and did not perform any HPO on it due to the computational resource limitations and since the initial choice of the hyper-parameters was already outperforming the other methods. Although TD3 was included as a candidate method for HPO on the final long-horizon environment, the one variable at a time parameter sweep experiments revealed extremely poor performance on its behalf, and performing a few HPO iterations did not improve its performance significantly. Considering our computational resource limitations, we excluded TD3 from further HPO analysis due to its poor initial performance.

\subsubsection{One Variable at a Time Parameter Sweeps}\label{sec:ovatdetails}
For better performance, HPO methods need a reasonable set of initial hyper-parameter guesses. For this, we perform a one-variable-at-a-time parameter sweep around the central (default) setting of the reinforcement learning method. The central HP values, the search domain, and the sweep values are shown in Tables~\ref{tab:ppo1shorthorovat},~\ref{tab:trposhorthorovat}, and ~\ref{tab:td3shorthorovat} for the PPO, TRPO, and TD3 methods, respectively. Each HP was swept with the candidate values while keeping all the other HPs fixed at their central value. The benchmarking problem for determining the best HPO method had a short horizon, so the default HPs of each method were used as the central value. The only exception for this was the entropy coefficient, which was set to zero due to the default value causing significant instability and performance deterioration in the central setting for both TRPO and PPO. For TD3, we implemented two common exploration noise types: the Ornstein-Uhlenbeck noise and the pink noise. Each of these noises was parametrized by their relative bandwidth, where the relative bandwidth lies within $[0,1]$. This relative bandwidth was treated as a hyper-parameter for TD3. Since the Ornstein-Uhlenbeck noise with a relative bandwidth of 1 is the same as the white Gaussian noise (TD3's default exploration noise), we set it as the default. All settings were repeated with 3 random seeds, and the best agent's return during each individual training was reported as the performance metric for HPO. The parameter sweep values were extended on both ends until a clear peek in performance was detected. These parameter sweeps created an initial set of HPs with their corresponding performances, which were initially input to all HPO methods.

Since there was a 40-fold increase in the number of time-steps per trajectory between the short- and long-horizon environments, relevant HPs were proportionally scaled in the long-horizon. In particular, the central value, domain, and parameter sweep values for all (a) batch-sizes, (b) initial pre-training samples, and (c) training intervals in samples were multiplied by 40. Due to computational resource limitations, TD3's default optimization batch-size was not scaled proportional to the horizon, since such scaling would make the optimization costs alone more than 100 times the sampling costs. We still used the same number of parallel workers for TD3, and the HPO methods were still allowed to propose optimization batch-sizes tens of times larger than the default value. We also adjusted the MDP and GAE discount factors so that their respective horizon lengths are multiplied by 40.

\begin{table}[t]
  \centering
    \begin{tabular}{|P{2.1cm}|P{0.9cm}|P{1.8cm}|P{7.4cm}|}
    \toprule
    Hyper-Param. & Center & Domain & Parameter Sweep Values \\
    \midrule
    MDP Discount & \multicolumn{1}{c|}{$0.99$} & $[0.36,$ $\phantom{[}0.99984375]$ & {$0.36, 0.68, 0.84, 0.92, 0.96, 0.98, 0.99, 0.995, 0.9975,$ $0.99875, 0.999375, 0.9996875, 0.99984375$} \\
    \midrule
    GAE Discount   & \multicolumn{1}{c|}{$0.98$} & $[0.36,$ $\phantom{[}0.99984375]$ & {$0.36, 0.68, 0.84, 0.92, 0.96, 0.98, 0.99, 0.995, 0.9975,$ $0.99875, 0.999375, 0.9996875, 0.99984375$} \\
    \midrule
    Sampling BS & \multicolumn{1}{c|}{$256$} & $[64, 16384]$ & {$64, 128, 256, 512, 1024, 2048, 4096, 8192, 16384$} \\
    \midrule
    Clip Param. & \multicolumn{1}{c|}{$0.2$} & $[0.02, 200]$ & {$0.02, 0.06, 0.2, 0.6, 2.0, 6, 20, 60, 200$} \\
    \midrule
    Entropy Coef. & \multicolumn{1}{c|}{$0$} & $[0, 0.1]$ & {$0, 0.001, 0.01, 0.1$} \\
    \midrule
    Opt. Epochs & \multicolumn{1}{c|}{$4$} & $[1, 128]$ & {$1, 2, 4, 8, 16, 32, 64, 128$} \\
    \midrule
    Opt. MBs & \multicolumn{1}{c|}{$4$} & $[1, 64]$ & {$1, 2, 4, 8, 16, 32, 64$} \\
    \midrule
    Opt. LR & \multicolumn{1}{c|}{$10^{-3}$} & $[10^{-5}, 10^{-1}]$ & {$10^{-5}, 3\times10^{-5}, 10^{-4}, 3\times10^{-4}, 10^{-3},$ $3\times10^{-3}, 10^{-2}, 3\times10^{-2}, 10^{-1}$} \\
    \midrule
    ADAM $\epsilon$ & \multicolumn{1}{c|}{$10^{-5}$} & $[10^{-8}, 10^{-4}]$ & {$10^{-8}, 3\times10^{-8}, 10^{-7},  3\times10^{-7}, 10^{-6},$ $3\times10^{-6}, 10^{-5}, 3\times10^{-5}, 10^{-4}$} \\
    \midrule
    LR Schedule & Linear & -     & {Constant, Linear} \\
    \bottomrule
    \end{tabular}%
  \caption{The one-variable-at-a-time parameter sweep details for the PPO method on the short-horizon leg benchmark environment. BS, MB, and LR are short for batch-size, mini-batch, and learning rate, respectively. See Section~\ref{sec:ovatdetails} for more information.}
  \label{tab:ppo1shorthorovat}%
  \vspace{-6mm}
\end{table}%

\begin{table}[t]
  \centering
    \begin{tabular}{|M{2.1cm}|P{0.9cm}|P{1.8cm}|P{7.4cm}|}
    \toprule
    Hyper-Param. & Center & Domain & Parameter Sweep Values \\
    \midrule
    Sampling BS & $1024$  & $[64, 16384]$ & {$64, 128, 256, 512, 1024, 2048, 4096, 8192, 16384$} \\
    \midrule
    MDP Discount & $0.99$  & $[0.36,$ $0.99984375]$ & {$0.36, 0.68, 0.84, 0.92, 0.96, 0.98, 0.99, 0.995,$ $0.9975, 0.99875, 0.999375, 0.9996875, 0.99984375$} \\
    \midrule
    GAE Discount   & $0.98$  & $[0.36,$ $0.99984375]$ & {$0.36, 0.68, 0.84, 0.92, 0.96, 0.98, 0.99, 0.995,$ $0.9975, 0.99875, 0.999375, 0.9996875, 0.99984375$} \\
    \midrule
    Max KL & $0.01$  & $[0.00125,$ $0.64]$ & {$0.00125, 0.0025, 0.005, 0.01, 0.02,$ $0.04, 0.08, 0.16, 0.32, 0.64$} \\
    \midrule
    CG Iterations & $10$    & $[1, 20]$ & {$1, 2, 5, 10, 20$} \\
    \midrule
    Entropy Coef. & $0$     & $[0.0, 10^{-3}]$ & {$0.0, 10^{-5}, 10^{-4}, 10^{-3}$} \\
    \midrule
    CG Damping & $0.01$  & $[10^{-4}, 1]$ & {$10^{-4}, 10^{-3}, 10^{-2}, 10^{-1}, 1$} \\
    \midrule
    VF LR & $0.0003$ & $[3\times 10^{-6},$ $3\times 10^{-2}]$ & {$3\times 10^{-6}, 10^{-5}, 3\times 10^{-5}, 10^{-4}, 3\times 10^{-4},$ $10^{-3}, 3\times 10^{-3}, 10^{-2}, 3\times 10^{-2},$} \\
    \midrule
    VF Iterations & $3$     & $[1, 24]$ & {$1, 3, 6, 12, 24$} \\
    \midrule
    VF MBs & $8$     & $[1, 64]$ & {$1, 2, 4, 8, 16, 32, 64$} \\
    \bottomrule
    \end{tabular}%
    \caption{The one variable at a time parameter sweep details for TRPO on the short-horizon leg benchmark environment. BS, MB, LR, VF, and CG are short for batch-size, mini-batch, learning-rate, value function, and conjugate gradient, respectively. See Section~\ref{sec:ovatdetails} for more information.}
  \label{tab:trposhorthorovat}%
  \vspace{-6.5mm}
\end{table}%

\begin{table}[t]
  \centering
    \begin{tabular}{|P{2.0cm}|P{1.1cm}|P{1.8cm}|P{7.3cm}|}
    \toprule
    Hyper-Param. & Center & Domain & Parameter Sweep Values \\
    \midrule
    MDP Discount & $0.99$ & $[0.36,$ $0.99984375]$ & {$0.36, 0.68, 0.84, 0.92, 0.96, 0.98, 0.99, 0.995, 0.9975,$ $0.99875, 0.999375, 0.9996875, 0.99984375$} \\
    \midrule
    Buffer Size & $50000$ & $[3125,$ $800000]$ & {$3125, 6250, 12500, 25000, 50000, 100000,$ $200000, 400000, 800000$} \\
    \midrule
    Pre-training & $100$ & $[25, 6400]$ & {$25, 50, 100, 200, 400, 800, 1600, 3200, 6400$} \\
    \midrule
    Train Interval & $100$ & $[25, 6400]$ & {$25, 50, 100, 200, 400, 800, 1600, 3200, 6400$} \\
    \midrule
    Opt. BS & $128$ & $[8, 2048]$ & {$8, 16, 32, 64, 128, 256, 512, 1024, 2048$} \\
    \midrule
    Opt. LR & $0.0003$ & $[3\times 10^{-6},$ $3\times 10^{-2}]$ & {$3\times 10^{-6}, 10^{-5}, 3\times 10^{-5}, 10^{-4},$ $3\times 10^{-4}, 10^{-4}, 3\times 10^{-4}, 10^{-2}, 3\times 10^{-2}$} \\
    \midrule
    GD Iterations & $100$ & $[6, 400]$ & {$6, 12, 24, 50, 100, 200, 400$} \\
    \midrule
    Soft Update Coefficient & $0.005$ & $[0.000625,$ $0.08]$ & {$0.000625, 0.00125, 0.0025, 0.005,$ $0.01, 0.02, 0.04, 0.08$} \\
    \midrule
    Policy Delay & $2$ & $[1, 16]$ & {$1, 2, 4, 8, 16$} \\
    \midrule
    Noise Type & Ornstein & -     & {Ornstein, Pink} \\
    \midrule
    Noise RFB & $1$ & $[2^{-8}, 1]$ & {$2^{-8}, 2^{-7}, 2^{-6}, 2^{-5}, 2^{-4}, 2^{-3}, 2^{-2}, 2^{-1}, 1$} \\
    \midrule
    Noise std & $0.1$ & $[0.00625, 0.8]$ & {$0.00625, 0.0125, 0.025, 0.05, 0.1, 0.2, 0.4, 0.8$} \\
    \midrule
    Target Noise std & $0.2$ & $[0.00625, 3.2]$ & {$0.00625, 0.0125, 0.025, 0.05, 0.1,$ $0.2, 0.4, 0.8, 1.6, 3.2$} \\
    \midrule
    Target Noise Clipping & $0.5$ & $[0.0625, 4]$ & {$0.0625, 0.125, 0.25, 0.5, 1, 2, 4$} \\
    \bottomrule
    \end{tabular}%
    \caption{The one-variable-at-a-time parameter sweep details for the TD3 method on the short-horizon robotic leg test environment. BS, LR, GD, and RFB are short for batch-size, learning rate, gradient descent, and relative frequency bandwidth, respectively. See Section~\ref{sec:ovatdetails} for more information.}
  \label{tab:td3shorthorovat}%
  \vspace{-6mm}
\end{table}%

\begin{table}[htbp]
  \centering
    \begin{tabular}{|P{1.0cm}|P{2.05cm}|P{1.1cm}|P{1.65cm}|P{6.05cm}|}
    \toprule
    Method & Hyper-Param. & Center & Domain & Parameter Sweep Values \\
    \midrule
    PPO, TRPO, TD3 & MDP Discount & $0.99975$ & $[0.488,$ $0.99996875]$ & {$0.488, 0.744, 0.872, 0.936, 0.968, 0.984,$ $0.992, 0.996, 0.998, 0.999, 0.9995, 0.99975,$ $0.999875, 0.9999375, 0.99996875$} \\
    \midrule
    PPO, TRPO & GAE Discount & $0.9995$ & $[0.488,$ $0.9999375]$ & {$0.488, 0.744, 0.872, 0.936, 0.968, 0.984,$ $0.992, 0.996, 0.998, 0.999, 0.9995, 0.9995,$ $0.99975, 0.999875, 0.9999375$} \\
    \midrule
    PPO, TRPO & Sampling BS & $40000$ & $[19,$ $160000]$ & {$19, 39, 78, 156, 312, 625, 1250, 2500, 5000,$ $10000, 20000, 40000, 80000, 160000$} \\
    \midrule
    TD3   & Buffer Size & $2000000$ & $[62500,$ $8000000]$ & {$62500, 125000, 250000, 500000, 1000000,$ $2000000, 4000000, 8000000$} \\
    \midrule
    TD3   & Pre-training, Train Interval & $4000$ & $[500, 64000]$ & {$500, 1000, 2000, 4000, 8000,$ $16000, 32000, 64000$} \\
    \midrule
    TD3   & Opt. BS & $128$ & $[8, 2048]$ & {$8, 16, 32, 64, 128, 256, 512, 1024, 2048$} \\
    \bottomrule
    \end{tabular}%
  \caption{The one-variable-at-a-time parameter sweep details for the long-horizon robotic leg environment. Only the HPs in need of horizon-scaling were given here, and the rest of the HPs used the same settings as Tables~\ref{tab:ppo1shorthorovat},~\ref{tab:trposhorthorovat}, and ~\ref{tab:td3shorthorovat}. 144 parallel workers were used in the long-horizon environment trainings, so the collective batch-sizes are 144 times the values in this table.}
  \label{tab:hpo5bovat}%
  \vspace{-6mm}
\end{table}%

\subsubsection{Individual Training Details}\label{sec:indtraininglegsto}
Our experiments in the short-horizon test environment showed that the TD3-specific architecture (i.e., using ReLU activations and a tanh output activation and normalizing the actions) resulted in harming the performance of the TD3 method, and was worse than using the neural architecture of other methods. We speculate that this is due to the robotic leg being high-powered making initial attempts at exploring the full actuation amplitude futile and thus resulting in poor performance. Therefore, we used the same neural architecture (a 3-layer MLP with 64 units in the hidden layers and tanh activation) for all methods. As shown in Figure~\ref{fig:hpopicking}, using this architecture, TD3 managed to outperform the other methods upon full HPO on the short-horizon test benchmark.

\paragraph{The short-horizon test environment:} We used four parallel workers in all trainings. Therefore, all the relevant batch-sizes in Tables~\ref{tab:ppo1shorthorovat},~\ref{tab:trposhorthorovat}, and ~\ref{tab:td3shorthorovat} must be quadrupled to reveal the collective values. Each training was performed for one million time-steps per worker (i.e., four million collective training steps). Since (a) some HPO methods were intolerant of evaluation stochasticity and (b) performing HPO with many seeds was intractable, the HPO's metric was defined as the average performance on 3 pre-determined random seeds. This environment defined trajectories with a duration of 2 seconds and a control frequency of 100 Hz, corresponding to a total of 200 time-steps per trajectory.

\paragraph{The long-horizon leg environment:}  We used a large number of parallel workers in this problem to make the training as stable as possible. In particular, we used 144 parallel workers, which means that the collective batch-sizes can be obtained by multiplying the values in Table~\ref{tab:hpo5bovat} by 144. We trained each set of proposed hyper-parameters for 5 billion collective time-steps. Similar to the short-horizon environment, the HPO's metric was defined as the average performance on 3 pre-determined random seeds. However, the final set of best hyper-parameters for each method was trained with 25 random seeds to be shown in Figure~\ref{fig:legsto} of the main paper. This environment defined rollout durations of 2 s and a control frequency of 4 kHz, corresponding to a total of 8000 time-steps per trajectory.

\subsubsection{The HPO Details}\label{sec:hpodetailssupp}

We used the default settings with each implementation. These default settings were specifically as follows: Optuna used a Tree Parzen Estimation (TPE) method, Bayesian Optimization used a Gaussian Process (GP) with Upper Confidence Bound (UCB) acquisitions and the Mattern kernel, Scikit-Opt used a GP with a hedge acquisition function (i.e., automatically determining the acquisition function from a pre-defined set), ProSRS used a GP with Radial Basis Function (RBF) acquisitions, and GPyOpt used a GP with Local Penalization (LP), UCB acquisitions, and a white noise kernel. All HPO methods parameters were left as their default values.

ProSRS and Scikit-Opt HPO implementations could tolerate evaluation noise. That is, these implementations were programmed to handle stochasticity in their evaluation metric (according to their documentation). On the other hand, Optuna, Bayesian Optimization, and GPyOpt's documentation suggested running them on non-noisy evaluation functions only. Since each HPO training in the long-horizon environment could take 10 hours, running thousands of sequential HPO iterations is impractical. Therefore, we used the HPO implementations in a ``batched'' capacity, where each HPO method proposed multiple sets of HPs for parallel evaluation and then received their performance values simultaneously. We allowed all HPO methods to ask for 30 parallel trainings in each proposal.

The noise-tolerant HPO implementations (ProSRS and Scikit-Opt) were allowed to propose 30 different HPs. We ran each of these 30 with one of the 3 pre-determined random seeds (picked at random) and returned the result to the HPO method. We experimentally validated that this is a practical choice leading to the best results for the noise-tolerant HPO implementations, as shown in Figure~\ref{fig:supphpolf}. On the other hand, the noise-intolerant HPO implementations (Optuna, Bayesian Optimization, and GPyOpt) could only propose 10 HP sets, since each set needed to be trained on all 3 pre-determined random seeds, and their average would have to be reported to the HPO method. This ``batched'' HPO approach allowed us to effectively optimize TRPO, PPO, and TD3 on the short horizon benchmark in 25 HPO iterations as shown in Figure~\ref{fig:legsto}. Due to resource limitations, we only ran 11 HPO iterations for PPO and TRPO on the long-horizon environment. 

\paragraph{Hyper-parameter pre-processing transformations:} Performing HPO on the original HP space uniformly can make (a) the search domain narrow, and (B) the HP landscape difficult to navigate. The results suggested that PPO, TRPO, and TD3 needed the HPO to go beyond merely fine-tuning, so we performed the HP search in the log-space to cover as large a domain as possible and make the search landscape smoother. Therefore, the natural log of all numeric hyper-parameters was taken before passing them to the HPO method. We made two exceptions to this rule. Instead of searching for the MDP and GAE discount factors in the log-space, we transformed them into their respective horizons and then searched for the horizons logarithmically. In other words, instead of searching for $\log(\gamma)$, we searched for $\log(1/{(1-\gamma)})$. Since we wanted the HPO method to be able to set the entropy coefficient $C_{\text{entropy}}$ to zero, we searched for $\log(C_{\text{entropy}} + 10^{-5})$ instead of searching for $\log(C_{\text{entropy}})$. This allowed the HPO method to be able to disable the entropy regularization without having to stretch the search domain to negative infinity. The search domains for all HPs were identical to the one variable at a time parameter sweep domains in Section~\ref{sec:ovatdetails}, where the domain bounds were extended until a clear peek in performance was detected.

\subsubsection{The Environment Specification}\label{sec:stolegexpdetails}

For the dynamics, we used the same physical model as the one in Sections~\ref{sec:legsimple} and ~\ref{legexpdetails}. For the initial state distribution, the hip and knee angles were uniformly chosen from the $[-180^{\circ}, -30^{\circ}]$ and $[-155^{\circ}, -35^{\circ}]$ intervals, respectively. The initial drop height of the robotic leg was uniformly chosen between 0.4 and 0.8 meters. To make the physical model realistically stochastic, we extracted physical sensing noises from the hardware and used them as a template for generating stochastic observation noise in our simulated model. Since the agent was to be implemented on physical hardware, we modified the reward definition to penalize non-optimal behavior, such as violating the physical constraints, more severely. In particular, we used the reward function described by the following equations:
\begin{equation}
R = R_{\text{torque sm.}} + R_{\text{foot offset}} + R_{\text{posture}} + R_{\text{velocity}} + R_{\text{torque}} + R_{\text{constraints}} 
\end{equation}
with
\begin{align}
R_{\text{torque sm.}} &= -2\times \big[ (\tau_{\text{knee}} - \tau_{\text{knee}}^{\text{old}})^2 + (\tau_{\text{hip}} - \tau_{\text{hip}}^{\text{old}})^2\big]\nonumber\\
R_{\text{foot offset}} &= -1\times x_{\text{foot}}^2 \nonumber\\
R_{\text{posture}} &= -0.1 \times \big[ (z_{\text{hip}} - z_{\text{foot}}) -  z^{\text{target}}_{\text{posture}}\big]^2 \nonumber\\
R_{\text{torque}} &= -10^{-7}\times \big[ \tau_{\text{knee}}^2 + \tau_{\text{hip}}^2  \big]\nonumber\\
R_{\text{velocity}} &= -10^{-4}\times \big[\omega_{\text{knee}}^2 + \omega_{\text{hip}}^2 \big]\nonumber\\
R_{\text{constraints}} &= -0.1 \times \mathbf{1}_{\text{phys. violation}}
\end{align}

where 
\begin{itemize}
	\item $\omega_{\text{knee}}$ and $\omega_{\text{hip}}$ are the knee and hip angular velocities in radians per second, respectively.
	\item $\tau_{\text{knee}}$ and $\tau_{\text{hip}}$ are the knee and hip torques in Newton meters, respectively.
	\item $x_{\text{foot}}$ and $z_{\text{foot}}$ are the horizontal and vertical foot offsets in meters from the desired standing point on the ground, respectively.
	\item $z_{\text{hip}}$ is the vertical hip offset in meters from the desired standing point on the ground.
	\item $z^{\text{target}}_{\text{posture}}$ is a target posture height of 0.1~m.
	\item $\tau_{\text{knee}}^{\text{old}}$ and $\tau_{\text{hip}}^{\text{old}}$ are the values of $\omega_{\text{knee}}$ and $\omega_{\text{hip}}$ from the previous time-step, respectively.
	\item $\mathbf{1}_{\text{phys. violation}}$ is an indicator variable only being one when the agent violates the physical safety bounds of the robotic leg hardware. This involves exceeding the limits of allowed hip or knee angles and angular velocities, or the vertical offsets of the hip and knee. Such violations could result in physical damage to the hardware and were penalized during the training.
\end{itemize}

\section{Broader Impact}

This work provides foundational theoretical results and builds upon reinforcement learning techniques within the area of machine learning. Reinforcement learning methods can provide stable control and decision-making processes for a range of challenging applications in robotics~\citep{hwangbo2019learning}, computer vision~\citep{yun2017action}, advertisement and recommendation systems~\citep{rohde2018recogym}, human search and rescue in natural disasters~\citep{niroui2019deep,doroodgar2014learning}, automated resource management~\citep{mao2016resource}, chemistry~\citep{zhou2017optimizing}, computational biology~\citep{popova2018deep}, and even clinical surgeries~\citep{nguyen2019manipulating}. 

Although many implications could result from the application of reinforcement learning, in this work we focused especially on settings where intelligence, precision, and speed are required for controlling robotic movements. Our work particularly investigated methods for controlling oscillation characteristics using reinforcement learning agents. Such improvements could help stabilize high-bandwidth robotic environments and may facilitate the training of intelligent agents for robotic hardware where safety is of concern~\citep{amodei2016concrete}. The negative consequences of this work could include the removal of human decision-making from the controller design loop, the unknown existence of unforeseen loopholes in the engineered behavior, and vulnerability to policy induction attacks~\citep{behzadan2017vulnerability}.

To mitigate the risks, we encourage further research to develop methods to provide guarantees and definitive answers about agent behavior. In other words, a general framework for making guaranteed statements about the behavior of the trained agents is missing. For instance, one cannot currently guarantee that the agent would act safely under all circumstances and perturbations even if it achieves high performances in practice. Understanding exploitation techniques of such reinforcement learning agents and designing processes to prevent such abuses could be of paramount societal concern. 



\end{document}